%% file: main_paper.tex
\documentclass{article}
\usepackage{iclr2025_conference,times}
\input{math_commands.tex}

\input{macro}

\usepackage{caption}
\usepackage{mathtools}
\mathtoolsset{showonlyrefs}
\usepackage[hyphens]{url}
\usepackage{booktabs}
\usepackage{multirow}
\usepackage{tablefootnote}
\usepackage{hyperref}
\usepackage{url}
\usepackage[linesnumbered,ruled,vlined]{algorithm2e}

\title{Efficient Training of Neural Stochastic Differential Equations by Matching Finite Dimensional Distributions}

\author{
Jianxin Zhang \thanks{Most of the work was done during an internship at Cisco.} \\
  EECS, University of Michigan \\
  Ann Arbor, MI 48109 \\
  \texttt{jianxinz@umich.edu} \\
\And
Josh Viktorov, Doosan Jung, Emily Pitler \\
Cisco Systems \\
New York, NY \quad Lakewood, CO \quad San Jose, CA \\
\texttt{\{joviktor, doojung, epitler\}@cisco.com}
}

\iclrfinalcopy 
\begin{document}

\date{}
\maketitle

\begin{abstract}
Neural stochastic differential equations (Neural SDEs) have emerged as powerful mesh-free generative models for continuous stochastic processes, with critical applications in fields such as finance, physics, and biology. Previous state-of-the-art methods have relied on adversarial training, such as GANs, or on minimizing distance measures between processes using signature kernels. However, GANs suffer from issues like instability, mode collapse, and the need for specialized training techniques, while signature kernel-based methods require solving linear PDEs and backpropagating gradients through the solver, whose computational complexity scales quadratically with the discretization steps. In this paper, we identify a novel class of strictly proper scoring rules for comparing continuous Markov processes. This theoretical finding naturally leads to a novel approach called Finite Dimensional Matching (FDM) for training Neural SDEs. Our method leverages the Markov property of SDEs to provide a computationally efficient training objective. This scoring rule allows us to bypass the computational overhead associated with signature kernels and reduces the training complexity from $O(D^2)$ to $O(D)$ per epoch, where $D$ represents the number of discretization steps of the process. We demonstrate that FDM achieves superior performance, consistently outperforming existing methods in terms of both computational efficiency and generative quality.

\end{abstract}

\section{Introduction}
\label{sec:intro}
Stochastic differential equations (SDEs) are a modeling framework used to describe systems influenced by random forces, with applications spanning finance, physics, biology, and engineering. They incorporate stochastic terms to allow the modeling of complex systems under uncertainties. 

A \emph{neural stochastic differential equation} (Neural SDE) \citep{kidger21b, issa2023sigker, tzen2019neuralstochasticdifferentialequations, jia2019neuraljump, Hodgkinson2021normalflows, Li2020latentSDE, Morrill2020NeuralRD} is an SDE where neural networks parameterize the drift and diffusion terms. This model acts as a mesh-free generative model for time-series data and has shown a significant impact in financial applications \citep{Arribas2021sigsde, gierjatowicz2020robust, choudhary2023funvolmultiassetimpliedvolatility, hoglund2023neuralrdeapproachcontinuoustime}.

Training Neural SDEs typically involves minimizing a distance measure between the distribution of generated paths and the distribution of observed data paths. State-of-the-art performance has been achieved using signature kernels to define a distance measure on path space \citep{issa2023sigker}. Although effective, this approach requires solving a linear partial differential equation (PDE) whose computational complexity scales quadratically with the discretization step, which becomes impractical for long time series. An alternative is training these models adversarially as Generative Adversarial Network (GAN) \citep{kidger21b}. However, GAN-based training can be fraught with issues such as instability, mode collapse, and the need for specialized techniques.


In this paper, we present a theoretical result that extends scoring rules for comparing distributions in finite-dimensional spaces to those for continuous Markov processes. This extension forms the basis of a novel algorithm, \textit{Finite Dimensional Matching} (FDM), designed for training generative models of stochastic processes. FDM exploits the Markovian nature of SDEs by leveraging the two-time joint distributions of the process, providing an efficient training objective that bypasses the complexities of signature kernels. Notably, FDM reduces the computational complexity from $O(D^2)$ to $O(D)$ per training step, where $D$ represents the number of discretization steps.

The key contributions of this paper are as follows:
\begin{itemize}
    \item Our main theorem shows that scoring rules to compare continuous Markov processes can be easily built upon scoring rules on finite-dimensional space.
    \item Our main theorem suggests an efficient training method, FDM, for Neural SDE.
    \item Our experiments show that FDM outperforms prior methods on multiple benchmarks.
\end{itemize}

The rest of the paper is organized as follows: Section \ref{sec:related} provides a review of the relevant literature. In Section \ref{sec:prelim}, we present preliminary results that lay the foundation for our main contributions. Section \ref{sec:fdm} introduces our main theorem, which extends scoring rules for finite dimensions to continuous Markov processes and leads to the development of our novel \textit{Finite Dimensional Matching} (FDM) algorithm. In section \ref{sec:analysis}, we analyze the sample complexity and sensitivity of FDM.  Section \ref{sec:experiments} details the experimental setup and results, demonstrating the superiority of FDM in terms of both computational efficiency and generative performance across several benchmark datasets. Finally, Section \ref{sec:conclusion} concludes the paper by summarizing the contributions, limitations, and directions for future work.

\section{Related Work}
\label{sec:related}

We begin by reviewing prior applications of scoring rules in generative modeling followed by an exploration of the literature on training Neural SDEs.

\subsection{Scoring Rules}
Scoring rules offer a method to measure discrepancies between distributions \citep{Gneiting07ProperScore} and are especially appealing for generative modeling and have been employed in training various generative models \citep{Bouchacourt16DiscoNets, Gritsenko20speech, Pacchiardi24forecastingscore, pacchiardi2022likelihoodfreeinferencegenerativeneural, issa2023sigker, Bonnier24truncated}. Notably, \cite{Pacchiardi24forecastingscore} apply scoring rules to discrete Markov chains, although their extension to continuous-time processes has not yet been explored. \cite{issa2023sigker} and \cite{Bonnier24truncated} construct scoring rules for continuous processes by utilizing signature kernels.
 
\subsection{Neural SDEs}
Several methods have been proposed for training Neural SDEs as generative models, each differing in how they define the divergence or distance between distributions on path space. In Table \ref{tab:methods}, we compare different methods for training Neural SDEs, highlighting their divergence measures and the corresponding discriminator or training objectives. Our approach, \textbf{Finite Dimensional Matching} (FDM), introduces a novel scoring rule specifically designed for continuous Markov processes.

\begin{table}[h!]
    \centering
    \caption{Methods for training Neural SDEs. \textbf{SigKer} stands for signature kernel \citep{issa2023sigker}, \textbf{TruncSig} is for truncated signature \citep{Bonnier24truncated}, \textbf{SDE-GAN} is proposed by \cite{kidger21b}, and \textbf{Latent SDE} is proposed by \cite{Li2020latentSDE}.}
    \label{tab:methods}

    \begin{tabular}{p{0.17\textwidth}p{0.33\textwidth}p{0.4\textwidth}}
        \textbf{Methods} & \textbf{Divergence or distance} & \textbf{Discriminator or training objective}  \\
        \hline
        \textbf{Latent SDE} & KL-divergence & Monte-Carlo simulation of free energy  \\
        \textbf{SDE-GAN} & 1-Wasserstein distance & Optimizing discriminator nets   \\
        \textbf{SigKer} & Signature kernel score & Solving Goursat PDEs \\
        \textbf{TruncSig} & Truncated signature kernel score & Truncated approximation of signature \\
        \textbf{FDM (Ours)} & A novel class of scoring rules for continuous Markov processes & Standard ERM of the expected scores  \\
    \end{tabular}
\end{table}

One method to train Neural SDE is the latent SDE model introduced by \cite{Li2020latentSDE}, which trains a Neural SDE using variational inference principles \citep{opper19variationalSDE}. In their framework, training involves optimizing the free energy that includes the Kullback-Leibler (KL) divergence between the original SDE and an auxiliary SDE. These two SDEs share the same diffusion term but have different drift terms. The KL divergence between their laws can be computed using Girsanov's change of measure theorem. However, the performance of latent SDEs is generally inferior to SDE-GANs due to lower model capacity \citep{kidger21b, issa2023sigker}.

A prominent method is the SDE-GAN introduced by \cite{kidger21b}, which employs adversarial training to fit a Neural SDE, as in Wasserstein-GANs \citep{arjovsky17WGAN}. This approach relies on the 1-Wasserstein distance, with the discriminator parameterized by neural controlled differential equations \citep{kidger2020neuralcde, morrill2021neuralrough}. However, SDE-GANs are notoriously difficult to train due to their high sensitivity to hyperparameters. Another major challenge is the need for a Lipschitz discriminator, which requires techniques like weight clipping and gradient penalties to enforce this constraint \citep{kidger2022neuraldifferentialequations}. Adversarial training for time-series generative models has also been explored in the context of discrete data \citep{Ni2022sigwGAN, Yoon2019timeGAN}.


Another key contribution to training Neural SDEs is the signature kernel method \citep{issa2023sigker, Bonnier24truncated}, which minimizes a distance measure based on signature kernels \citep{lee2023signaturekernel} of paths. However, evaluation of the signature kernel requires solving Goursat partial differential equations (PDEs) and backpropagating gradients through the solver \citep{GoursatPDE}. The computational complexity of solving Goursat PDEs scales quadratically with the number of discretization steps, which can be prohibitive for long time series data. \cite{Bonnier24truncated} approximates the signature kernel as inner products of truncated signature transforms, called truncated signature. However, the scoring rule based on truncated signature is not strictly proper and has $\cO(d^N)$  memory complexity where $d$ is the number of features and $N$ is the truncation size.

The concurrent work of \cite{snow2025efficientneuralsdetraining} introduces a novel technique that leverages Wiener space cubature theory to bypass Monte Carlo simulations in Neural SDE training.
 
\section{Preliminaries}\label{sec:prelim}
In this section, we set up the notations and introduce the following preliminary concepts: Neural SDEs, Markov processes, and scoring rules.


\textbf{Background and Notations} Let $\cbs{\Omega, \cF, \PP}$ be a probability space where $\Omega, \cF, \PP$ denote the sample space, sigma-algebra, and probability measure, respectively. For a random variable $\xi$, the function $\PP_{\xi} = \PP \circ \xi^{-1}$ is the induced measure on its range space. In particular, for a random process $X$, $\PP_{X}$ denotes its law. We use the superscript $^{\top}$ for the transposition of a matrix or vector.

\textbf{Neural SDE} Let $B:  [0, T] \times \Omega  \to \RR^{d_{noise}}$ be a Brownian motion on $\RR^{d_{noise}}$, where $d_{noise} \in \NN$. We define a Neural SDE as in \cite{issa2023sigker} and \cite{kidger21b}:
$$
Z_0 = \xi^{\theta}(a),\text{ } d Z_t = \mu^{\theta}(t, Z_t) dt + \sigma^{\theta} (t, Z_t) d W_t, \text{ }  X^{\theta}_t = A^{\theta} Z_t + b^{\theta}
$$
where $a$ is sampled from a $d_{initial}$-dimensional standard Gaussian distribution, 
$$\xi^{\theta}: \RR^{d_{initial}} \to \RR^{d_z}, \text{ }  \mu^{\theta}: [0, T] \times \RR^{d_z} \to \RR^{d_z},\text{ } \sigma^{\theta}: [0, T] \times \RR^{d_z} \to \RR^{d_z \times d_{noise}}$$
along with $A^{\theta} \in \RR^{d_x \times d_z}$, $b^{\theta} \in \RR^{d_x}$, are functions parameterized by neural networks, and $d_{initial}, d_x, d_z \in \NN$. We assume additionally that $\mu^{\theta}$ and $\sigma^{\theta}$ are Lipschitz continuous in both arguments and $\xi^{\theta}(a)$ has finite second-order momentum. These conditions ensure that the SDE for $Z_t$ has a unique strong solution. Suppose $Y_t$ is the data process, we'd like to train the neural networks $\theta$ on data sampled from $\PP_{Y}$ so that $\PP_{X^{\theta}}$ matches $\PP_{Y}$.

\textbf{Markov Process}
We say a continuous process $X_t$ with filtration $\{\mathcal{F}_t\}$ is Markov if $X_u$ is independent of $\mathcal{F}_t$ for all $u \geq t$ given $X_t$ \citep{Kallenberg2021}. For an SDE of the form 
$
dX_t = \mu(t, X_t) \, dt + \sigma(t, X_t) \, dB_t,
$
with the filtration generated by the Brownian motion $B_t$, $X_t$ is Markov as long as the SDE has a unique strong solution (Theorem 9.1, \cite{mao2007stochastic}).

\textbf{Scoring Rules}
Given a measurable space $(\Omega_0, \mathcal{F}_0)$ and \(\omega_0 \in \Omega_0\), a scoring rule \citep{Gneiting07ProperScore} \(s(P, \omega_0)\) maps a probability measure \(P\) on \(\Omega_0\) and a sample \(\omega_0\) to \(\mathbb{R}\). The expected score is defined as \(S(P, Q) = \mathbb{E}_Q[s(P, \omega_0)] = \int_{\Omega_0} s(P, \omega_0) \, dQ(\omega_0)\), where \(P\) is the predictive distribution and \(Q\) is the true distribution. The scoring rule \(s\) is said to be \emph{proper} if the expected score satisfies \(S(P, Q) \leq S(Q, Q)\). It is \emph{strictly proper} if \(S(P, Q) = S(Q, Q) \iff P = Q\). For example, let $k:\RR^{d} \times \RR^{d} \to \RR$ be the RBF kernel defined as $k(x, y) = \exp\left(-\gamma\|x - y\|^2 \right),$ where $\gamma > 0$ is a parameter that determines the width of the kernel, then $s(P, z)=\frac{1}{2} \EE_{Z,Z'\sim P} k(Z, Z') - \EE_{Z\sim P} k(Z, z)$ is a strictly proper scoring rule for distribution on $\RR^{d}$ \citep{Gneiting07ProperScore}.

Let $P^{\theta}$ be a distribution controlled by a generative model $\theta$, and let $Q$ be the true distribution accessed through data. Given a strictly proper scoring rule $s$, sufficient model capacity of $\theta$, and sufficient data points from $Q$, $P^{\theta}$ can be trained by maximizing $S(P^{\theta}, Q)$ over $\theta$, leading to $P^{\theta} = Q$. While many scoring rules for finite-dimensional spaces have been proposed, we lack strictly proper scoring rules for random processes that can be evaluated efficiently. In our main claim, we show that a strictly proper scoring rule for a two-time joint distribution, \emph{i.e.}, the distributions $\{(X_{t_1}, X_{t_2}), \forall t_1, t_2 \in [0, T]\}$, for a random process \(X\), can be converted into a strictly proper scoring rule for continuous Markov processes.

\section{Finite Dimensional Matching}
\label{sec:fdm}
In this section, we present our main theorem, which converts a scoring rule for a two-time joint distribution into a scoring rule for a Markov process. Specifically, if we have a scoring rule for $\Omega_0 = \RR^{2d}$, then Theorem \ref{thm:scoring4continuous} allows us to convert it into a scoring rule for Markov processes $X, Y: [0, T] \to \RR^d$, where $d \in \NN$ and $T \in \RR_{>0}$.

\subsection{Scoring Rule for Markov Process}
\label{sec:mainthm}
In this section, we present our main theorem which shows that a strictly proper scoring rule for the two-time joint distributions can be converted to a scoring rule for two Markov processes. Let continuous Markov processes $X, Y$ on $[0, T]$ take values in a Polish space $\cE$ endowed with its Borel $\sigma$-algebra.  Let $s$ be any strictly proper scoring rule defined on $\cE \times \cE$. Let $S(P,Q)=\EE_Q [s(P, \omega)] < \infty, \forall $ measures $P, Q$ on $\cE \times \cE$. 
We define the scoring rule $\Bar{s}$ for continuous Markov processes as following:
\begin{defi}
    $\Bar{s} (\PP_X, y) = \EE_{(t_1, t_2) \sim U([0,T]^2)} s (\PP_{(X_{t_1}, X_{t_2})}, (y_{t_1}, y_{t_2})),$ where $\PP_{(X_{t_1}, X_{t_2})}$ is the joint marginal distributions at times $t_1, t_2$ of $X$, and $U([0,T]^2)$ is the uniform distribution on $[0, T]^2$.
\end{defi}
Let $\Bar{S}(\PP_X, \PP_Y) = \EE_{y\sim \PP_Y} [\Bar{s} (\PP_X, y)] $. Now we present our main claim, with its proofs deferred to the appendix.

\begin{thm}
\label{thm:scoring4continuous}
If $s$ is a strictly proper scoring rule for distributions on $\cE \times \cE$,  $\Bar{s}$ is a strictly proper scoring rule for $\cE$-valued continuous Markov processes on $[0, T]$ where $T \in \RR_{>0}$. That is, for any $\cE$-valued continuous Markov processes $X, Y$ with laws $\PP_X, \PP_Y$, respectively, $\Bar{S}(\PP_X, \PP_Y) \leq \Bar{S}(\PP_Y, \PP_Y)$ with equality achieved only if $\PP_X = \PP_Y$.
\end{thm}

In the appendix, we present a more generalized version of Theorem \ref{thm:scoring4continuous} that does not require the timestamps \( t_1 \) and \( t_2 \) to be sampled from $U([0, T]^2)$ in the definition of \( \Bar{s} \). Nonetheless, to maintain clarity and simplicity, we focus our discussion on the uniform sampling case in the main paper.

Suppose $X^{\theta}$ is a Markov process parameterized by neural net parameters $\theta$ with sufficient capacity. Therefore, maximizing $\Bar{S}(\PP_{X^{\theta}}, \PP_Y) = \EE_{Y \sim \PP_{Y}} [\Bar{s}(\PP_{X^{\theta}}, Y)]$, which can be achieved by maximizing the corresponding empirical average, will result in $\PP_{X^{\theta}} = \PP_{Y}$. 

We present a concrete example on the application of Theorem \ref{thm:scoring4continuous}. Consider continuous Markov processes $X, Y$ on $[0, T]$ taking values in $\RR^d$. Let $k:\RR^{2d} \times \RR^{2d} \to \RR$ be  the RBF kernel, recall that $s(P, z)=\frac{1}{2} \EE_{Z,Z'\sim P} k(Z, Z') - \EE_{Z\sim P} k(Z, z)$ is a strictly proper scoring rule for distribution on $\RR^{2d}$ \citep{Gneiting07ProperScore}. By Theorem \ref{thm:scoring4continuous},
\begin{equation}
    \Bar{s}(\PP_X, y) =  \EE_{(t_1, t_2) \sim U([0,T]^2)}  \brs{\frac{1}{2} \EE_{X,X'} k([X_{t_1}, X_{t_2}], [X'_{t_1}, X'_{t_2}]) - \EE_{X} k([X_{t_1}, X_{t_2}], [y_{t_1}, y_{t_2}])} \label{eq:expected_score}
\end{equation}
is strictly proper, where $[\cdot, \cdot]$ is the concatenation of two vectors.  $\Bar{S}(\PP_{X^{\theta}}, \PP_Y) = \EE_{Y \sim \PP_{Y}} [\Bar{s}(\PP_{X^{\theta}}, Y)]$ can be estimated through empirical average and optimized efficiently. 

\subsection{FDM Algorithm}
\label{sec:algorithm}

 We consider an expected score $\Bar{S}(\PP_{X^{\theta}}, \PP_Y) = \EE_{Y \sim \PP_{Y}} [\Bar{s}(\PP_{X^{\theta}}, Y)]$, which can be estimated using an empirical average $\hat{S}$. For example, an unbiased estimator of $\Bar{S}$ for $\Bar{s}$ defined in \eqref{eq:expected_score} can be constructed using batches of generated paths \(\mathcal{B}_X = \{ x^i \}_{i=1}^B\) and data paths \(\mathcal{B}_Y = \{ y^i \}_{i=1}^B\). For each \(i\), independently sample two timestamps \(t_i\) and \(t'_i\). The empirical estimator is then given by:
\[
\hat{S}(\mathcal{B}_X, \mathcal{B}_Y) = \frac{1}{2B(B-1)} \sum_{i \ne j} k\left( [x^i_{t_j}, x^i_{t'_j}], [x^j_{t_j}, x^j_{t'_j}] \right) - \frac{1}{B^2} \sum_{i=1}^B \sum_{j=1}^B k\left( [x^i_{t_j}, x^i_{t'_j}], [y^j_{t_j}, y^j_{t'_j}] \right).
\]
Note that the above estimator $\hat{S}$ only requires each data path to be (potentially irregularly) observed at two distinct timestamps, and we can observe the $x^i$'s at any timestamps since they are generated by the Neural SDE model. Alternative empirical objectives are provided in the appendix.

In Algorithm \ref{algo:fdm}, we present the concrete \textit{finite dimensional matching} (FDM) algorithm derived from Theorem \ref{thm:scoring4continuous} to train a Neural SDE $X^{\theta}$.

\begin{algorithm}[H]
    \label{algo:fdm}
        \caption{Finite Dimensional Matching (FDM)}
        \KwIn{Neural SDE $X^{\theta}$, data paths $\{y^i: i \in [N]\}$, strictly proper scoring rule $s$ , batch size $B$}
        \Repeat{stopping criterion is met}{
            Generate a batch of simulated paths $\cB_X = \{x^i: i \in [B]\}$ using the Neural SDE model $\theta$\;
            
            Randomly sample a batch of data paths $\cB_Y \subset \{y^i: i \in [N]\}$ with $|\cB_Y|=B$ \; 
            
            Compute the empirical estimate $\hat{S}(\cB_X, \cB_Y)$ of $\Bar{S}(\PP_{X^{\theta}}, \PP_Y)$\; 
            
            Maximize $\hat{S}$ with respect to $\theta$ using an optimizer of the user's choice\;
        }
\end{algorithm}

\section{Theoretical Properties}
\label{sec:analysis}

In this section, we investigate the sample complexity and sensitivity of the proposed scoring rules \(\Bar{s}\). All proofs are deffered to the appendix.

\subsection{Sample Complexity}

We show that the sample complexity of the estimator $\hat{S}(\cB_X, \cB_Y)$ retains the classical sample complexity of a kernel-based scoring rule $s$ \citep{gretton12a}. Let \(k\) be a kernel associated with a \textit{Reproducing Kernel Hilbert Space} (RKHS) and $s(P, z) = \frac{1}{2} \EE_{Z,Z'\sim P} k(Z, Z') - \EE_{Z\sim P} k(Z, z)$ be a strictly proper scoring rule. Recall that  $\Bar{s} (\PP_X, y) = \EE_{(t_1, t_2) \sim U([0,T]^2)} s (\PP_{(X_{t_1}, X_{t_2})}, (y_{t_1}, y_{t_2}))$ and $\Bar{S}(\PP_X, \PP_Y) = \EE_{y\sim \PP_Y} [\Bar{s} (\PP_X, y)]$. 

\begin{thm}
\label{thm:samplecomplexity}
Let \(k(\cdot, \cdot)\) satisfy \(0 \leq k(\cdot, \cdot) \leq K\) and the batch size $B \geq$ 2.
 For any \(\varepsilon > 0\),
$$
\PP\left( |\hat{S} - \mathbb{E}[\hat{S}]| \geq \varepsilon \right) \leq  2 \exp\left( -\frac{8B\varepsilon^2}{47 K^2} \right).
$$
 Equivalently, with probability at least \(1 - \delta\), the deviation of \(\hat{S}\) from its expected value \(\bar{S}(\mathbb{P}_{X}, \mathbb{P}_Y)\) is bounded as
$$
\big|\,\hat{S}(\mathcal{B}_X, \mathcal{B}_Y) - \bar{S}(\mathbb{P}_{X}, \mathbb{P}_Y)\,\big| \leq K \sqrt{ \frac{47 \ln(2/\delta)}{8B} }.
$$
\end{thm}

$\hat{S}$ exhibits a sample complexity analogous to the classical sample complexity of kernel-based scoring rules $s$. In Section \ref{appendix:samplecomplexity} of the appendix, we extend this analysis to an alternative estimator where all sample paths are evaluated at $n$ shared timestamps.

\subsection{Sensitivity}
Let \( X \) be an Itô diffusion on \( \mathbb{R}^d \), i.e., \( d X_t = \mu(t, X_t) dt + \sigma(t, X_t) dB_t \). The following theorem shows how perturbations in \(\mu\) and \(\sigma\) affect the value of the scoring rule \(\Bar{s}(\mathbb{P}_X, y)\).
\begin{thm}
    \label{thm:sensitivity}
    Let $X$ satisfy $d X_t = \mu(t, X_t)dt + \sigma(t, X_t) dB_t$ on $\RR^d$. Let $\Tilde{X}$ satisfy $d \Tilde{X}_t = \Tilde{\mu}(t, \Tilde{X}_t)dt + \Tilde{\sigma}(t, \Tilde{X}_t) dB_t$ on $\RR^d$ where $\forall t, x, \norm{\mu(t, x) - \Tilde{\mu}(t, x)}_2 \leq \delta_{\mu}$, $\norm{\sigma(t, x) - \Tilde{\sigma}(t, x)}_2 \leq \delta_{\sigma}$, and $\delta_{\mu}$, $\delta_{\sigma}$ are constants. Assume the scoring rule $s(P, z)$ is Lipschitz in terms of distribution $P$ with respect to the Wasserstein-2 distance. Assume both $X$ and $\Tilde{X}$ have unique strong solutions and share the same initial conditions, then
    $\abs{\Bar{s}(\PP_X, y) - \Bar{s}(\PP_{\Tilde{X}}, y)} \leq L_s C(\delta_{\mu} + \delta_{\sigma}), $
    where $L_s$ is the Lipschitz constant of $s$, the constant $C$ depends on Lipschitz constants of $\mu$ and $\sigma$.
\end{thm}

If we allow different sampling distributions of \(t_1, t_2\) than uniform as in the generalized main theorem Theorem \ref{thm:scoring4continuous_general} in the appendix, \(C\) may also depend on the sampling distribution of \(t_1, t_2\). This result provides a theoretical guarantee that small changes in the dynamics of the process result in changes to the scoring rule that are linear with respect to $\delta_{\mu} + \delta_{\sigma}$. 

\section{Experiments}
\label{sec:experiments}

\begin{table}[ht]
    \centering
    \caption{Average KS test scores (\textbf{lower is better}) and chance of rejecting the null hypothesis (\%) at 5\%-significance level on marginals (\textbf{lower is better}) of metal prices, trained on paths evenly sampled at 64 timestamps.}
    \label{tab:ks_forex_metals64}
    \begin{tabular}{ccccccc}
        Dim & Model & $t=6$ & $t=19$ & $t=32$ & $t=44$ & $t=57$ \\
        \hline
        \multirow{4}{*}{SILVER} 
        & \textbf{SigKer} & .144, 23.9 & .134, 14.4 & .130, 11.5 & .126, 9.20 & .122, \textbf{7.96} \\
        & \textbf{TruncSig} & .274, 97.9 & .277, 98.7 & .293, 99.3 & .304, 99.6 & .315, 99.6 \\
        & \textbf{SDE-GAN} & .330, 70.0 & .647, 100. & .789, 100. & .813, 100. & .828, 100. \\
        & \textbf{FDM (ours)} & \textbf{.118}, \textbf{9.76} & \textbf{.114}, \textbf{7.24} & \textbf{.112}, \textbf{6.08} & \textbf{.114}, \textbf{7.20} & \textbf{.117}, {8.16} \\
        \hline
        \multirow{4}{*}{GOLD} 
        & \textbf{SigKer} & .129, 10.7 & .127, 9.40 & .128, 10.1 & .126, 9.80 & .123, \textbf{8.16} \\
        & \textbf{TruncSig} & .255, 94.7 & .274, 98.4 & .298, 96.8 & .316, 99.8 & .330, 99.8 \\
        & \textbf{SDE-GAN} & .244, 90.6 & .299, 94.8 & .318, 96.8 & .336, 96.8 & .352, 91.7 \\
        & \textbf{FDM (ours)} & \textbf{.119}, \textbf{9.32} & \textbf{.117}, \textbf{8.00} & \textbf{.115}, \textbf{7.34} & \textbf{.116}, \textbf{7.66} & \textbf{.118}, {8.52} \\
    \end{tabular}
\end{table}

\begin{table}[ht]
    \centering
    \caption{Average KS test scores and chance of rejecting the null hypothesis (\%) at 5\%-significance level on marginals of U.S. stock indices, trained on paths evenly sampled at 64 timestamps. "DOLLAR", "USA30", "USA500", "USATECH", and "USSC2000" stand for US Dollar Index, USA 30 Index, USA 500 Index, USA 100 Technical Index, and US Small Cap 2000, respectively.}
    \label{tab:ks_indices64}
    \begin{tabular}{ccccccc}
        Dim & Model & t=6 & t=19 & t=32 & t=44 & t=57 \\
        \midrule
        \multirow{4}{*}{DOLLAR} 
        & \textbf{SigKer} & .262, 76.4 & .316, 82.1 & .322, 83.7 & .314, 84.4 & .296, 83.9 \\
        & \textbf{TruncSig} & .279, 98.3 & .303, 99.4 & .323, 99.7 & .339, 99.8 & .354, 99.9 \\
        & \textbf{SDE-GAN} & .389, 93.0 & .544, 98.3 & .605, 99.5 & .599, 99.8 & .553, 99.8 \\
        & \textbf{FDM (ours)} & \textbf{.143}, \textbf{25.6} & \textbf{.151}, \textbf{29.6} & \textbf{.153}, \textbf{30.7} & \textbf{.155}, \textbf{31.7} & \textbf{.156}, \textbf{33.0} \\
        \midrule
        \multirow{4}{*}{USA30} 
        & \textbf{SigKer} & .200, 56.5 & .239, 78.8 & .264, 91.8 & .279, 94.2 & .291, 93.7 \\
        & \textbf{TruncSig} & .171, 51.5 & .194, 61.9 & .213, 70.9 & .228, 79.1 & .250, 89.5 \\
        & \textbf{SDE-GAN} & .311, 80.9 & .402, 91.6 & .428, 90.6 & .550, 99.9 & .666, 100. \\
        & \textbf{FDM (ours)} & \textbf{.132}, \textbf{15.6} & \textbf{.123}, \textbf{10.0} & \textbf{.124}, \textbf{9.50} & \textbf{.124}, \textbf{9.30} & \textbf{.121}, \textbf{8.04} \\
        \midrule
        \multirow{4}{*}{USA500} 
        & \textbf{SigKer} & .287, 86.8 & .350, 92.3 & .367, 94.0 & .365, 93.5 & .355, 93.1 \\
        & \textbf{TruncSig} & .189, 57.7 & .204, 63.7 & .221, 71.3 & .231, 77.4 & .244, 86.1 \\
        & \textbf{SDE-GAN} & .310, 97.0 & .448, 93.3 & .625, 100. & .713, 100. & .746, 100. \\
        & \textbf{FDM (ours)} & \textbf{.122}, \textbf{9.82} & \textbf{.117}, \textbf{6.84} & \textbf{.117}, \textbf{6.56} & \textbf{.117}, \textbf{6.30} & \textbf{.118}, \textbf{6.52} \\
        \midrule
        \multirow{4}{*}{USATECH} 
        & \textbf{SigKer} & .212, 80.6 & .240, 91.0 & .242, 90.9 & .239, 90.6 & .245, 92.2 \\
        & \textbf{TruncSig} & .197, 74.1 & .227, 86.1 & .247, 91.3 & .265, 94.8 & .280, 97.6 \\
        & \textbf{SDE-GAN} & .420, 99.9 & .786, 100. & .910, 100. & .950, 100. & .969, 100. \\
        & \textbf{FDM (ours)} & \textbf{.123}, \textbf{9.92} & \textbf{.119}, \textbf{7.48} & \textbf{.118}, \textbf{6.68} & \textbf{.118}, \textbf{6.54} & \textbf{.118}, \textbf{6.74} \\
        \midrule
        \multirow{4}{*}{USSC2000} 
        & \textbf{SigKer} & .255, 70.1 & .312, 88.3 & .328, 94.5 & .325, 94.2 & .314, 93.6 \\
        & \textbf{TruncSig} & .180, 46.5 & .199, 62.5 & .221, 78.8 & .233, 89.0 & .248, 96.4 \\
        & \textbf{SDE-GAN} & .317, 75.4 & .572, 98.2 & .764, 100. & .843, 100. & .887, 100. \\
        & \textbf{FDM (ours)} & \textbf{.134}, \textbf{16.6} & \textbf{.126}, \textbf{11.1} & \textbf{.122}, \textbf{8.80} & \textbf{.124}, \textbf{9.14} & \textbf{.122}, \textbf{8.34} \\

    \end{tabular}
\end{table}

\begin{table}[ht]
    \centering
    \caption{Average KS test scores and the chance of rejecting the null hypothesis (\%) at 5\%-significance level on marginals for different currency pairs (EUR/USD and USD/JPY), trained on paths evenly sampled at 64 timestamps.}
    \label{tab:ks_forex64}
    \begin{tabular}{ccccccc}
        Dim & Model & t=6 & t=19 & t=32 & t=44 & t=57 \\
        \midrule
        \multirow{4}{*}{EUR/USD}
        & \textbf{SigKer} & .251, 66.5 & .293, 70.3 & .288, 66.4 & .271, 55.4 & .248, 38.5 \\
        & \textbf{TruncSig} & .273, 97.9 & .313, 99.6 & .340, 99.8 & .354, 99.9 & .369, 99.9 \\
        & \textbf{SDE-GAN} & .529, 89.2 & .665, 95.8 & .723, 96.0 & .754, 97.7 & .784, 99.8 \\
        & \textbf{FDM (ours)} & \textbf{.125}, \textbf{12.9} & \textbf{.113}, \textbf{6.26} & \textbf{.109}, \textbf{5.04} & \textbf{.110}, \textbf{5.46} & \textbf{.111}, \textbf{6.10} \\
        \midrule
        \multirow{4}{*}{USD/JPY}
        & \textbf{SigKer} & .165, 34.3 & .189, 38.1 & .191, 34.4 & .188, 30.5 & .185, 29.3 \\
        & \textbf{TruncSig} & .252, 87.8 & .291, 98.0 & .317, 99.4 & .334, 99.7 & .354, 99.9 \\
        & \textbf{SDE-GAN} & .212, 73.4 & .267, 85.2 & .309, 88.6 & .359, 91.2 & .425, 92.8 \\
        & \textbf{FDM (ours)} & \textbf{.120}, \textbf{9.54} & \textbf{.111}, \textbf{6.04} & \textbf{.110}, \textbf{5.54} & \textbf{.111}, \textbf{5.92} & \textbf{.111}, \textbf{5.98} \\
    \end{tabular}
\end{table}

\begin{table}[ht]
    \centering
    \caption{Average KS test scores and the chance of rejecting the null hypothesis (\%) at 5\%-significance level on marginals of energy prices, trained on paths evenly sampled at 64 timestamps. We reserve the latest 20\% data as test dataset and measure how well the model predicts into future. "BRENT", "DIESEL", "GAS", and "LIGHT" stand for U.S. Brent Crude Oil, Gas oil, Natural Gas, and U.S. Light Crude Oil, respectively.}
    \label{tab:ks_energy64}
    \begin{tabular}{ccccccc}
        Dim & Model & t=6 & t=19 & t=32 & t=44 & t=57 \\
        \midrule
        \multirow{4}{*}{BRENT}
        & \textbf{SigKer} & .284, 69.4 & .339, 70.7 & .343, 68.0 & .328, 65.4 & .302, 60.0 \\
        & \textbf{TruncSig} & .254, 91.8 & .264, 95.4 & .273, 96.7 & .292, 98.1 & .303, 98.8 \\
        & \textbf{SDE-GAN} & .487, 97.0 & .812, 100 & .929, 100 & .961, 100 & .981, 100 \\
        & \textbf{FDM (ours)} & \textbf{.127}, \textbf{12.9} & \textbf{.123}, \textbf{10.2} & \textbf{.125}, \textbf{11.5} & \textbf{.124}, \textbf{11.6} & \textbf{.124}, \textbf{11.4} \\
        \midrule
        \multirow{4}{*}{DIESEL}
        & \textbf{SigKer} & .187, 47.1 & .218, 55.4 & .223, 49.6 & .222, 42.9 & .219, 38.1 \\
        & \textbf{TruncSig} & .221, 81.9 & .244, 93.9 & .262, 97.9 & .274, 98.9 & .305, 99.6 \\
        & \textbf{SDE-GAN} & .279, 77.9 & .522, 97.6 & .664, 99.8 & .735, 100 & .793, 100 \\
        & \textbf{FDM (ours)} & \textbf{.122}, \textbf{10.2} & \textbf{.117}, \textbf{7.82} & \textbf{.117}, \textbf{8.20} & \textbf{.123}, \textbf{10.5} & \textbf{.123}, \textbf{10.7} \\
        \midrule
        \multirow{4}{*}{GAS}
        & \textbf{SigKer} & .244, 70.4 & .298, 77.1 & .305, 71.5 & .295, 69.1 & .273, 63.7 \\
        & \textbf{TruncSig} & .244, 82.0 & .280, 93.4 & .301, 97.6 & .328, 99.4 & .342, 99.7 \\
        & \textbf{SDE-GAN} & .337, 86.2 & .586, 99.8 & .717, 99.9 & .801, 100 & .877, 100 \\
        & \textbf{FDM (ours)} & \textbf{.116}, \textbf{7.52} & \textbf{.116}, \textbf{7.36} & \textbf{.123}, \textbf{11.7} & \textbf{.127}, \textbf{14.4} & \textbf{.126}, \textbf{13.7} \\
        \midrule
        \multirow{4}{*}{LIGHT}
        & \textbf{SigKer} & .184, 60.1 & .200, 66.9 & .195, 59.1 & .186, 53.0 & .173, 43.6 \\
        & \textbf{TruncSig} & .245, 91.7 & .261, 94.6 & .272, 96.4 & .292, 98.7 & .308, 99.5 \\
        & \textbf{SDE-GAN} & .266, 74.8 & .403, 76.0 & .464, 85.7 & .604, 98.6 & .717, 99.9 \\
        & \textbf{FDM (ours)} & \textbf{.121}, \textbf{9.82} & \textbf{.122}, \textbf{10.4} & \textbf{.131}, \textbf{15.5} & \textbf{.131}, \textbf{15.2} & \textbf{.132}, \textbf{15.5} \\
    \end{tabular}
\end{table}

\begin{table}[ht]
    \centering
    \caption{Average KS test scores and chance of rejecting the null hypothesis (\%) at 5\%-significance level on marginals of bonds, trained on paths evenly sampled at 64 timestamps. We reserve the most latest 20\% data as test dataset and measure how well the model predicts into future. "BUND", "UKGILT", and "USTBOND" stand for Euro Bund, UK Long Gilt, and US T-BOND, respectively.}
    \label{tab:ks_bonds64}
    \begin{tabular}{ccccccc}
        Dim & Model & t=6 & t=19 & t=32 & t=44 & t=57 \\
        \midrule
        \multirow{4}{*}{BUND}
& \textbf{SigKer} & .210, 49.0 & .244, 52.8 & .251, 54.7 & .252, 58.7 & .245, 54.3 \\
& \textbf{TruncSig} & .261, 95.0 & .296, 99.4 & .328, 99.7 & .350, 99.9 & .362, 99.9 \\
& \textbf{SDE-GAN} & .339, 97.8 & .582, 100 & .613, 99.3 & .764, 100 & .831, 100 \\
& \textbf{FDM (ours)} & \textbf{.119}, \textbf{10.2} & \textbf{.120}, \textbf{8.80} & \textbf{.124}, \textbf{8.66} & \textbf{.125}, \textbf{9.72} & \textbf{.118}, \textbf{8.20} \\
\midrule
        \multirow{4}{*}{UKGILT} 
& \textbf{SigKer} & .158, 26.0 & .183, 27.5 & .197, 30.5 & .201, 31.8 & .201, 32.0 \\
& \textbf{TruncSig} & .200, 67.8 & .244, 89.8 & .285, 98.9 & .312, 99.7 & .337, 99.8 \\
& \textbf{SDE-GAN} & .366, 84.3 & .640, 99.8 & .862, 100 & .902, 100 & .921, 100 \\
& \textbf{FDM (ours)} & \textbf{.127}, \textbf{12.6} & \textbf{.113}, \textbf{6.62} & \textbf{.109}, \textbf{5.28} & \textbf{.109}, \textbf{5.36} & \textbf{.109}, \textbf{5.70} \\
\midrule
        \multirow{4}{*}{USTBOND} 
& \textbf{SigKer} & .189, 40.6 & .207, 37.6 & .213, 35.7 & .213, 35.3 & .214, 36.2 \\
& \textbf{TruncSig} & .229, 73.5 & .248, 81.3 & .284, 95.0 & .310, 99.3 & .334, 99.8 \\
& \textbf{SDE-GAN} & .351, 84.4 & .688, 100 & .853, 100 & .903, 100 & .916, 100 \\
& \textbf{FDM (ours)} & \textbf{.137}, \textbf{17.9} & \textbf{.124}, \textbf{10.8} & \textbf{.115}, \textbf{6.60} & \textbf{.113}, \textbf{6.14} & \textbf{.111}, \textbf{5.64} \\

    \end{tabular}
\end{table}

\begin{table}[ht]
    \centering
    \caption{Average KS test scores and the chance of rejecting the null hypothesis (\%) at 5\%-significance level on marginals for different currency pairs (EUR/USD and USD/JPY), trained on paths evenly sampled at 256 timestamps.}
    \label{tab:ks_forex256}
    \begin{tabular}{ccccccc}
        Dim & Model & $t=25$ & $t=76$ & $t=128$ & $t=179$ & $t=230$ \\
        \midrule
        \multirow{4}{*}{EUR/USD} 
        & \textbf{SigKer} & 
        .535,  100. & 
        .535,  100. & 
        .536,  100. & 
        .546,  100. & 
        .540,  100. \\
        & \textbf{TruncSig} & 
        .137,  \textbf{18.9} & 
        .184,  67.1 & 
        .252,  99.6 & 
        .290,  100. & 
        .318,  100. \\
        & \textbf{SDE-GAN} & 
        \textbf{.134}, 21.6 & 
        .411,  100. & 
        .569,  100. & 
        .548,  100. & 
        .338,  99.9 \\
        & \textbf{FDM (ours)} & 
        .136,  23.0 & 
        \textbf{.112}, \textbf{5.70} & 
        \textbf{.123}, \textbf{12.7} & 
        \textbf{.132}, \textbf{17.8} & 
        \textbf{.141}, \textbf{26.6} \\
        \midrule
        \multirow{4}{*}{USD/JPY} 
        & \textbf{SigKer} & 
        .535,  100. & 
        .534,  100. & 
        .535,  100. & 
        .538,  100. & 
        .541,  100. \\
        & \textbf{TruncSig} & 
        \textbf{.114}, \textbf{7.10} & 
        .152,  28.3 & 
        .199,  82.8 & 
        .232,  97.9 & 
        .242,  99.2 \\
        & \textbf{SDE-GAN} & 
        .201,  72.6 & 
        .334,  99.9 & 
        .407,  100. & 
        .405,  100. & 
        .338,  100. \\
        & \textbf{FDM (ours)} & 
        .124,  13.8 & 
        \textbf{.112}, \textbf{6.30} & 
        \textbf{.118}, \textbf{6.90} & 
        \textbf{.122}, \textbf{9.00} & 
        \textbf{.115}, \textbf{6.20} \\
    \end{tabular}
\end{table}

\begin{table}[ht]
    \centering
    \caption{Average KS test scores and the chance of rejecting the null hypothesis (\%) at 5\%-significance level on marginals for different currency pairs (EUR/USD and USD/JPY), trained on paths evenly sampled at 1024 timestamps. A thread limit error is encountered during the training of the \textbf{SigKer \citep{issa2023sigker}}, which relies on a dedicated parallel PDE solver.}
    \label{tab:ks_forex1024}
    \begin{tabular}{ccccccc}
        Dim & Model & $t=102$ & $t=307$ & $t=512$ & $t=716$ & $t=921$ \\
        \midrule
        \multirow{4}{*}{EUR/USD} 
        & \textbf{SigKer} & 
        - & 
        - & 
        - & 
        - & 
        - \\
        & \textbf{TruncSig} & 
        .476,  100. & 
        .718,  100. & 
        .993,  100. & 
        .996,  100. & 
        .887,  100. \\
        & \textbf{SDE-GAN} & 
        .280,  98.4 & 
        .818,  100. & 
        .963,  100. & 
        .846,  100. & 
        .805,  100. \\
        & \textbf{FDM (ours)} & 
        \textbf{.117}, \textbf{11.1} & 
        \textbf{.117}, \textbf{9.00} & 
        \textbf{.138}, \textbf{25.1} & 
        \textbf{.153}, \textbf{36.2} & 
        \textbf{.191}, \textbf{66.5} \\
        \midrule
        \multirow{4}{*}{USD/JPY} 
        & \textbf{SigKer} & 
        - & 
        - & 
        - & 
        - & 
        - \\
        & \textbf{TruncSig} & 
        .766,  100. & 
        .743,  100. & 
        .670,  100. & 
        .998,  100. & 
        1.00,  100. \\
        & \textbf{SDE-GAN} & 
        .528,  100. & 
        .291,  100. & 
        .389,  100. & 
        .530,  100. & 
        .655,  100. \\
        & \textbf{FDM (ours)} & 
        \textbf{.138}, \textbf{20.9} & 
        \textbf{.124}, \textbf{14.3} & 
        \textbf{.150}, \textbf{32.1} & 
        \textbf{.199}, \textbf{74.9} & 
        \textbf{.260}, \textbf{97.9} \\
    \end{tabular}
\end{table}

\begin{table}[ht]
    \centering
    \caption{Average KS test scores and the chance of rejecting the null hypothesis (\%) at 5\%-significance level on marginals across all dimensions, trained on paths evenly sampled at 64 timestamps from a 16-dimension rough Bergomi model.}
    \label{tab:ks_rbergomi16_avg}
    \begin{tabular}{cccccc}
        Model &  $t=6$ & $t=19$ & $t=32$ & $t=44$ & $t=57$ \\
        \midrule
        \textbf{SigKer} & \textbf{.112}, \textbf{6.60} & .118,  \textbf{7.80} & .124,  10.8 & .132,  16.3 & .144,  25.5 \\
        \textbf{TruncSig} & .450,  100. & .458,  100. & .462,  100. & .461,  100. & .460,  100. \\
        \textbf{SDE-GAN} & .308,  99.8 & .374,  99.4 & .393,  99.5 & .406,  99.6 & .430,  99.7 \\
        \textbf{FDM (ours)} & .113,  7.20 & \textbf{.116}, \textbf{7.80} & \textbf{.119}, \textbf{8.80} & \textbf{.124}, \textbf{11.8} & \textbf{.131}, \textbf{15.8} \\
    \end{tabular}
\end{table}

\begin{table}[ht]
    \centering
    \caption{Average KS test scores and the chance of rejecting the null hypothesis (\%) at 5\%-significance level on marginals across all dimensions, trained on paths evenly sampled at 64 timestamps from a 32-dimension rough Bergomi model. TruncSig runs out of GPU memory.}
    \label{tab:ks_rbergomi32_avg}
    \begin{tabular}{cccccc}
        Model & t=6 & t=19 & t=32 & t=44 & t=57 \\
        \midrule
        \multirow{1}{*}{\textbf{SigKer}} & 
        .120,  11.1 & 
        .137,  18.5 & 
        .149,  26.5 & 
        .157,  35.3 & 
        .168,  45.2 \\
        \textbf{TruncSig} & - & - & - & - & - \\
        \multirow{1}{*}{\textbf{SDE-GAN}} & 
        .284,  99.8 & 
        .288,  99.7 & 
        .298,  99.8 & 
        .311,  99.9 & 
        .326,  100. \\
        \multirow{1}{*}{\textbf{FDM (ours)}} & 
        \textbf{.117}, \textbf{9.10} & 
        \textbf{.119}, \textbf{10.2} & 
        \textbf{.122}, \textbf{11.4} & 
        \textbf{.124}, \textbf{13.0} & 
        \textbf{.128}, \textbf{15.4} \\
    \end{tabular}
\end{table}

\begin{figure}[ht]
\begin{center}
\includegraphics[width=0.8\textwidth]{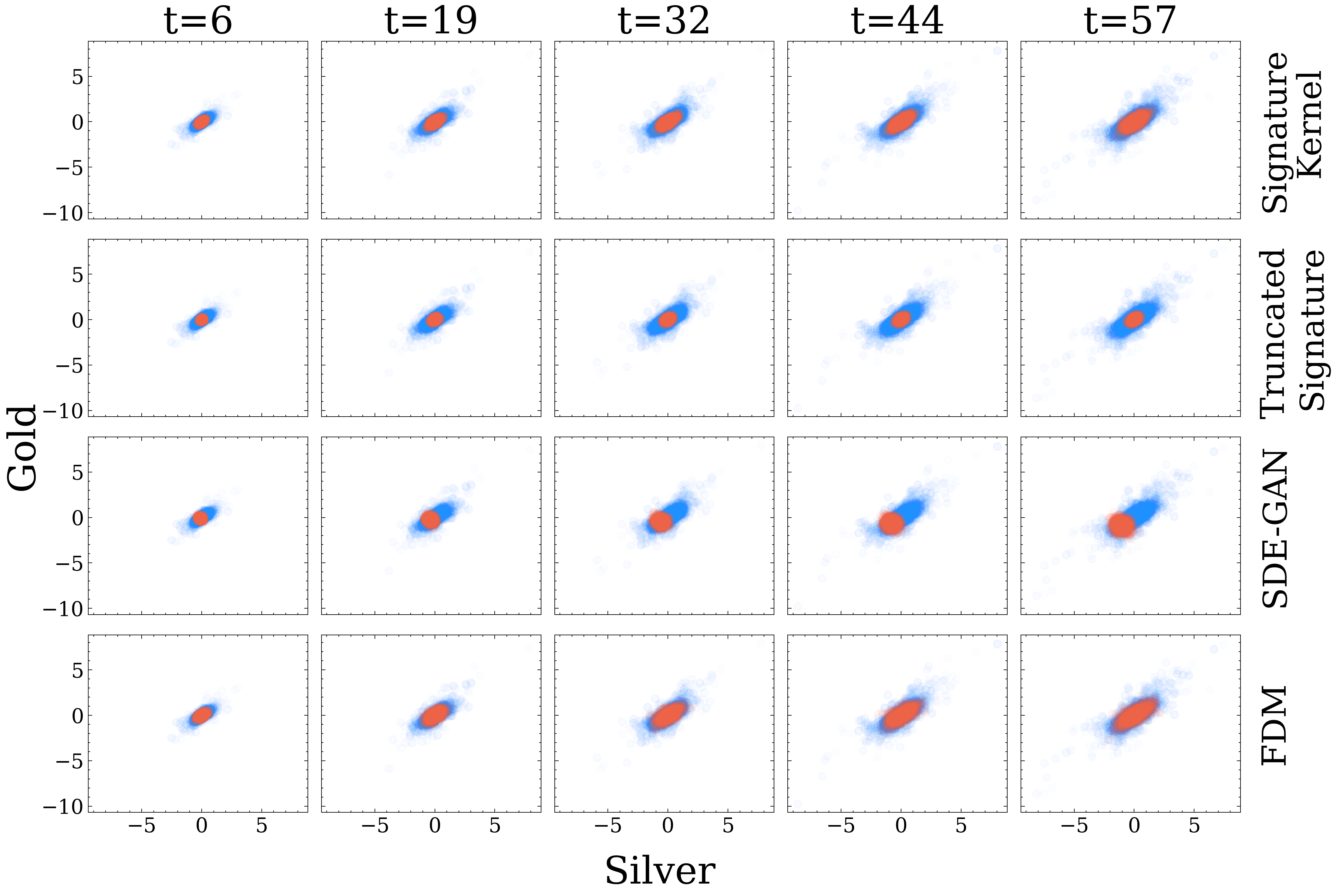} 
\end{center}
\caption{Blue points are real samples and orange points are generated by Neural SDEs. The dynamics of the joint distribution of gold and silver prices in the metal price data. Each row of plots corresponds to a method and each row corresponds to a timestamp. For each plot, the horizontal axis is the silver price and the vertical axis is the gold price. }
\label{fig:forex64_silver_gold_main}
\end{figure}

\begin{figure}[ht]
\begin{center}
\begin{tabular}{c}
    \includegraphics[width=0.7\textwidth]{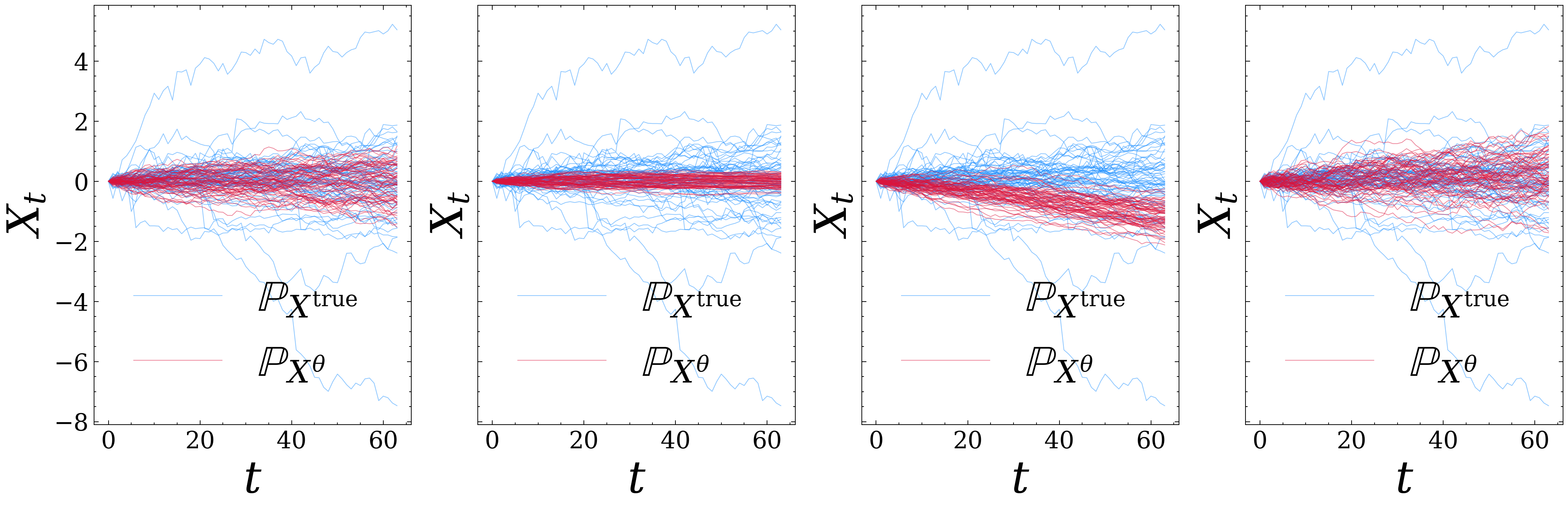} \\
    \includegraphics[width=0.7\textwidth]{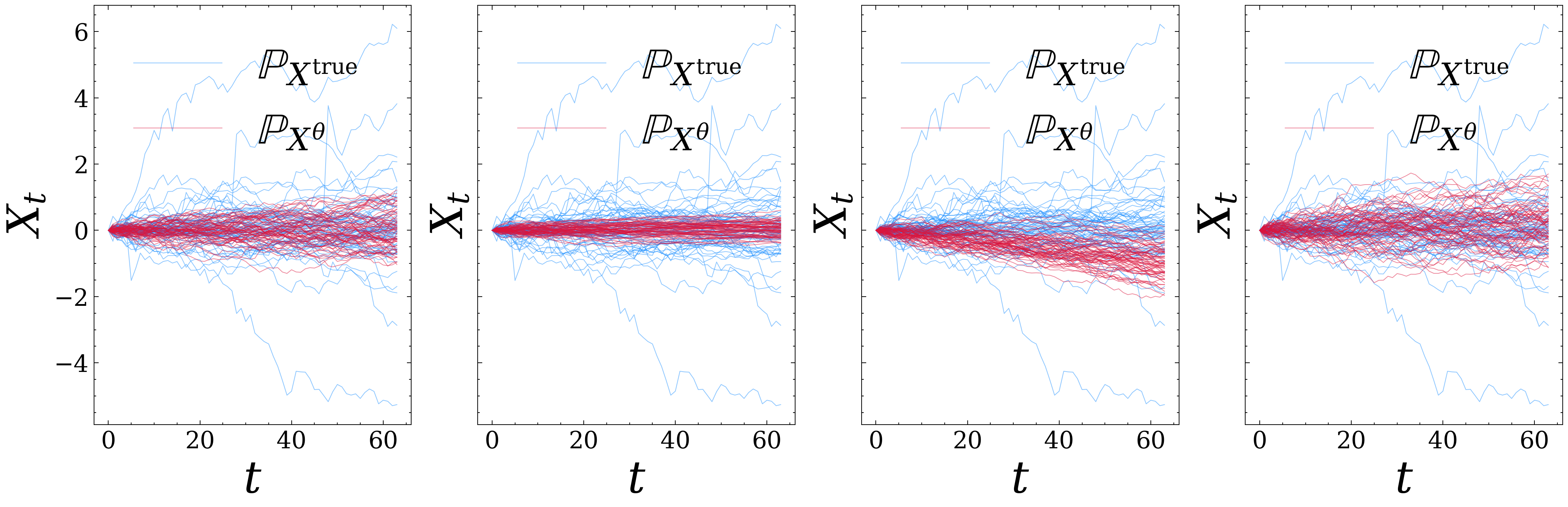} 
\end{tabular}
\caption{Sample paths for silver (top) and gold (bottom) prices from the metal dataset. Blue lines represent real samples, while red lines represent those generated by Neural SDEs. From left to right, the plots correspond to signature kernels, truncated signature, SDE-GAN, and FDM, respectively. The horizontal axis represents time, and the vertical axis represents metal prices.}
\label{fig:paths_metal_combined}
\end{center}
\end{figure}

We evaluate our method, \textbf{FDM}\footnote{code available at \url{https://github.com/Z-Jianxin/FDM}}, by comparing it to three existing methods for training Neural SDEs: the signature kernel method (\textbf{SigKer}, \cite{issa2023sigker}), the truncated signature method (\textbf{TruncSig}, \cite{Bonnier24truncated}), and \textbf{SDE-GAN} \citep{kidger21b}. Our experiments are conducted across five real-world datasets: energy prices, bonds, metal prices, U.S. stock indices, and exchange rates, as well as one synthetic dataset, the Rough Bergomi model\footnote{All real-world datasets are obtained from \url{https://www.dukascopy.com/swiss/english/marketwatch/historical/}}. The real-world datasets are historical price data for variety of financial instruments. The rough Bergomi model is a widely used stochastic volatility model and has been extensively described in \cite{issa2023sigker}. For all datasets, we model all features jointly with a single multi-dimensional Neural SDE.

Consistent with \cite{issa2023sigker}, we use the Kolmogorov-Smirnov (KS) test to assess the marginal distributions for each dimension. Specifically, we compare a batch of generated paths against an unseen batch from the real data distribution and calculate the KS scores and the chance of rejecting the null hypothesis, which states that the two distributions are identical. This process is repeated for all the test batches and we report the averaged KS scores and the chance of rejecting the null hypothesis across all the batches. 


For all experiments, we use fully connected neural networks to parameterize the drift and diffusion terms, with hyperparameters and preprocessings suggested in \cite{issa2023sigker}. We choose $s$ to be $s(P, z)=\frac{1}{2} \EE_{Z,Z'\sim P} k(Z, Z') - \EE_{Z\sim P} k(Z, z)$ where $k$ is the rbf kernel with unit kernel bandwidth. In particular, following \cite{issa2023sigker}, we let our method and \textbf{TruncSig} train for 10000 steps, while \textbf{SDE-GAN} trains for 5000 steps and \textbf{SigKer} for 4000 steps, to normalize the training time. Despite the differences in training steps, our method remains the fastest in terms of wall-clock time. All models are trained and evaluated on a single NVIDIA H100 GPU.

For our experiments, we first follow \cite{issa2023sigker} to train and evaluate the models on three datasets—metal prices, stock indices, and exchange rates—using sequences with 64 timestamps and random train-test splits. This training and evaluation process is repeated with five different random seeds, and the average KS scores and rejection rates are reported in Tables \ref{tab:ks_forex_metals64}, \ref{tab:ks_indices64}, and \ref{tab:ks_forex64}, respectively, with corresponding standard deviations provided in the appendix. For the energy price and bonds datasets, we reserve the latest 20\% of the data for testing, evaluating the trained models via the KS test on generated sequences against unseen future sequences. These results are presented in Tables \ref{tab:ks_energy64} and \ref{tab:ks_bonds64}. We repeat the experiments on these two datasets five times, with standard deviations also reported in the appendix. Additionally, for the exchange rates dataset, we trained and tested the models using sequences with 256 and 1024 timestamps, reporting KS test results in Tables \ref{tab:ks_forex256} and \ref{tab:ks_forex1024}, and training time in Table \ref{tab:training_time} in the appendix. For the synthetic rough Bergomi model, we generated sequences with 64 timestamps across both 16 and 32 dimensions, with results reported in Tables \ref{tab:ks_rbergomi16_avg} and \ref{tab:ks_rbergomi32_avg}, where we report the average KS scores and the chance of rejecting the null hypothesis across different dimensions. We also compare the computational efficiency of the models in terms of training time for different dimensions of the Rough Bergomi model, with detailed results summarized in Table \ref{tab:training_time_rough_bergomi} in the appendix. We highlight the best-performing model across all tables. 

We include qualitative studies in Figure \ref{fig:forex64_silver_gold_main}, which compare the dynamics of joint distributions of real and generated data points for the metal price dataset. We compare the sample paths of the metal price dataset in Figure \ref{fig:paths_metal_combined}. Due to space constraints, we provide further qualitative studies comparing pairwise joint distributions and sample paths, along with tables comparing computational efficiency, in section \ref{sec:appendixexp} in the appendix. Our results demonstrate that our method outperforms competitors in an overwhelming majority of cases in terms of KS test results, qualitative results, and computational efficiency.

\section{Conclusion, Limitations, and Future Work}
\label{sec:conclusion}
Our main theorem demonstrates that any strictly proper scoring rule for comparing distributions on finite dimensions can be extended to strictly proper scoring rules for comparing the laws of continuous Markov processes. This theorem naturally leads to the \textbf{FDM} algorithm for training Neural SDEs. We empirically show that \textbf{FDM} outperforms current state-of-the-art methods for training Neural SDEs, both in terms of generative quality and computational efficiency. However, the applicability of our main theorem is currently constrained by the assumptions of continuity and the Markov property. Although this lies beyond the scope of Neural SDEs, we provide a straightforward extension of the main theorem to Càdlàg Markov processes in the appendix. This extension broadens the applicability of \textbf{FDM} to a wider range of models, including jump processes. Furthermore, an intriguing direction for future work would be to relax the Markov assumptions, for instance, by incorporating hidden Markov models.



\clearpage
\bibliography{refs}
\bibliographystyle{iclr2025_conference}

\clearpage
\appendix
\section{Proof of Theorem \ref{thm:scoring4continuous}}

Suppose random processes $X, Y$ on $\cT$ take values in a Polish space $\cE$ endowed with its Borel $\sigma$-algebra $\cA$ \footnote{$S$ is Borel isomorphic to a Borel set in $[0, 1]$. A Polish space with its Borel $\sigma$-algebra is Borel [p14, Kallenberg]}, let their transition kernels be $\mu^X_{u, v}(X_v, B) = \PP(X_v \in B | X_u)$ and $\mu^Y_{u, v}(Y_v, B) = \PP(Y_v \in B | Y_u)$ for $u, v \in \cT$. For convenience, we use the kernel operations introduced in Chapter 3 of \cite{Kallenberg2021}. Let $B_1, B_2 \in \cA$ and $t, u, v \in \cT$, recall that $\mu^X_{t, u} \otimes \mu^X_{u,v} $ is given by
$$(\mu^X_{t, u} \otimes \mu^X_{u,v})(x, B_1 \times B_2) = \int \mu^X_{t, u}(x, dz_1) \int \mu^X_{u,v}(z_1,dz_2) \mathbbm{1}_{B_1 \times B_2} (z_1, z_2) $$

We need the following lemma to prove the main claim.

\begin{lem}
\label{lem:two_joint_decides_markov}
    Let $\cT$ be an index set. Let $X, Y$ be $\cE$-valued Markov processes on $\cT$. Then $X \overset{d}{=} Y$ $\iff$ $\forall t_1, t_2 \in \cT$, $(X_{t_1}, X_{t_2}) \overset{d}{=} (Y_{t_1}, Y_{t_2})$, where $\overset{d}{=}$ stands for equal in distribution.
\end{lem}
\begin{proof}
    The $\implies$ direction is straightforward; we prove the other direction. Fix $t_1 \leq t_2 \in T$. Since $S$ is Borel, Theorem 8.5 of \cite{Kallenberg2021} implies that the conditional distribution $\mu^X_{t_1, t_2} (z, \cdot) = \mu^Y_{t_1, t_2} (z, \cdot)$ for almost all $z$ under $\PP_{X_{t_1}}$. By Proposition 11.2 of \cite{Kallenberg2021}, for any $t_0 \leq t_1 \dots \leq t_n$ in $T$, 
    \begin{align}
        \PP_{X_{t_0}, X_{t_1}, \dots, X_{t_n}} & = \PP_{X_{t_0}} \otimes \mu^X_{t_0, t_1} \otimes \dots \otimes \mu^X_{t_{n-1}, t_n} \\
        & = \PP_{Y_{t_0}} \otimes \mu^Y_{t_0, t_1} \otimes \dots \otimes \mu^Y_{t_{n-1}, t_n} \\
        & = \PP_{Y_{t_0}, Y_{t_1}, \dots, Y_{t_n}},
    \end{align}
    $i.e.$ $(X_{t_0}, X_{t_1}, \dots, X_{t_n}) \overset{d}{=} (Y_{t_0}, Y_{t_1}, \dots, Y_{t_n})$.
    Then $X \overset{d}{=} Y$ as their finite-dimensional distributions agree.
\end{proof}

Recall that $\cbs{\Omega, \cF, \PP}$ is a probability space where $\Omega, \cF, \PP$ denote the sample space, sigma-algebra, and probability measure, respectively.  Random processes $X, Y$ on $\cT = [0, T]$ take values in a Polish space $\cE$ endowed with its Borel $\sigma$-algebra $\cA$. For a random variable $\xi$, the function $\PP_{\xi} = \PP \circ \xi^{-1}$ is the induced measure on its range space. In particular, for a random process $X$, $\PP_{X}$ denotes its law. Let $s$ be any strictly proper scoring rule defined on $\cE \times \cE$ and $S(P,Q)=\EE_Q [s(P, \omega)] < \infty, \forall$ measures $P, Q$ on $\cE \times \cE$ equipped with $\sigma$-algebra $\cA \otimes \cA$. 

Here we present a more general version of Theorem \ref{thm:scoring4continuous} where $t_1$ and $t_2$ do not need to be uniformly sampled from $\cT$. Let $\mu$ be the Lebesgue measure on $\cT^2$. Let $\nu$ be a measure that is equivalent to $\mu$. That is, there exists the function $\lambda: \cT^2 \to \RR$ such that $\lambda(t_1, t_2) > 0$ $\mu$-$a.e.$ and $\nu(A) = \int_A \lambda(t_1, t_2) d \mu$ for any measurable set $A$. We define the scoring rule $\Bar{s}_{\nu}$ for continuous Markov processes with respect to the sampling measure $\nu$:
\begin{defi}
\label{def:scoringrule_general_process}
    $\Bar{s}_{\nu} (\PP_X, y) = \EE_{(t_1, t_2) \sim \nu} s (\PP_{(X_{t_1}, X_{t_2})}, (y_{t_1}, y_{t_2})),$ where $\PP_{(X_{t_1}, X_{t_2})}$ is the joint marginal distributions at times $t_1, t_2$ of $X$.
\end{defi}
Let $\Bar{S}_{\nu}(\PP_X, \PP_Y) = \EE_{y\sim \PP_Y} [\Bar{s}_{\nu} (\PP_X, y)] $. We present a generalized version of the main statement:

\begin{thm}
\label{thm:scoring4continuous_general}
If $s$ is a strictly proper scoring rule for distributions on $\cE \times \cE$,  $\Bar{s}_{\nu}$ is a strictly proper scoring rule for $\cE$-valued continuous Markov processes on $[0, T]$ where $T \in \RR_{>0}$.  That is, for any $\cE$-valued continuous Markov processes $X, Y$ with laws $\PP_X, \PP_Y$, respectively, $\Bar{S}_{\nu}(\PP_X, \PP_Y) \leq \Bar{S}_{\nu}(\PP_Y, \PP_Y)$ with equality achieved only if $\PP_X = \PP_Y$.
\end{thm}

\begin{proof}[Proof for Theorem \ref{thm:scoring4continuous_general}]
    \begin{align}
        \Bar{S}_{\nu}(\PP_X, \PP_Y) &= \int \EE_{(t_1, t_2) \sim \nu} s (\PP_{(X_{t_1}, X_{t_2})}, (y_{t_1}, y_{t_2})) \quad \PP_Y(dy) \nonumber \\
        & =  \EE_{(t_1, t_2) \sim \nu} \int  s (\PP_{(X_{t_1}, X_{t_2})}, (y_{t_1}, y_{t_2})) \quad \PP_Y(dy) \label{proofmainineq:sec} \\ 
        & = \EE_{(t_1, t_2) \sim \nu} \int  s (\PP_{(X_{t_1}, X_{t_2})}, (y_{t_1}, y_{t_2}) ) \quad \PP_{(Y_{t_1}, Y_{t_2})} (d(y_{t_1}, y_{t_2})) \label{proofmainineq:third} \\ 
        & = \EE_{(t_1, t_2) \sim \nu} S(\PP_{(X_{t_1}, X_{t_2})}, \PP_{(Y_{t_1}, Y_{t_2})}) \label{proofmainineq:fourth} \\ 
        & \leq \EE_{(t_1, t_2) \sim \nu} S(\PP_{(Y_{t_1}, Y_{t_2})}, \PP_{(Y_{t_1}, Y_{t_2})}) \label{proofmainineq:ineqproper}\\ 
        & = \int \EE_{(t_1, t_2) \sim \nu}  s (\PP_{(Y_{t_1}, Y_{t_2})}, (y_{t_1}, y_{t_2}) ) \quad \PP_{(Y_{t_1}, Y_{t_2})} (d(y_{t_1}, y_{t_2})) \label{proofmainineq:fifth}\\ 
        & =  \int \EE_{(t_1, t_2) \sim \nu}  s (\PP_{(Y_{t_1}, Y_{t_2})}, (y_{t_1}, y_{t_2}) ) \quad \PP_Y(dy) \label{proofmainineq:sixth}\\ 
        & = \Bar{S}_{\nu} (\PP_Y, \PP_Y), 
    \end{align}


    We apply Fubini's theorem for the \eqref{proofmainineq:sec} and use the substitution rule (Lemma 1.24, \cite{Kallenberg2021}) \eqref{proofmainineq:third}. \eqref{proofmainineq:fourth} and \eqref{proofmainineq:ineqproper} follow from the definition of $S$ and the properness of the scoring rule $s$, respectively. Fubini's theorem and the substitution rule (Lemma 1.24, \cite{Kallenberg2021}) are used again for the \eqref{proofmainineq:fifth} and \eqref{proofmainineq:sixth}, respectively.
    
    We then show strictness. Let $\Bar{S}_{\nu} (\PP_X, \PP_Y) =\Bar{S}_{\nu} (\PP_Y, \PP_Y) $. Then 
    \begin{align}
         \EE_{(t_1, t_2) \sim \nu} S(\PP_{(X_{t_1}, X_{t_2})}, \PP_{(Y_{t_1}, Y_{t_2})}) &= \EE_{(t_1, t_2) \sim \nu} S(\PP_{(Y_{t_1}, Y_{t_2})}, \PP_{(Y_{t_1}, Y_{t_2})}) \\
        \iff  \EE_{(t_1, t_2) \sim \mu} \lambda(t_1, t_2) S(\PP_{(X_{t_1}, X_{t_2})}, \PP_{(Y_{t_1}, Y_{t_2})}) &= \EE_{(t_1, t_2) \sim \mu} \lambda(t_1, t_2) S(\PP_{(Y_{t_1}, Y_{t_2})}, \PP_{(Y_{t_1}, Y_{t_2})}).
    \end{align}
    
    So $S(\PP_{(X_{t_1}, X_{t_2})}, \PP_{(Y_{t_1}, Y_{t_2})}) = S(\PP_{(Y_{t_1}, Y_{t_2})}, \PP_{(Y_{t_1}, Y_{t_2})})$ $\mu$-$a.e.$. This implies $(X_{t_1}, X_{t_2}) \overset{d}{=} (Y_{t_1}, Y_{t_2})$ $\mu$-$a.e.$. Next, we show that this statement can be extended to all $(t_1, t_2)$.

    Without loss of generality, let $(u_0, u'_0) \in [0, T]^2$ and $u_0 < u'_0$. We can inductively select $u_1, u_2, \dots, u_n, \dots$ and $u'_1, u'_2, \dots, u'_n, \dots$ such that $u_1 \in (u_0, \frac{u_0 + u_0'}{2}]$, $u'_1 \in [\frac{u_0 + u_0'}{2}, u'_0)$, $u_{n+1} \in (u_0, \frac{u_0 + u_n}{2}]$, $u'_{n+1} \in [\frac{u'_n + u_0'}{2}, u'_0)$, and $(X_{u_n}, X_{u'_n}) \overset{d}{=} (Y_{u_n}, Y_{u'_n}) \forall n$.  This is possible because $(u_0, \frac{u_0 + u_n}{2}] \times [\frac{u'_n + u_0'}{2}, u'_0)$ has positive measure. Recall that $X$ and $Y$ are continuous processes. $(X_{u_n}, X_{u'_n})$ converges to $(X_{u_0}, X_{u'_0})$ and $(Y_{u_n}, Y_{u'_n})$ converges to $(Y_{u_0}, Y_{u'_0})$ almost surely as $u_n \to u_0$ and $u'_n \to u'_0$. Recall that $\cE \times \cE$ is also Polish. Then the convergence also holds in distribution and $(X_{u_0}, X_{u'_0}) \overset{d}{=} (Y_{u_0}, Y_{u'_0})$ (Lemma 5.2 and 5.7, \cite{Kallenberg2021}).

    By Lemma \ref{lem:two_joint_decides_markov}, $X \overset{d}{=} Y$.
\end{proof}

Theorem \ref{thm:scoring4continuous} is a straightforward result of Theorem \ref{thm:scoring4continuous_general}:

\begin{proof} [Proof for Theorem \ref{thm:scoring4continuous}]
    Theorem \ref{thm:scoring4continuous} is a direct consequence of Theorem \ref{thm:scoring4continuous_general} by letting $\nu=\mu$.
\end{proof}

\section{Proof of Sample Complexity}
\label{appendix:samplecomplexity}

We'll use McDiarmid's inequality, due to \cite{mcdiarmid1989}.

\begin{thm}
Let \(X_1, X_2, \dots, X_m\) be independent random variables taking values in a set \(\mathcal{X}\). Let \(f: \mathcal{X}^m \to \mathbb{R}\) be a function satisfying the bounded differences condition: for each \(i \in \{1, \dots, m\}\),
\[
\sup_{x_1, \dots, x_m, x_i'} \left| f(x_1, \dots, x_m) - f(x_1, \dots, x_{i-1}, x_i', x_{i+1}, \dots, x_m) \right| \leq \Delta_i,
\]
where \(\Delta_i \geq 0\) are constants. Then, for all \(\varepsilon > 0\),
\[
\mathbb{P}\left( \abs{f(X_1, \dots, X_m) - \mathbb{E}[f(X_1, \dots, X_m)]} \geq \varepsilon \right) \leq 2\exp\left( -\frac{2\varepsilon^2}{\sum_{i=1}^m \Delta_i^2} \right).
\]
\end{thm}

We state and prove the sample complexity results with a general sampling meausre $\nu$. Theorem \ref{thm:samplecomplexity} follows by letting $\nu = \mu$.

\begin{thm}
Let \(\hat{S}(\mathcal{B}_X, \mathcal{B}_Y)\) be the empirical estimator defined as:
\[
\hat{S}(\mathcal{B}_X, \mathcal{B}_Y) = \frac{1}{2B(B-1)} \sum_{i \ne j} k\left( [x^i_{t_j}, x^i_{t'_j}], [x^j_{t_j}, x^j_{t'_j}] \right) - \frac{1}{B^2} \sum_{i=1}^B \sum_{j=1}^B k\left( [x^i_{t_j}, x^i_{t'_j}], [y^j_{t_j}, y^j_{t'_j}] \right),
\]
where:
\begin{itemize}
    \item \(x^i\), \(y^j\) are independently sampled from \(\mathbb{P}_{X}\) and \(\mathbb{P}_Y\), respectively,
    \item $t_j, t'_j$ are independently sampled timestamp pairs, \(\nu\) is a measure equivalent to the Lebesgue measure \(\mu\) on \(\mathcal{T}^2\)
    \item \(k(\cdot, \cdot)\) is a kernel function satisfying \(0 \leq k(\cdot, \cdot) \leq K\).
    \item B $\geq$ 2.
\end{itemize}
 For any \(\varepsilon > 0\),
\[
\PP\left( |\hat{S} - \mathbb{E}[\hat{S}]| \geq \varepsilon \right) \leq 2 \exp\left( -\frac{8B\varepsilon^2}{47 K^2} \right).
\]
 Equivalently, with probability at least \(1 - \delta\), the deviation of \(\hat{S}\) from its expected value \(\bar{S}(\mathbb{P}_{X^\theta}, \mathbb{P}_Y)\) is bounded as:
\[
\big|\,\hat{S}(\mathcal{B}_X, \mathcal{B}_Y) - \bar{S}(\mathbb{P}_{X^\theta}, \mathbb{P}_Y)\,\big| \leq K \sqrt{ \frac{47 \ln(2/\delta)}{8B} }.
\]
\end{thm}

\begin{proof}
We will apply McDiarmid's inequality to the estimator \(\hat{S}\). Recall that \(\nu\) is a measure equivalent to the Lebesgue measure \(\mu\) on \(\mathcal{T}^2\). First, we verify the conditions of the inequality.

The estimator \(\hat{S}\) depends on the independent variables:
\begin{itemize}
    \item \(x^i \in \mathcal{B}_X\): These are the generated paths. Changing a single \(x^i\) while keeping other variables fixed changes \(\hat{S}\) by at most $\frac{(2B-1)K}{2B(B-1)}+\frac{K}{B} \leq \frac{3K}{2B}+\frac{K}{B} = \frac{5K}{2B}$.
    \item \(y^j \in \mathcal{B}_Y\): These are the data paths. Changing a single \(y^j\) while keeping other variables fixed changes \(\hat{S}\) by at most \(\frac{K}{B}\).
    \item \(t_j, t'_j \sim \nu\): These are the timestamps sampled from the measure \(\nu\). Changing a single \(t_j\) or \(t'_j\) while keeping other variables fixed changes \(\hat{S}\) by at most \(\frac{3K}{2B}\).
\end{itemize}

Define the bounded differences:
\[
\Delta_i =
\begin{cases}
\frac{3K}{2B}, & \text{for } i = 1, \dots, 2B, \\
\frac{5K}{2B}, & \text{for } i = 2B+1, \dots, 3B, \\
\frac{K}{B}, & \text{for } i = 3B+1, \dots, 4B.
\end{cases}
\]

The sum of the squared bounded differences is:
\[
\sum_{i=1}^{4B} \Delta_i^2 = 2B \left( \frac{3K}{2B} \right)^2 +B \left( \frac{5K}{2B} \right)^2+ B \left( \frac{K}{B} \right)^2 =  \frac{47K^2}{4B}.
\]

By McDiarmid's inequality, for any \(\varepsilon > 0\):
\[
\PP\left( |\hat{S} - \mathbb{E}[\hat{S}]| \geq \varepsilon \right) \leq 2 \exp\left( -\frac{2\varepsilon^2}{\sum_{i=1}^{4B} \Delta_i^2} \right) = 2 \exp\left( -\frac{8B\varepsilon^2}{47 K^2} \right).
\]

\end{proof}

We analyze the sample complexity for a different estimator where all sample paths are observed at $n$ shared timestamps.
\begin{thm}
\label{thm:samplecomplexity_shared}
Let \(\hat{S}\) be the empirical estimator defined as:
$$
\hat{S} = \frac{1}{n} \sum_{r=1}^n \left[ \frac{1}{2m(m-1)} \sum_{\overset{i,j=1}{i\ne j}}^m k\left( [x^i_{t_r}, x^i_{t'_r}], [x^j_{t_r}, x^j_{t'_r}] \right) - \frac{1}{m^2} \sum_{i=1}^m \sum_{j=1}^m k\left( [x^i_{t_r}, x^i_{t'_r}], [y^j_{t_r}, y^j_{t'_r}] \right) \right],
$$
where:
\begin{itemize}
    \item \(x^i\), \(y^j\) are independently sampled from \(\mathbb{P}_{X}\) and \(\mathbb{P}_Y\), respectively,
    \item \(t_r, t'_r \sim \nu \) are independently sampled timestamp pairs, \(\nu\) is a measure equivalent to the Lebesgue measure \(\mu\) on \(\mathcal{T}^2\)
    \item \(k(\cdot, \cdot)\) is a kernel function satisfying \(0 \leq k(\cdot, \cdot) \leq K\).
    \item m $\geq$ 2.
\end{itemize}
For any \(\varepsilon > 0\),
\[
\PP\left( |\hat{S} - \mathbb{E}[\hat{S}]| \geq \varepsilon \right) \leq  2 \exp\left( -\frac{8mn \varepsilon^2}{K^2 (29n + 18m)} \right).
\]
where \(\mathbb{E}[\hat{S}] = \bar{S}_{\nu}(\mathbb{P}_{X}, \mathbb{P}_Y) = \mathbb{E}_{Y \sim \mathbb{P}_Y}[\bar{s}_{\nu}(\mathbb{P}_{X}, Y)]\).
\end{thm}

\begin{proof}
We will apply McDiarmid's inequality to the estimator \(\hat{S}\). First, we verify the conditions of the inequality.

The estimator \(\hat{S}\) depends on the independent variables:
\begin{itemize}
    \item Generated paths \(x^i\): Changing a single \(x^i\) while keeping other variables fixed changes \(\hat{S}\) by at most $\frac{(2m-1)K}{2m(m-1)}+\frac{K}{m} \leq \frac{3K}{2m}+\frac{K}{m} = \frac{5K}{2m}$.
    \item Data paths \(y^j\): Changing a single \(y^j\) while keeping other variables fixed changes \(\hat{S}\) by at most \(\frac{K}{m}\).
    \item Timestamp pairs \((t_r, t'_r) \sim \nu\): Changing a single timestamp \(t_r\) or \(t'_r\) while keeping other variables fixed changes \(\hat{S}\) by at most \(\frac{3K}{2n}\).
\end{itemize}

Define the bounded differences:
\[
\Delta_i =
\begin{cases}
\frac{5K}{2m}, & \text{for } i = 1, \dots, m, \\
\frac{K}{m}, & \text{for } i = m+1, \dots, 2m, \\
\frac{3K}{2n}, & \text{for } i = 2m+1, \dots, 2m+2n.
\end{cases}
\]

The sum of the squared bounded differences is:
\[
\sum_{i=1}^{2m+2n} \Delta_i^2 = m \left( \frac{5K}{2m} \right)^2 + m \left( \frac{K}{m} \right)^2 + 2n \left( \frac{3K}{2n} \right)^2.
\]

Simplifying each term:
\[
\sum_{i=1}^{2m+2n} \Delta_i^2 = \frac{25K^2}{4m}+ \frac{K^2}{m} + \frac{9K^2}{2n} = K^2 \left( \frac{29}{4m} + \frac{9}{2n} \right).
\]

By McDiarmid's inequality, for any \(\varepsilon > 0\):
\[
\PP\left( |\hat{S} - \mathbb{E}[\hat{S}]| \geq \varepsilon \right) \leq 2 \exp\left( -\frac{2\varepsilon^2}{\sum_{i=1}^{2m+2n} \Delta_i^2} \right) = 2 \exp\left( -\frac{8mn \varepsilon^2}{K^2 (29n + 18m)} \right).
\]


Equivalently, with probability at least \(1 - \delta\):
\[
\big|\,\hat{S} - \mathbb{E}[\hat{S}]\,\big| \leq K \sqrt{ \frac{(29n + 18m)}{8mn} \ln\left( \frac{2}{\delta} \right) }.
\]
\end{proof}
The sample complexity bound established in Theorem \ref{thm:samplecomplexity_shared} demonstrates that the generalization error of the kernel-based scoring rule \(\Bar{s}\) is influenced not only by the number of sample paths \(m\), as is the case for traditional scoring rules where the complexity depends on \(m\) through \(1/\sqrt{m}\), but also by the sampling frequency \(n\) of the timestamp pairs.

\section{Proof of Sensitivity}

We first prove the following lemma to bound the difference process $\Delta_t = X_t - \Tilde{X}_t$.

\begin{lem}
    \label{lemma:diffprocess}
    Let $X$ satisfy $d X_t = \mu(t, X_t)dt + \sigma(t, X_t) dB_t$ on $\RR^d$ for $t\in [0, T]$. Let $\Tilde{X}$ satisfy $d \Tilde{X}_t = \Tilde{\mu}(t, \Tilde{X}_t)dt + \Tilde{\sigma}(t, \Tilde{X}_t) dB_t$ on $\RR^d$ for $t\in [0, T]$ where $\forall t, x, \norm{\mu(t, x) - \Tilde{\mu}(t, x)}_2 \leq \delta_{\mu}$, $\norm{\sigma(t, x) - \Tilde{\sigma}(t, x)}_2 \leq \delta_{\sigma}$, and $\delta_{\mu}$, $\delta_{\sigma}$ are constants. Assume $X$ and $\Tilde{X}$ share the same initial conditons. Assume both $X$ and $\Tilde{X}$ have unique strong solutions so $\mu$ and $\sigma$ are Lipschitz with constant $L_{\mu}$ and $L_{\sigma}$, respectively. Then $\forall t \in [0, T]$,
    $$\EE \norm{\Delta_t}_2^2 \leq \prs{\frac{2D^2}{C}+1}\prs{\delta_{\mu}+\delta_{\sigma}}^2 t e^{\frac{3C}{2}t}, $$
    where $D = \max(1, L_{\sigma})$ and $C = 2L_{\mu} + L_{\sigma}^2$.
\end{lem}

\begin{proof}
    Apply Ito's lemma,
    \begin{align*}
        d \norm{\Delta_t}_2^2 = 2 \innerprod{\Delta_t, \mu(t, X_t)-\Tilde{\mu}(t, \Tilde{X}_t)} dt
        &+\norm{\sigma(t, X_t)-\Tilde{\sigma}(t, \Tilde{X}_t)}_2^2 dt \\
        & + 2 \innerprod{\Delta_t, \sigma(t, X_t)-\Tilde{\sigma}(t, \Tilde{X}_t)} dB_t
    \end{align*}
    Hence,
    \begin{equation}\label{eqn:diffbound1}
        \EE \norm{\Delta_t}_2^2 = 2\EE \brs{ \int_0^t  \innerprod{\Delta_s, \mu(s, X_s)-\Tilde{\mu}(s, \Tilde{X}_s)} ds + \int_0^t \norm{\sigma(s, X_s)-\Tilde{\sigma}(s, \Tilde{X}_s)}_2^2  ds} 
    \end{equation}
    Using the Lipsthitz property of $\mu$ and the bounded difference between $\mu$ and $\Tilde{\mu}$,
    \begin{align}
        \innerprod{\Delta_s, \mu(s, X_s)-\Tilde{\mu}(s, \Tilde{X}_s)} 
        & \leq \norm{\Delta_s}_2 \norm{\mu(s, X_s)- \mu(s, \Tilde{X}_s) + \mu(s, \Tilde{X}_s) -\Tilde{\mu}(s, \Tilde{X}_s)}_2 \\
        &\leq \norm{\Delta_s}_2 \prs{\norm{\mu(s, X_s)- \mu(s, \Tilde{X}_s)}_2 + \norm{ \mu(s, \Tilde{X}_s) -\Tilde{\mu}(s, \Tilde{X}_s)}_2} \\
        & \leq \norm{\Delta_s}_2 \prs{L_{\mu} \norm{\Delta_s}_2 + \delta_{\mu}} \label{ineqn:pmtrick1}
    \end{align}

    Apply the plus-minus trick again, we have
    \begin{equation}\label{ineqn:pmtrick2}
    \norm{\sigma(s, X_s)-\Tilde{\sigma}(s, \Tilde{X}_s)}_2^2 \leq \prs{L_{\sigma} \norm{\Delta_s}_2 + \delta_{\sigma}}^2
    \end{equation}

    Substitute \eqref{ineqn:pmtrick1} and \eqref{ineqn:pmtrick2} back to equation \eqref{eqn:diffbound1} and apply Cauchy-Schwarz on $\EE \norm{\Delta_t}_2$,
    \begin{align}
        \EE \norm{\Delta_t}_2^2 & \leq \int_0^t (2L_{\mu} + L_{\sigma}^2) \EE \norm{\Delta_s}_2^2 + 2\prs{\delta_{\mu} + L_{\sigma}\delta_{\sigma} } \EE \norm{\Delta_s}_2 +  \delta_{\sigma}^2ds \\
        & \leq \int_0^t (2L_{\mu} + L_{\sigma}^2) \EE \norm{\Delta_s}_2^2 + 2\prs{\delta_{\mu} + L_{\sigma}\delta_{\sigma} } \sqrt{\EE \norm{\Delta_s}_2^2} +  \delta_{\sigma}^2ds \\
    \end{align}
    Let $f(t) = \EE \norm{\Delta_t}_2^2, C = 2L_{\mu} + L_{\sigma}^2$ and $D=\max(1, L_{\sigma})$. Then,
    $$f(t) \leq \int_0^t C f(s) + 2D\prs{\delta_{\mu} + \delta_{\sigma} } \sqrt{f(s)} +  \delta_{\sigma}^2ds  $$

    Recall the inequality, $\forall \epsilon > 0, a\in \RR$, $a\sqrt{f(s)} \leq \frac{a^2}{2 \epsilon} + \frac{\epsilon}{2} f(s)$. Let $\epsilon = C$, we have,

    $$f(t) \leq \int_0^t \frac{3C}{2} f(s) + \frac{2D^2 \prs{\delta_{\mu} + \delta_{\sigma} }^2}{C} + \delta_{\sigma}^2ds = \frac{3C}{2} \int_0^t f(s) ds + \prs{\frac{2D^2}{C} + 1}\prs{\delta_{\mu} + \delta_{\sigma} }^2 t.$$

    Apply Gronwall's inequality,
    $$\EE \norm{\Delta_t}_2^2 =  f(t) \leq \prs{\frac{2D^2}{C} + 1}\prs{\delta_{\mu} + \delta_{\sigma} }^2 t e^{\frac{3Ct}{2}}$$
\end{proof}

Now we're ready to prove Theorem \ref{thm:sensitivity}.

\begin{proof}[Proof of Theorem \ref{thm:sensitivity}]
We again work with the general scoring rule $\Bar{s}_{\nu}$ defined in Definition \ref{def:scoringrule_general_process} where the sampling measure $\nu$ can be any measure equivalent to the Lesbegue measure $\mu$.  

Recall that the scoring rule is Lipshtiz with constant $L_s$, with respect to the Wasserstein-2 distance $W_2$, so $\forall z$, and measures $P, P'$, $\abs{s(P,z)-s(P',z)} \leq L_s W_2(P, P')$.

Hence, for any realization $y$,
\begin{align}
\abs{\Bar{s}_{\nu}(\PP(X, y)) - \Bar{s}_{\nu}(\PP(\Tilde{X}, y))} 
&= \abs{\EE_{t_1,t_2\sim \nu}\brs{s(\PP_{(X_{t_1}, X_{t_2})}, y) - s(\PP_{(\Tilde{X}_{t_1}, \Tilde{X}_{t_2})},y)}} \\
& \leq \EE_{t_1,t_2\sim \nu}\abs{s(\PP_{(X_{t_1}, X_{t_2})}, y) - s(\PP_{(\Tilde{X}_{t_1}, \Tilde{X}_{t_2})},y)} \\
& \leq L_s \EE_{t_1,t_2\sim \nu} \brs{W_2 \prs{\PP_{(X_{t_1}, X_{t_2})}, \PP_{(\Tilde{X}_{t_1}, \Tilde{X}_{t_2})}}}
\end{align}

Let $\Gamma(\cdot, \cdot)$ be the couplings of two measures. Then, 

\begin{align}
    & \quad W_2^2 \prs{\PP_{(X_{t_1}, X_{t_2})}, \PP_{(\Tilde{X}_{t_1}, \Tilde{X}_{t_2})}} \\
    &= \inf_{\gamma \in \Gamma\prs{\PP_{(X_{t_1}, X_{t_2})}, \PP_{(\Tilde{X}_{t_1}, \Tilde{X}_{t_2})}}} \EE_{\gamma} \norm{[X_{t_1}, X_{t_2}] - [\Tilde{X}_{t_1}, \Tilde{X}_{t_2}]}_2^2 \label{eqn:w2def}\\
    & = \inf_{\gamma \in \Gamma\prs{\PP_{(X_{t_1}, X_{t_2})}, \PP_{(\Tilde{X}_{t_1}, \Tilde{X}_{t_2})}}} \EE_{\gamma} \norm{X_{t_1}-\Tilde{X}_{t_1}}_2^2 + \norm{X_{t_2}-\Tilde{X}_{t_2}}_2^2 \\
    &\leq \EE \norm{\Delta_{t_1}}_2^2 + \EE \norm{\Delta_{t_2}}_2^2 \\
    &\leq \prs{\frac{2D^2}{C} + 1}\prs{\delta_{\mu} + \delta_{\sigma} }^2 (t_1 e^{\frac{3Ct_1}{2}}+t_2 e^{\frac{3Ct_2}{2}})\label{ineq:lastline},
\end{align}
where \eqref{eqn:w2def} follows the definition of $W_2$, \eqref{ineq:lastline} follows Lemma \ref{lemma:diffprocess}, and $C$, $D$ are defined in Lemma \ref{lemma:diffprocess}.

Finally,
\begin{align}
\abs{\Bar{s}_{\nu}(\PP(X, y)) - \Bar{s}_{\nu}(\PP(\Tilde{X}, y))} 
& \leq L_s \EE_{t_1,t_2\sim \nu} \brs{W_2 \prs{\PP_{(X_{t_1}, X_{t_2})}, \PP_{(\Tilde{X}_{t_1}, \Tilde{X}_{t_2})}}} \\
& \leq L_s  \EE_{t_1,t_2\sim \nu} \brs{\sqrt{t_1 e^{\frac{3Ct_1}{2}}+t_2 e^{\frac{3Ct_2}{2}}}} \sqrt{1+\frac{2D^2}{C}} (\delta_{\mu} + \delta_{\sigma})
\end{align}
The proof is then concluded by renaming the constants.
\end{proof}

\section{Extension to Càdlàg Markov Process}
We show that the proof can be extended to Càdlàg Markov processes, where the paths $t \mapsto X_t$ are right-continuous with left limits everywhere, with probability one. Although this extension goes beyond the scope of Neural SDEs, such processes encompass a wide range of applications. Consider Càdlàg Markov processes $X, Y$ on $\cT' = [0, T)$ that take values in a Polish space $\cE$ endowed with its Borel $\sigma$-algebra $\cA$. Let $s$ be any strictly proper scoring rule defined on $\cE \times \cE$, and let $S(P,Q)=\EE_Q [s(P, \omega)] < \infty$ for all measures $P, Q$ on $\cE \times \cE$ equipped with $\sigma$-algebra $\cA \otimes \cA$.

We generalize Theorem \ref{thm:scoring4continuous_general} to Càdlàg Markov processes. Let $\mu$ denote the Lebesgue measure on $\cT' \times \cT'$, and let $\nu$ be a measure equivalent to $\mu$. We define $\Bar{s}_{\nu}$ as in Definition \ref{def:scoringrule_general_process}. Let $\Bar{S}_{\nu}(\PP_X, \PP_Y) = \EE_{y \sim \PP_Y} [\Bar{s}_{\nu} (\PP_X, y)]$. The main statement is presented below in its Càdlàg form:

\begin{thm}
\label{thm:scoring4cadlag_general}
If $s$ is a strictly proper scoring rule for distributions on $\cE \times \cE$,  $\Bar{s}_{\nu}$ is a strictly proper scoring rule for $\cE$-valued Càdlàg Markov processes on $[0, T)$ where $T \in \RR_{>0}$.  That is, for any $\cE$-valued Càdlàg Markov processes $X, Y$ with laws $\PP_X, \PP_Y$, respectively, $\Bar{S}_{\nu}(\PP_X, \PP_Y) \leq \Bar{S}_{\nu}(\PP_Y, \PP_Y)$ with equality achieved only if $\PP_X = \PP_Y$.
\end{thm}

\begin{proof}
    Following the proof of theorem \ref{thm:scoring4continuous_general}, we can show that $(X_{t_1}, X_{t_2}) \overset{d}{=} (Y_{t_1}, Y_{t_2})$ $\mu$-$a.e.$. We show that this statement can be extended to all $(t_1, t_2) \in \cT' \times \cT'$ using the right continuity.

    Without loss of generality, let $(u_0, u'_0) \in [0, T)^2$ and $u_0 < u'_0$. We can inductively select $u_1, u_2, \dots, u_n, \dots$ and $u'_1, u'_2, \dots, u'_n, \dots$ such that $u_1 \in (u_0, \frac{u_0 + u_0'}{2}]$, $u'_1 \in [u'_0, T)$, $u_{n+1} \in (u_0, \frac{u_0 + u_n}{2}]$, $u'_{n+1} \in [u'_0, \frac{u'_n + u_0'}{2})$, and $(X_{u_n}, X_{u'_n}) \overset{d}{=} (Y_{u_n}, Y_{u'_n}) \forall n$.  This is possible because $(u_0, \frac{u_0 + u_n}{2}] \times [u'_0, \frac{u'_n + u_0'}{2})$ has positive measure. Recall that $X$ and $Y$ are Càdlàg processes. $(X_{u_n}, X_{u'_n})$ converges to $(X_{u_0}, X_{u'_0})$ and $(Y_{u_n}, Y_{u'_n})$ converges to $(Y_{u_0}, Y_{u'_0})$ almost surely as $u_n \to u_0$ and $u'_n \to u'_0$. Recall that $\cE \times \cE$ is also Polish. Then the convergence also holds in distribution and $(X_{u_0}, X_{u'_0}) \overset{d}{=} (Y_{u_0}, Y_{u'_0})$ (Lemma 5.2 and 5.7, \cite{Kallenberg2021}).
\end{proof}

\section{Computational Efficiency}

In this section, we clarify and explain the reduction in computational complexity achieved by our proposed method.

The $O(D^2)$ complexity arises from the previous state-of-the-art Neural SDE training method proposed in \citet{issa2023sigker}, which involves solving a partial differential equation (PDE):
\[
f(s, t) = 1 + \int_{0}^{s} \int_{0}^{t} f(u, v) \langle dx_u, dy_v \rangle_1 dv du,
\]
as shown in Equation (2) of their paper. Backpropagation through the PDE solver introduces significant computational cost.

To approximate the double integral numerically, a rectangular rule with $D$ discretization steps is typically employed:
\[
\int_{0}^{T} \int_{0}^{T} f(u, v) \langle dx_u, dy_v \rangle_1 dv du \approx \sum_{i=1}^{D} \sum_{j=1}^{D} f(u_i, v_j) \langle dx_{u_i}, dy_{v_j} \rangle \Delta u \Delta v,
\]
where $\Delta u = T / D$, $\Delta v = T / D$, and $u_i = i \Delta u$, $v_j = j \Delta v$ for $i, j = 1, \dots, D$. This double sum results in $O(D^2)$ complexity.

Furthermore, their method involves a double sum over the batch size $B$ in the objective function (Equation (4) in their paper). Our $B$ corresponds to their $m$, and the double integral appears in their $k_{sig}$ term. Consequently, their overall complexity is $O(D^2 B^2)$.

Our proposed method reduces the complexity from $O(D^2)$ to $O(D)$, or from $O(D^2 B^2)$ to $O(D B^2)$ when considering the batch size. This improvement is achieved because our approach eliminates the need to solve the PDE with the double integral, avoiding the computationally expensive operations required by the previous method.

\section{Alternative Empirical Objectives}

\textbf{Empirical Objective: Multiple Observations Concatenated.}
An unbiased estimator can be constructed using batches of generated paths $\mathcal{B}_X = \{ x^i \}_{i=1}^B$ and data paths $\mathcal{B}_Y = \{ y^i \}_{i=1}^B$, where each path is observed at more than two timestamps. Suppose for each $i$, we select multiple (potentially irregular) observations at timestamps $t_i^1, t_i^2, \ldots, t_i^N$ where $N$ itself can be random or a tuning parameter. We concatenate these multiple observations to form vectors $[x^i_{t_i^1}, x^i_{t_i^2}, \ldots, x^i_{t_i^N}]$. The empirical estimator is then given by:
\begin{align}
    \hat{S}_1(\mathcal{B}_X, \mathcal{B}_Y) =
    & \frac{1}{2B(B-1)} \sum_{i \ne j} k\left( [x^i_{t_j^1}, \ldots, x^i_{t_j^N}],\ [x^j_{t_j^1}, \ldots, x^j_{t_j^N}] \right) \\
    - & \frac{1}{B^2} \sum_{i=1}^B \sum_{j=1}^B k\left( [x^i_{t_j^1}, \ldots, x^i_{t_j^N}],\ [y^j_{t_j^1}, \ldots, y^j_{t_j^N}] \right).
\end{align}

\textbf{Empirical Objective: Adjacent Timestamps as IID Samples.}
Alternatively, we consider every pair of adjacent timestamps as independent and identically distributed (i.i.d.) samples. Suppose each data path is observed at timestamps  $t^1_i < t^2_i < \ldots < t^{M}_i$. For each pair of adjacent timestamps $(t^m_i, t^{m+1}_i)$, we treat the (potentially irregular) observations as i.i.d.\ samples. The empirical estimator is then:
\begin{align}
    \hat{S}_2(\mathcal{B}_X, \mathcal{B}_Y) =
    & \frac{1}{2B(M-1)(B-1)} \sum_{m=1}^{M-1} \sum_{i \ne j} k\left( [x^i_{t^m_j}, x^i_{t^{m+1}_j}],\ [x^j_{t^m_j}, x^j_{t^{m+1}_j}] \right) \\
    - & \frac{1}{B^2(M-1)} \sum_{m=1}^{M-1} \sum_{i=1}^B \sum_{j=1}^B k\left( [x^i_{t^m_j}, x^i_{t^{m+1}_j}],\ [y^j_{t^m_j}, y^j_{t^{m+1}_j}] \right).
\end{align}
Note that both estimators only require each data path to be observed at multiple timestamps, which can be irregular and path-dependent. All three empirical objectives, including the one presented in the main paper, perform similarly well in our preliminary experiments.

\section{Additional Experimental Results}
\label{sec:appendixexp}
The standard deviations on metal prices, stock indices, exchange rates, energy price, and bonds datasets are reported in Tables \ref{tab:ks_metals64_std}, \ref{tab:ks_indices64_std}, \ref{tab:ks_forex64_std}, \ref{tab:ks_energy64_std}, \ref{tab:ks_bonds64_std}, respectively.

We evaluate the computational efficiency of the models by comparing their training times across different numbers of timestamps for the exchange rates dataset, with the detailed results presented in Table~\ref{tab:training_time}. Additionally, we assess the training times for various dimensions of the Rough Bergomi model, with the corresponding results summarized in Table~\ref{tab:training_time_rough_bergomi}.


\begin{table}[ht]
    \centering
    \caption{Standard deviations of average KS test scores and chance of rejecting the null hypothesis (\%) at 5\%-significance level on marginals of metal prices, trained on paths evenly sampled at 64 timestamps.}
    \label{tab:ks_metals64_std}
    \begin{tabular}{ccccccc}
        Dim & Model & $t=6$ & $t=19$ & $t=32$ & $t=44$ & $t=57$ \\
        \hline
        \multirow{4}{*}{SILVER} 
        & \textbf{SigKer} & .020, 20.0 & .013, 9.26 & .008, 3.96 & .006, 2.47 & .006, 2.37 \\
        & \textbf{TruncSig} & .080, 1.27 & .081, .890 & .079, .500 & .078, .320 & .073, .360 \\
        & \textbf{SDE-GAN} & .190, 41.0 & .167, .000 & .243, .000 & .240, .000 & .236, .000 \\
        & \textbf{FDM (ours)} & .006, 4.12 & .003, 1.63 & .002, .640 & .003, 1.18 & .004, 1.59 \\
        \hline
        \multirow{4}{*}{GOLD} 
        & \textbf{SigKer} & .004, 1.95 & .004, 1.85 & .005, 2.13 & .006, 2.66 & .006, 2.53 \\
        & \textbf{TruncSig} & .039, 9.49 & .039, 2.62 & .038, .510 & .034, .28 & .034, .130 \\
        & \textbf{SDE-GAN} & .033, 12.0 & .072, 10.5 & .100, 6.20 & .131, 5.36 & .194, 8.21 \\
        & \textbf{FDM (ours)} & .002, .900 & .002, .720 & .003, 1.05 & .004, 1.65 & .005, 2.01 \\
    \end{tabular}
\end{table}

\begin{table}[ht]
    \centering
    \caption{Standard deviations of average KS test scores and chance of rejecting the null hypothesis (\%) at 5\%-significance level on marginals of U.S. stock indices, trained on paths evenly sampled at 64 timestamps. "DOLLAR", "USA30", "USA500", "USATECH", and "USSC2000" stand for US Dollar Index, USA 30 Index, USA 500 Index, USA 100 Technical Index, and US Small Cap 2000, respectively.}
    \label{tab:ks_indices64_std}
    \begin{tabular}{ccccccc}
        Dim & Model & t=6 & t=19 & t=32 & t=44 & t=57 \\
        \midrule
        \multirow{4}{*}{DOLLAR} 
        & \textbf{SigKer} & .105, 30.9 & .168, 23.3 & .184, 20.0 & .186, 17.3 & .185, 14.7 \\
        & \textbf{TruncSig} & .081, 1.90 & .075, .365 & .067, .224 & .063, .134 & .058, .130 \\
        & \textbf{SDE-GAN} & .144, 15.4 & .218, 3.85 & .264, 1.07 & .256, .447 & .273, .346 \\
        & \textbf{FDM (ours)} & .032, 30.6 & .036, 33.9 & .037, 34.9 & .038, 35.3 & .038, 35.1 \\
        \midrule
        \multirow{4}{*}{USA30} 
        & \textbf{SigKer} & .081, 31.7 & .106, 23.3 & .106, 6.00 & .110, 5.36 & .116, 6.09 \\
        & \textbf{TruncSig} & .050, 42.2 & .058, 46.6 & .056, 38.5 & .055, 29.3 & .054, 16.2 \\
        & \textbf{SDE-GAN} & .205, 29.2 & .328, 10.3 & .335, 17.6 & .266, .179 & .229, .045 \\
        & \textbf{FDM (ours)} & .008, 4.69 & .004, 2.76 & .003, 2.17 & .002, 1.74 & .001, 1.20 \\
        \midrule
        \multirow{4}{*}{USA500} 
        & \textbf{SigKer} & .124, 16.2 & .191, 9.24 & .209, 6.53 & .215, 6.66 & .210, 7.04 \\
        & \textbf{TruncSig} & .064, 43.9 & .070, 43.9 & .070, 38.7 & .067, 33.4 & .061, 22.6 \\
        & \textbf{SDE-GAN} & .067, 4.03 & .246, 14.7 & .235, .045 & .252, .000 & .236, .000 \\
        & \textbf{FDM (ours)} & .005, 3.23 & .002, 1.17 & .001, .740 & .001, .682 & .001, .789 \\
        \midrule
        \multirow{4}{*}{USA1000} 
        & \textbf{SigKer} & .030, 11.8 & .052, 7.61 & .059, 7.83 & .062, 7.47 & .057, 6.84 \\
        & \textbf{TruncSig} & .034, 34.1 & .037, 27.0 & .039, 17.6 & .042, 10.2 & .041, 4.54 \\
        & \textbf{SDE-GAN} & .056, .141 & .102, .000 & .090, .000 & .071, .000 & .056, .000 \\
        & \textbf{FDM (ours)} & .003, 1.75 & .001, .973 & .000, .554 & .001, .603 & .001, .598 \\
        \midrule
        \multirow{4}{*}{USA2000} 
        & \textbf{SigKer} & .115, 27.8 & .164, 11.5 & .175, 5.85 & .183, 5.87 & .190, 5.90 \\
        & \textbf{TruncSig} & .062, 37.2 & .060, 35.4 & .058, 26.4 & .051, 13.7 & .046, 4.13 \\
        & \textbf{SDE-GAN} & .181, 32.4 & .296, 2.99 & .250, .045 & .203, .000 & .156, .000 \\
        & \textbf{FDM (ours)} & .007, 4.22 & .005, 2.90 & .003, 1.63 & .002, 1.51 & .002, 1.48 \\
    \end{tabular}
\end{table}

\begin{table}[ht]
    \centering
    \caption{Standard deviations of average KS test scores and the chance of rejecting the null hypothesis (\%) at 5\%-significance level on marginals for different currency pairs (EUR/USD and USD/JPY), trained on paths evenly sampled at 64 timestamps.}
    \label{tab:ks_forex64_std}
    \begin{tabular}{ccccccc}
        Dim & Model & t=6 & t=19 & t=32 & t=44 & t=57 \\
        \midrule
        \multirow{4}{*}{EUR/USD}
        & \textbf{SigKer} & .134, 44.3 & .190, 39.8 & .215, 39.3 & .231, 39.1 & .239, 36.6 \\
        & \textbf{TruncSig} & .078, 1.97 & .071, .230 & .066, .148 & .061, .152 & .054, .089 \\
        & \textbf{SDE-GAN} & .218, 24.2 & .277, 9.44 & .309, 8.90 & .312, 5.23 & .287, .447 \\
        & \textbf{FDM (ours)} & {.007}, {4.91} & {.002}, {1.03} & {.001}, {0.627} & {.002}, {0.789} & {.002}, {0.852} \\
        \midrule
        \multirow{4}{*}{USD/JPY}
        & \textbf{SigKer} & .066, 37.4 & .106, 38.3 & .126, 38.7 & .138, 39.8 & .138, 39.9 \\
        & \textbf{TruncSig} & .063, 19.7 & .062, 3.69 & .061, .841 & .056, .378 & .052, .173 \\
        & \textbf{SDE-GAN} & .051, 31.2 & .092, 25.3 & .147, 18.3 & .190, 16.4 & .237, 13.5 \\
        & \textbf{FDM (ours)} & {.003}, {1.51} & {.001}, {0.428} & {.001}, {0.410} & {.001}, {0.349} & {.002}, {0.522} \\
    \end{tabular}
\end{table}

\begin{table}[ht]
    \centering
    \caption{Standard deviations of average KS test scores and the chance of rejecting the null hypothesis (\%) at 5\%-significance level on marginals of energy prices, trained on paths evenly sampled at 64 timestamps. We reserve the latest 20\% data as test dataset and measure how well the model predicts into future. "BRENT", "DIESEL", "GAS", and "LIGHT" stand for U.S. Brent Crude Oil, Gas oil, Natural Gas, and U.S. Light Crude Oil, respectively.}
    \label{tab:ks_energy64_std}
    \begin{tabular}{ccccccc}
        Dim & Model & t=6 & t=19 & t=32 & t=44 & t=57 \\
        \midrule
        \multirow{4}{*}{BRENT}
        & \textbf{SigKer} & .143, 41.6 & .217, 39.9 & .242, 42.4 & .238, 43.4 & .231, 44.1 \\
        & \textbf{TruncSig} & .086, 5.70 & .087, 4.86 & .085, 4.40 & .086, 2.94 & .080, 1.94 \\
        & \textbf{SDE-GAN} & .195, 6.65 & .133, .000 & .104, .000 & .068, .000 & .036, .000 \\
        & \textbf{FDM (ours)} & {.007}, {2.77} & {.006}, {2.27} & {.007}, {4.08} & {.007}, {3.66} & {.006}, {2.87} \\
        \midrule
        \multirow{4}{*}{DIESEL} 
        & \textbf{SigKer} & .067, 31.0 & .112, 26.3 & .139, 28.5 & .158, 32.9 & .170, 35.9 \\
        & \textbf{TruncSig} & .044, 21.3 & .044, 7.02 & .043, 2.22 & .042, .971 & .046, .404 \\
        & \textbf{SDE-GAN} & .146, 28.7 & .277, 4.69 & .333, .303 & .302, .045 & .262, .000 \\
        & \textbf{FDM (ours)} & {.003}, {1.25} & {.002}, {1.23} & {.002}, {1.29} & {.002}, {1.44} & {.003}, {1.47} \\
        \midrule
        \multirow{4}{*}{GAS}
        & \textbf{SigKer} & .099, 27.5 & .157, 23.6 & .179, 30.5 & .182, 33.3 & .189, 36.6 \\
        & \textbf{TruncSig} & .071, 32.8 & .069, 13.3 & .064, 4.65 & .063, .952 & .057, .311 \\
        & \textbf{SDE-GAN} & .176, 23.2 & .290, .358 & .302, .268 & .251, .045 & .182, .000 \\
        & \textbf{FDM (ours)} & {.002}, {0.99} & {.004}, {1.77} & {.004}, {2.12} & {.005}, {2.94} & {.006}, {3.67} \\
        \midrule
        \multirow{4}{*}{LIGHT} 
        & \textbf{SigKer} & .017, 18.6 & .039, 26.4 & .050, 32.1 & .048, 32.7 & .051, 34.4 \\
        & \textbf{TruncSig} & .038, 16.5 & .040, 10.9 & .041, 6.92 & .042, 2.40 & .042, .728 \\
        & \textbf{SDE-GAN} & .125, 30.4 & .349, 31.6 & .338, 21.2 & .268, 3.18 & .268, .179 \\
        & \textbf{FDM (ours)} & {.003}, {1.38} & {.003}, {1.42} & {.004}, {2.05} & {.006}, {3.30} & {.008}, {4.97} \\
    \end{tabular}
\end{table}

\begin{table}[ht]
    \centering
    \caption{Standard deviations of average KS test scores and chance of rejecting the null hypothesis (\%) at 5\%-significance level on marginals of bonds, trained on paths evenly sampled at 64 timestamps. We reserve the most latest 20\% data as test dataset and measure how well the model predicts into future. "BUND", "UKGILT", and "USTBOND" stand for Euro Bund, UK Long Gilt, and US T-BOND, respectively.}
    \label{tab:ks_bonds64_std}
    \begin{tabular}{ccccccc}
        Dim & Model & t=6 & t=19 & t=32 & t=44 & t=57 \\
        \midrule
\multirow{4}{*}{BUND} 
& \textbf{SigKer} & .126, 37.5 & .181, 32.6 & .190, 27.7 & .183, 24.7 & .179, 27.0 \\
& \textbf{TruncSig} & .077, 3.95 & .072, .523 & .068, .251 & .063, .130 & .054, .130 \\
& \textbf{SDE-GAN} & .090, 3.75 & .237, .045 & .282, 1.48 & .202, .000 & .218, .000 \\
& \textbf{FDM (ours)} & {.008}, {4.08} & {.012}, {4.78} & {.008}, {3.56} & {.006}, {2.38} & {.007}, {2.60} \\
\midrule
       \multirow{4}{*}{UKGILT} 
& \textbf{SigKer} & .094, 41.1 & .135, 40.7 & .146, 39.3 & .150, 38.5 & .152, 38.2 \\
& \textbf{TruncSig} & .050, 37.7 & .047, 16.5 & .043, 1.58 & .039, .336 & .037, .182 \\
& \textbf{SDE-GAN} & .232, 30.5 & .303, .394 & .212, .000 & .188, .000 & .156, .000 \\
& \textbf{FDM (ours)} & {.003}, {1.61} & {.002}, {.606} & {.001}, {0.691} & {.001}, {.611} & {.002}, {.962} \\
\midrule
        \multirow{4}{*}{USTBOND} 
& \textbf{SigKer} & .096, 38.1 & .146, 39.4 & .159, 36.3 & .160, 36.2 & .158, 35.9 \\
& \textbf{TruncSig} & .077, 37.7 & .076, 30.1 & .071, 9.59 & .066, 1.18 & .057, .261 \\
& \textbf{SDE-GAN} & .162, 32.3 & .238, .045 & .207, .000 & .197, .000 & .188, .000 \\
& \textbf{FDM (ours)} & {.006}, {4.46} & {.003}, {1.40} & {.001}, {0.583} & {.001}, {0.560} & {.001}, {.669} \\
    \end{tabular}
\end{table}

\begin{table}[ht]
\centering
\caption{Training time of different methods on forex data with different lengths in terms of hours. SDE-GAN hits the max wall times of 20 hours while the training progress is nearly 25\%.}
\label{tab:training_time}
\begin{tabular}{@{}lccc@{}}
\textbf{Method} & \textbf{64 Timestamps} & \textbf{256 Timestamps} & \textbf{1024 Timestamps} \\ \midrule
Signature Kernel & 0.66 & 7.80 & thread limit error \\
Truncated Signature & 0.31 & 1.34 & 5.61 \\
SDE-GAN & 0.64 & 4.21 & $>$ 80 \\ 
FDM (ours) & \textbf{0.27} & \textbf{1.21} & \textbf{5.43} \\
\end{tabular}
\end{table}

\begin{table}[ht]
\centering
\caption{Training time of Rough Bergomi model with different data dimensions in terms of hours.}
\label{tab:training_time_rough_bergomi}
\begin{tabular}{@{}lcc@{}}
\toprule
\textbf{Method} & \textbf{16 Dim} & \textbf{32 Dim} \\ \midrule
SDE-GAN & 1.41 & 1.58 \\
FDM (ours) & \textbf{0.40} & \textbf{0.54} \\
Signature Kernel & 4.11 & 6.74 \\
Truncated Signature & 6.86 & GPU out of RAM \\ \bottomrule
\end{tabular}
\end{table}

We present additional qualitative results showing the pairwise joint dynamics generated by models trained on different datasets. Results for the metal price dataset are shown in Figure \ref{fig:forex64_silver_gold}. Results for the U.S. stock indices dataset are presented in Figures \ref{fig:indices64_DOLLAR_USA30}, \ref{fig:indices64_DOLLAR_USA500}, \ref{fig:indices64_DOLLAR_USATECH}, \ref{fig:indices64_DOLLAR_USSC2000}, \ref{fig:indices64_USA30_USA500}, \ref{fig:indices64_USA30_USATECH}, \ref{fig:indices64_USA30_USSC2000}, \ref{fig:indices64_USA500_USATECH}, \ref{fig:indices64_USA500_USSC2000}, and \ref{fig:indices64_USATECH_USSC2000}. Results for the exchange rates data are presented in Figure \ref{fig:forex64_EURUSD_USDJPY}. Results for the energy price data are shown in Figures \ref{fig:energy64_BRENT_DIESEL}, \ref{fig:energy64_BRENT_GAS}, \ref{fig:energy64_BRENT_LIGHT}, \ref{fig:energy64_DIESEL_GAS}, \ref{fig:energy64_DIESEL_LIGHT}, and \ref{fig:energy64_GAS_LIGHT}. Finally, results for the bonds data are presented in Figures \ref{fig:bonds64_BUND_UKGILT}, \ref{fig:bonds64_BUND_USTBOND}, and \ref{fig:bonds64_UKGILT_USTBOND}.

\begin{figure}[ht]
\begin{center}
\includegraphics[width=0.9\textwidth]{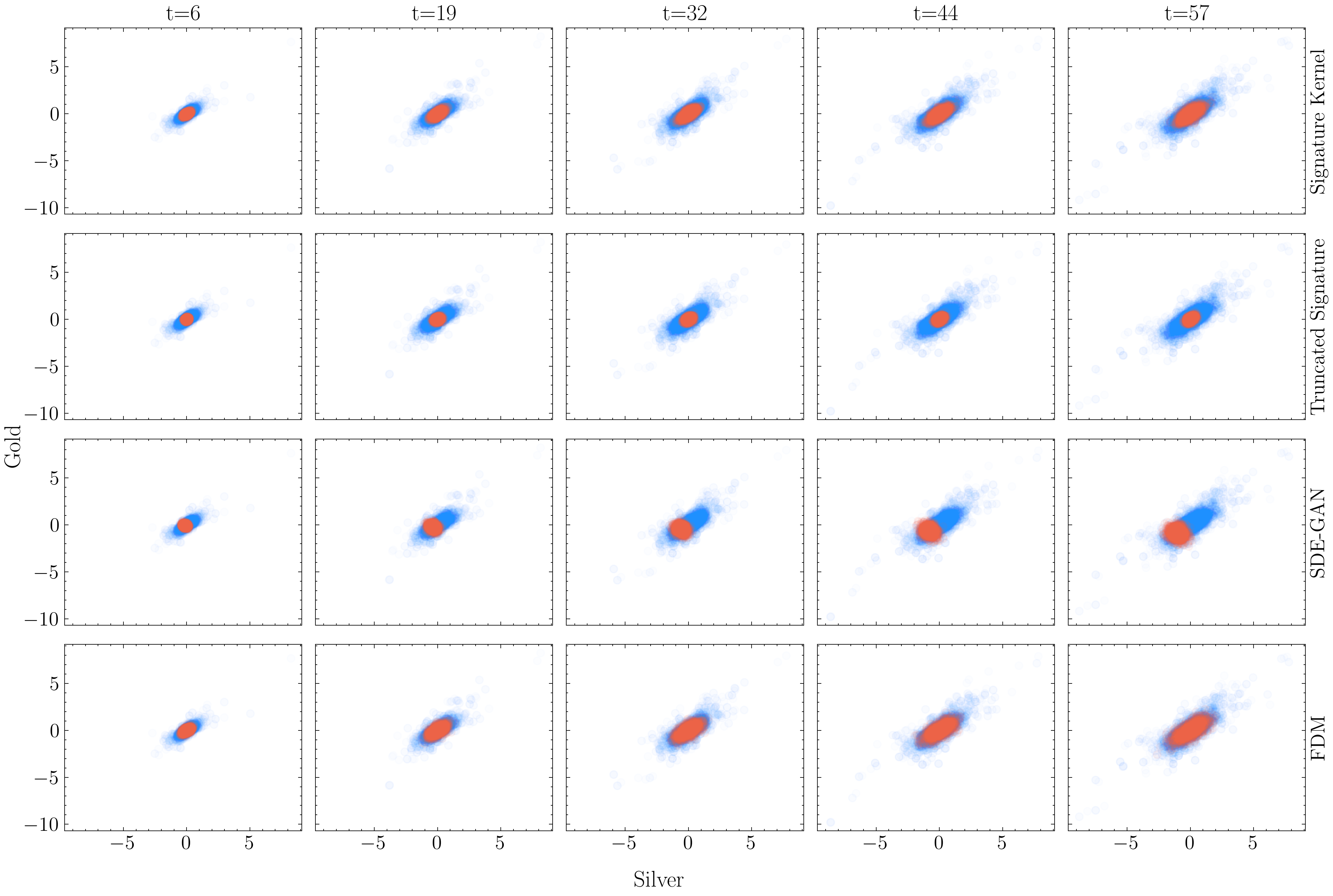} 
\end{center}
\caption{Blue points are real samples and orange points are generated by Neural SDEs. The dynamics of the joint distribution of gold and silver prices in the metal price data. Each row of plots corresponds to a method and each row corresponds to a timestamp. For each plot, the horizontal axis is the silver price and the vertical axis is the gold price. }
\label{fig:forex64_silver_gold}
\end{figure}

\begin{figure}[ht]
\begin{center}
\includegraphics[width=0.9\textwidth]{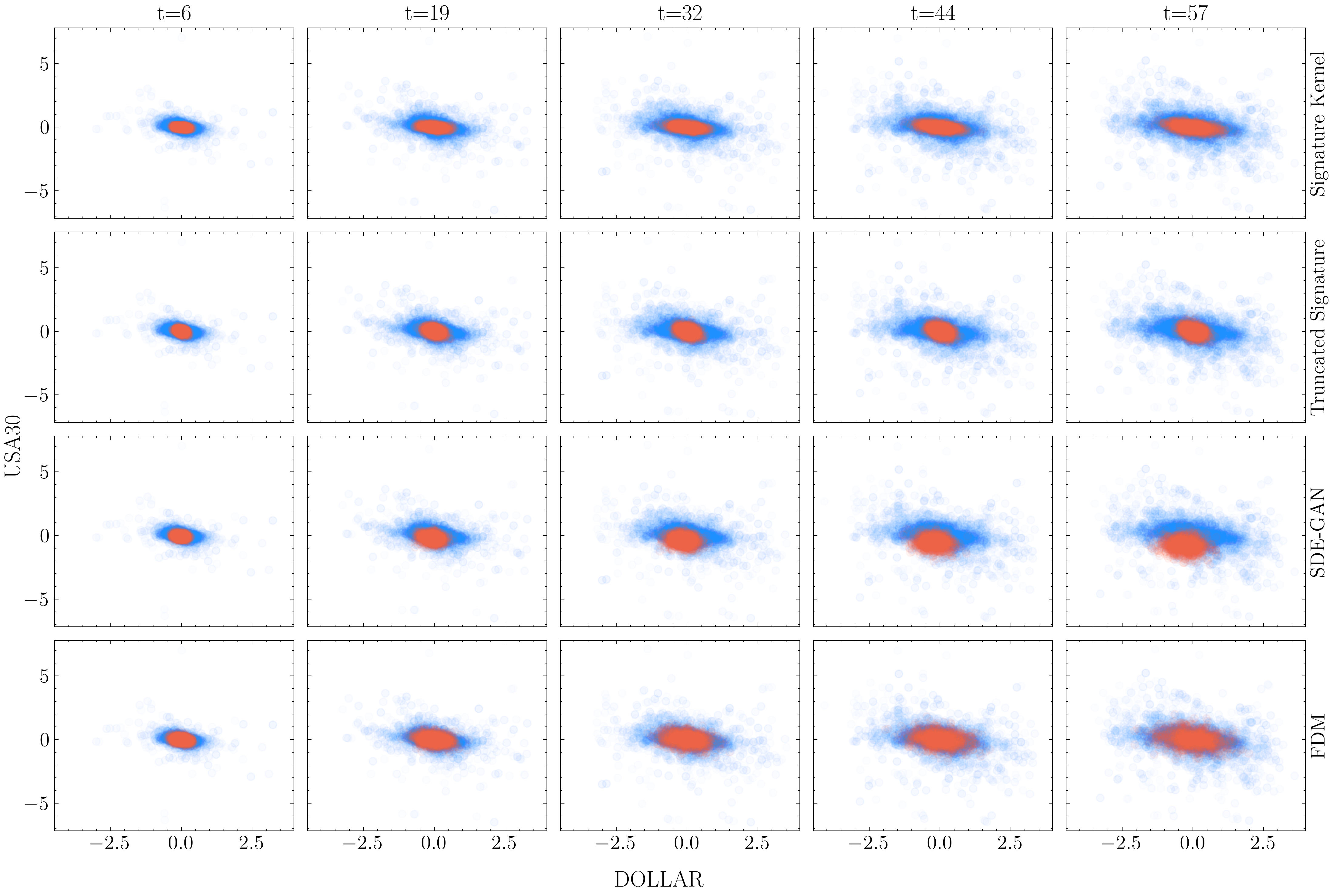} 
\end{center}
\caption{Blue points are real samples and orange points are generated by Neural SDEs. The dynamics of the joint distribution of Dollar and USA30 in the U.S. stock indices data. Each row of plots corresponds to a method and each row corresponds to a timestamp. For each plot, the horizontal axis is Dollar (US Dollar Index) and the vertical axis is USA30 (USA 30 Index). }
\label{fig:indices64_DOLLAR_USA30}
\end{figure}

\begin{figure}[ht]
\begin{center}
\includegraphics[width=0.9\textwidth]{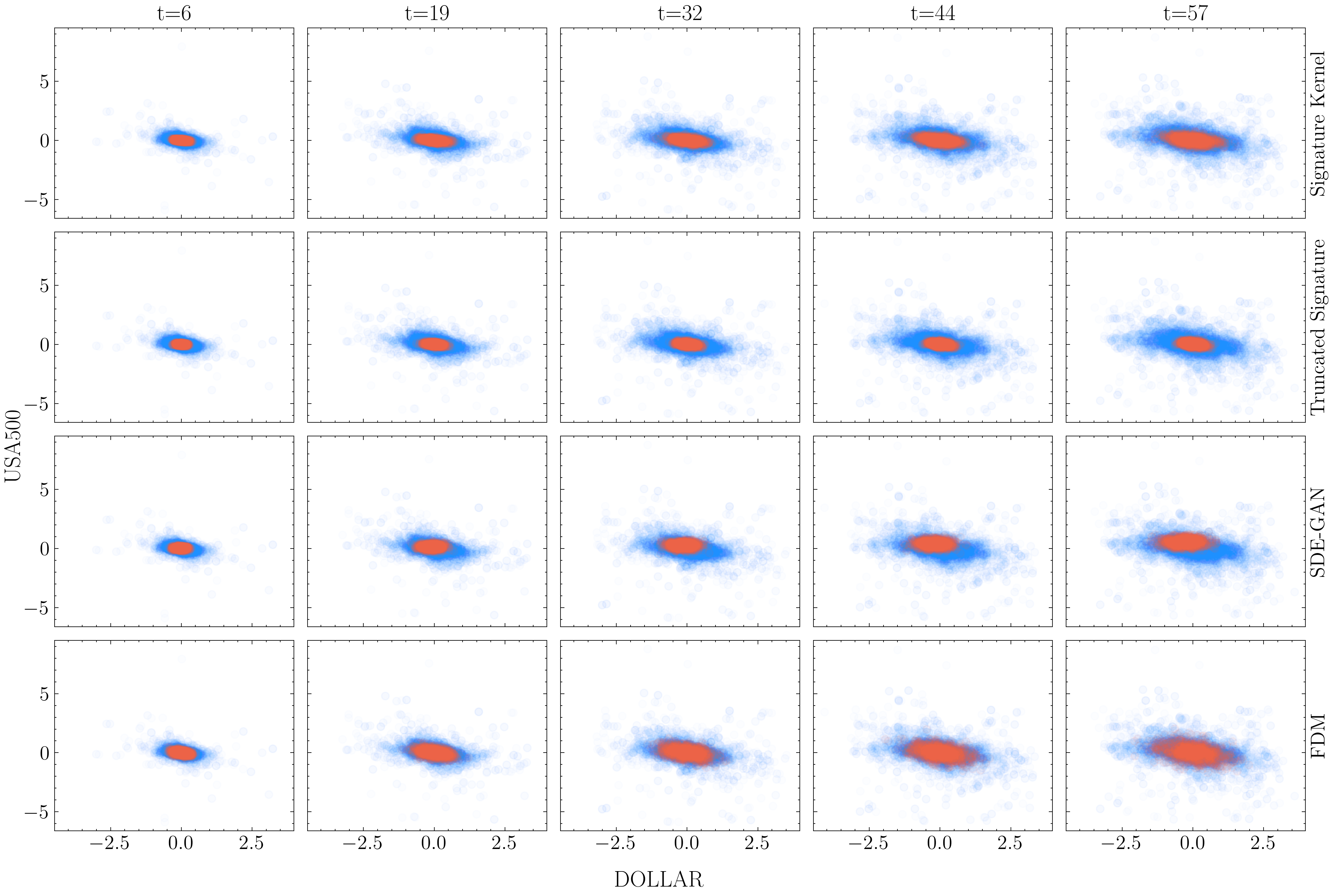} 
\end{center}
\caption{Blue points are real samples and orange points are generated by Neural SDEs. The dynamics of the joint distribution of Dollar and USA500 in the U.S. stock indices data. Each row of plots corresponds to a method and each row corresponds to a timestamp. For each plot, the horizontal axis is Dollar (US Dollar Index) and the vertical axis is USA500 (USA 500 Index).}
\label{fig:indices64_DOLLAR_USA500}
\end{figure}

\begin{figure}[ht]
\begin{center}
\includegraphics[width=0.9\textwidth]{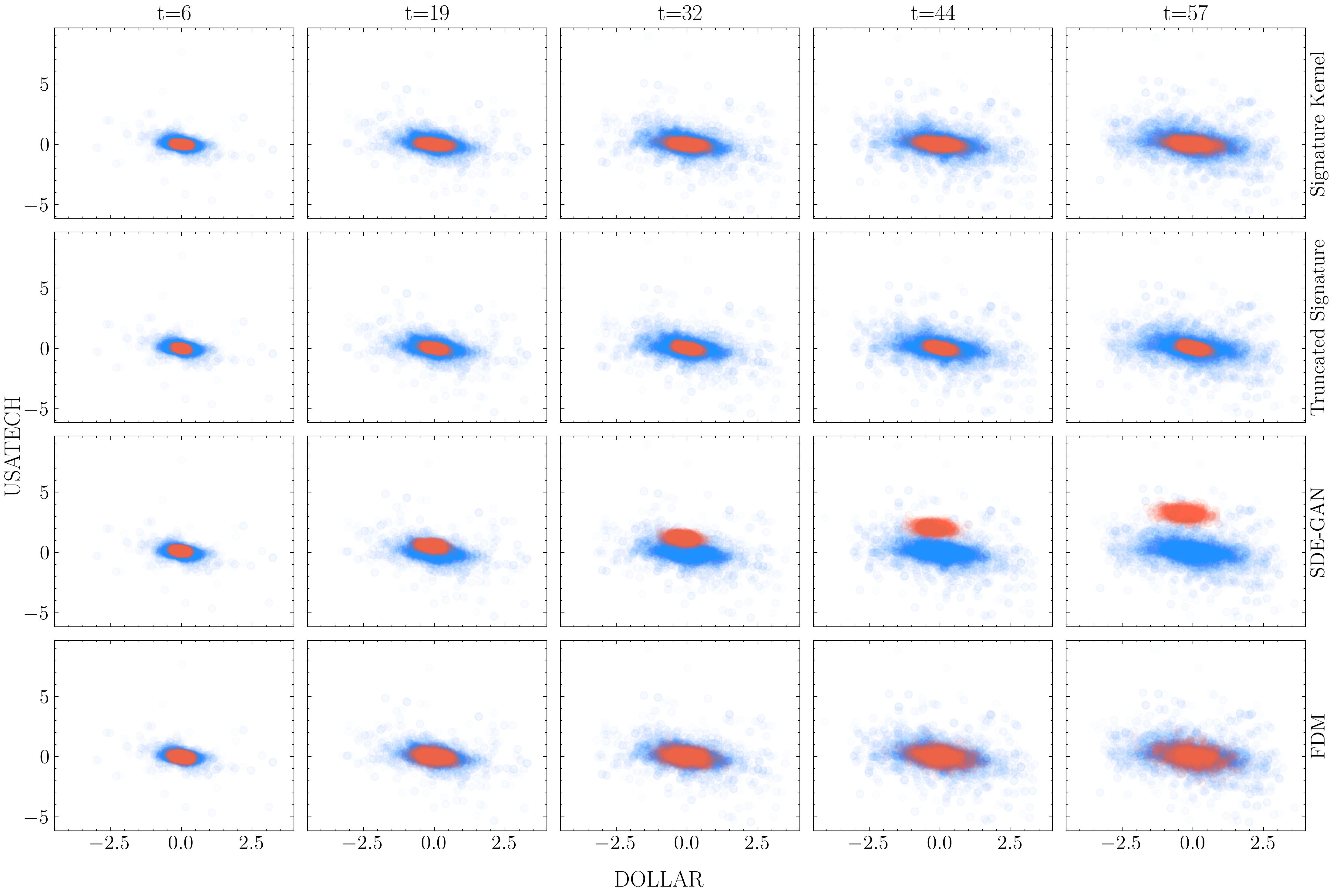} 
\end{center}
\caption{Blue points are real samples and orange points are generated by Neural SDEs. The dynamics of the joint distribution of Dollar and USATECH in the U.S. stock indices data. Each row of plots corresponds to a method and each row corresponds to a timestamp. For each plot, the horizontal axis is Dollar (US Dollar Index) and the vertical axis is USATECH (USA 100 Technical Index).}
\label{fig:indices64_DOLLAR_USATECH}
\end{figure}

\begin{figure}[ht]
\begin{center}
\includegraphics[width=0.9\textwidth]{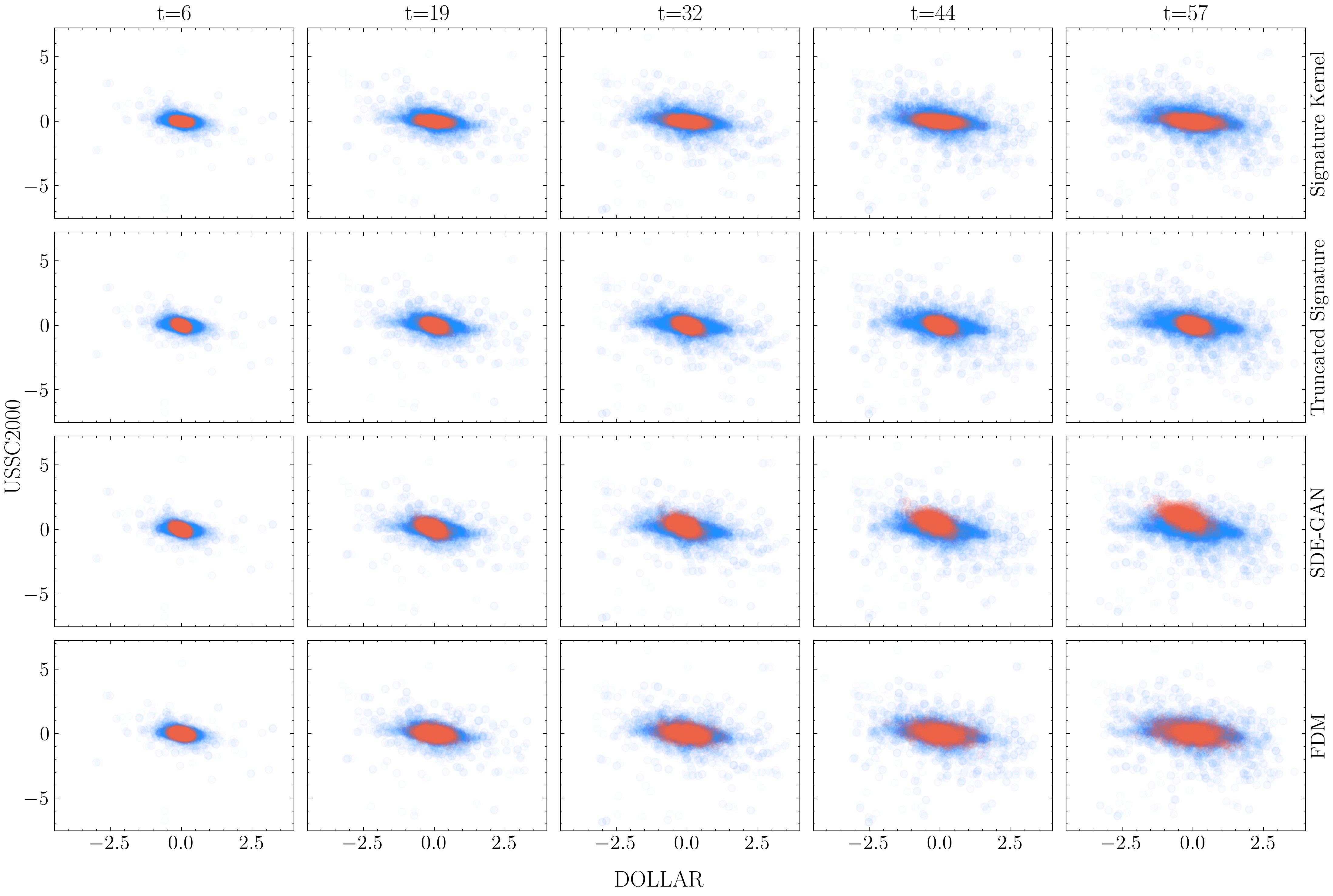} 
\end{center}
\caption{Blue points are real samples and orange points are generated by Neural SDEs. The dynamics of the joint distribution of Dollar and USSC2000 in the U.S. stock indices data. Each row of plots corresponds to a method and each row corresponds to a timestamp. For each plot, the horizontal axis is Dollar (US Dollar Index) and the vertical axis is USSC2000 (US Small Cap 2000).}
\label{fig:indices64_DOLLAR_USSC2000}
\end{figure}

\begin{figure}[ht]
\begin{center}
\includegraphics[width=0.9\textwidth]{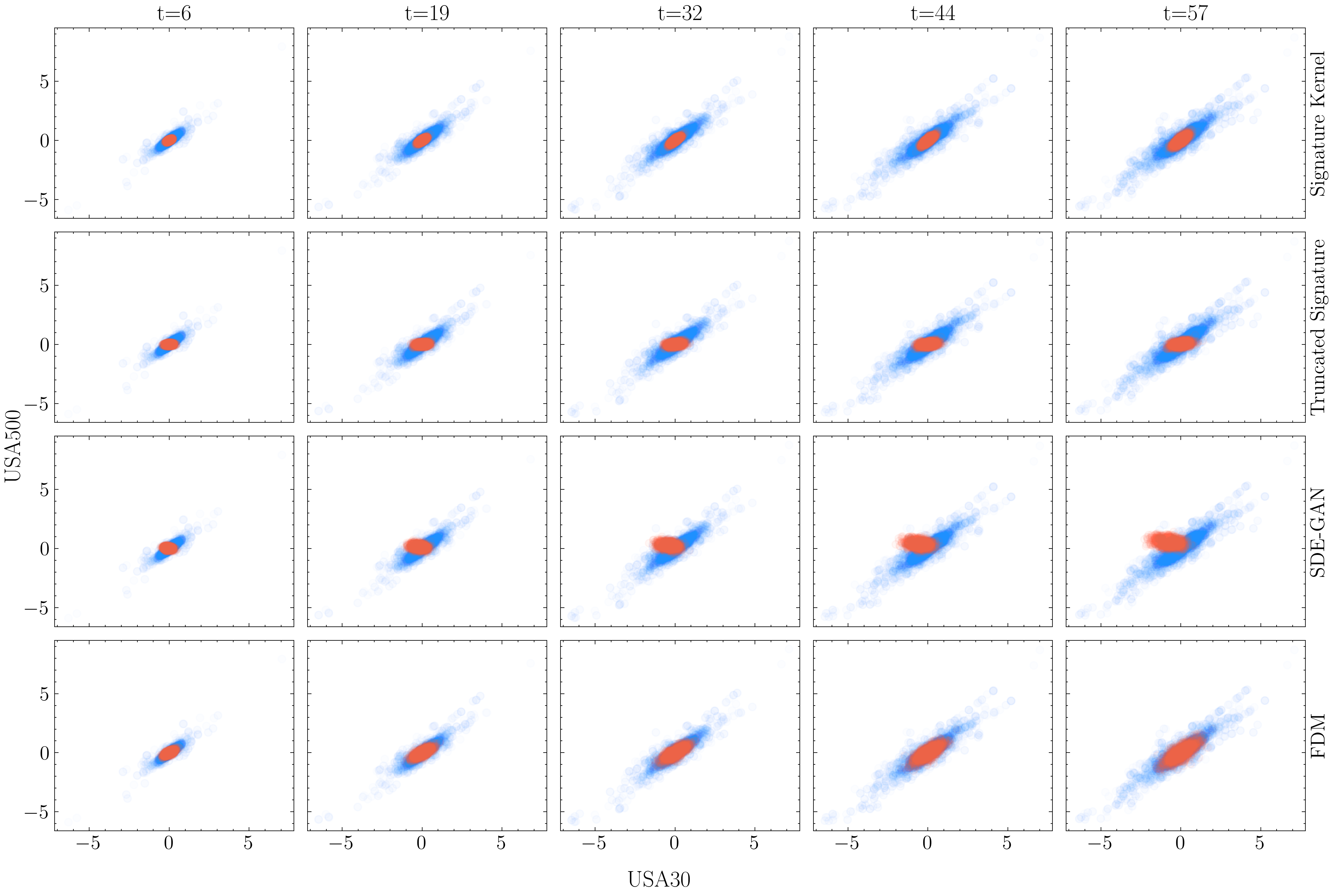} 
\end{center}
\caption{Blue points are real samples and orange points are generated by Neural SDEs. The dynamics of the joint distribution of USA30 and USA500 in the U.S. stock indices data. Each row of plots corresponds to a method and each row corresponds to a timestamp. For each plot, the horizontal axis is USA30 (USA 30 Index) and the vertical axis is USA500 (USA 500 Index).}
\label{fig:indices64_USA30_USA500}
\end{figure}

\begin{figure}[ht]
\begin{center}
\includegraphics[width=0.9\textwidth]{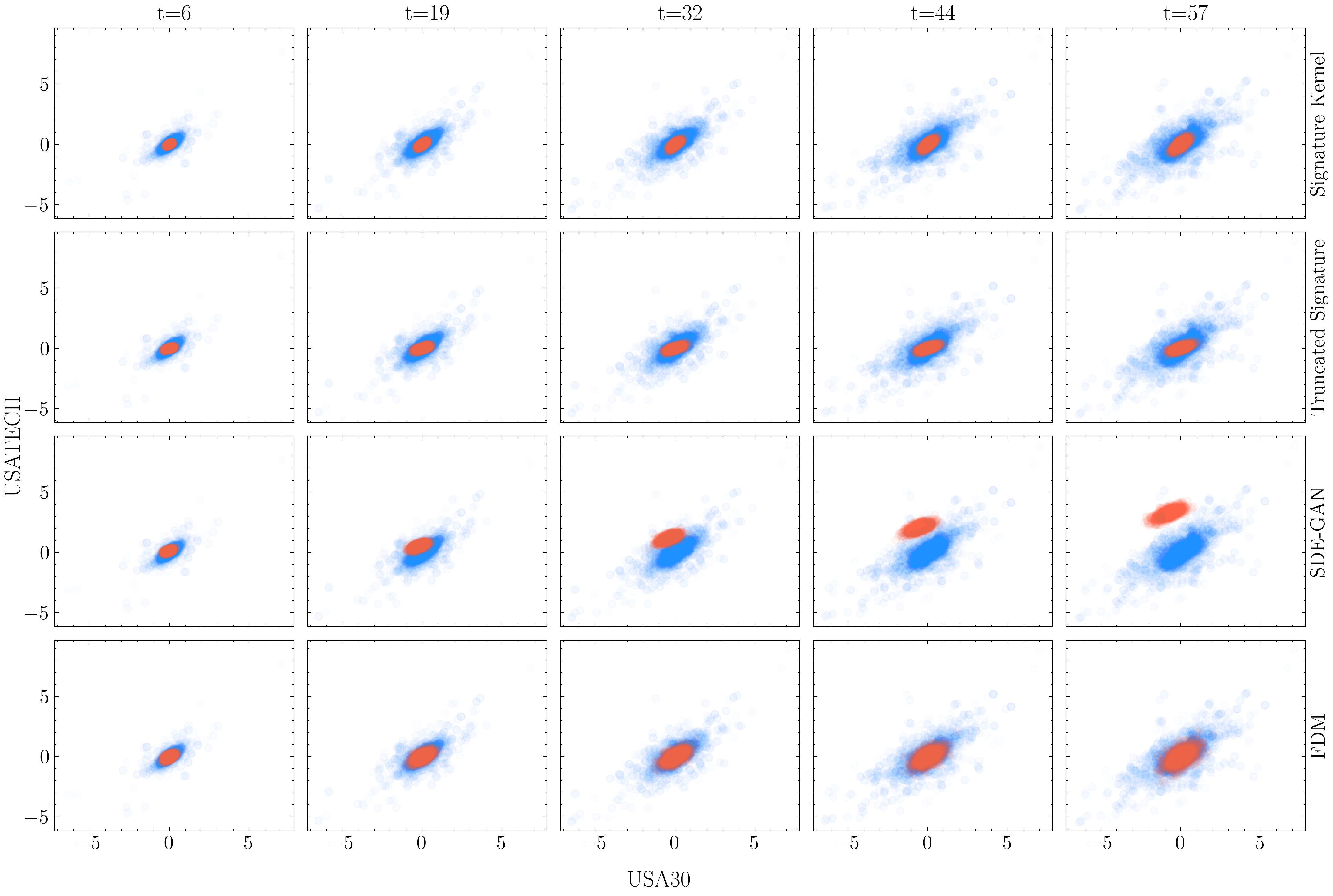} 
\end{center}
\caption{Blue points are real samples and orange points are generated by Neural SDEs. The dynamics of the joint distribution of USA30 and USATECH in the U.S. stock indices data. Each row of plots corresponds to a method and each row corresponds to a timestamp. For each plot, the horizontal axis is USA30 (USA 30 Index) and the vertical axis is USATECH (USA 100 Technical Index).}
\label{fig:indices64_USA30_USATECH}
\end{figure}

\begin{figure}[ht]
\begin{center}
\includegraphics[width=0.9\textwidth]{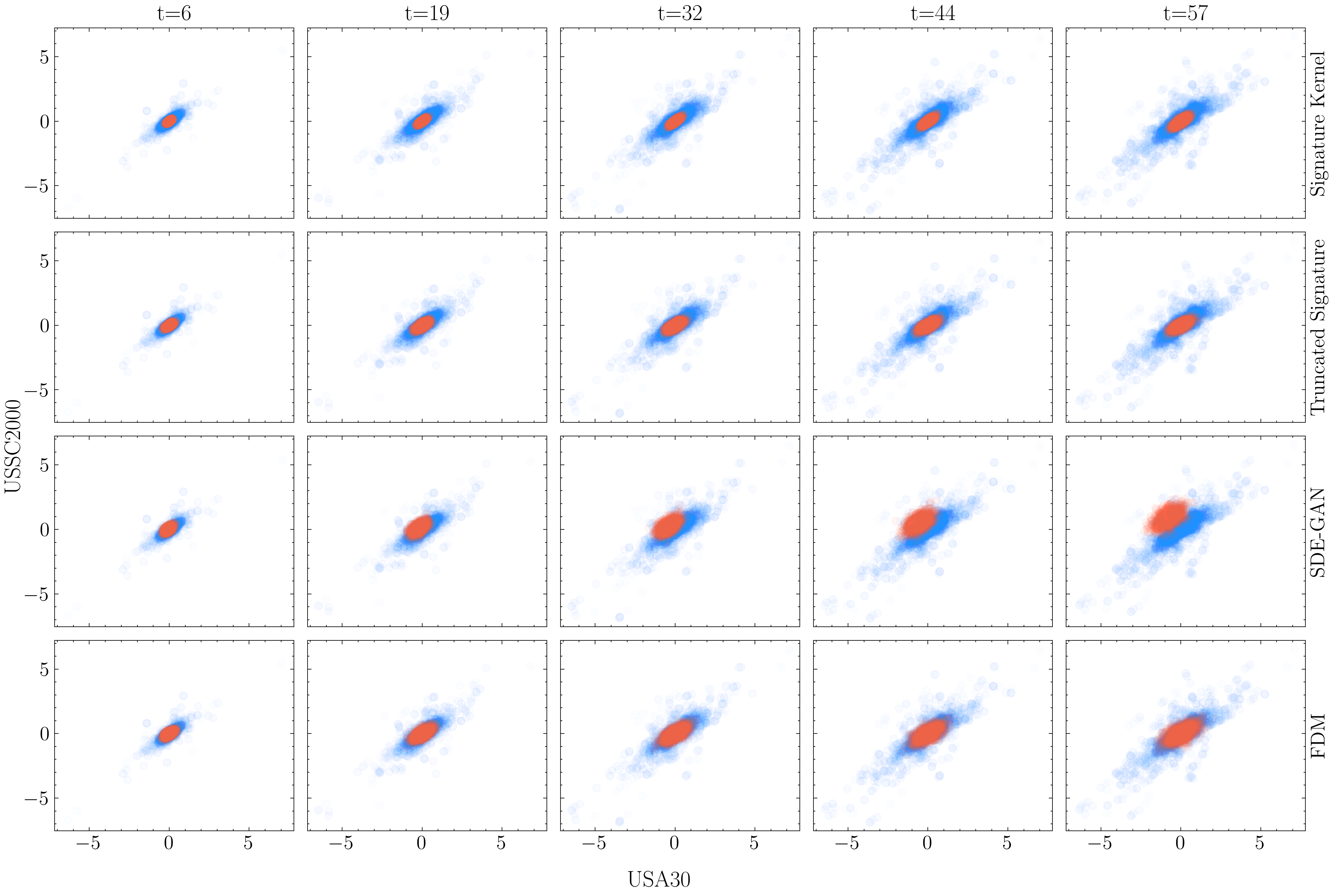} 
\end{center}
\caption{Blue points are real samples and orange points are generated by Neural SDEs. The dynamics of the joint distribution of USA30 and USSC2000 in the U.S. stock indices data. Each row of plots corresponds to a method and each row corresponds to a timestamp. For each plot, the horizontal axis is USA30 (USA 30 Index) and the vertical axis is USSC2000 (US Small Cap 2000).}
\label{fig:indices64_USA30_USSC2000}
\end{figure}

\begin{figure}[ht]
\begin{center}
\includegraphics[width=0.9\textwidth]{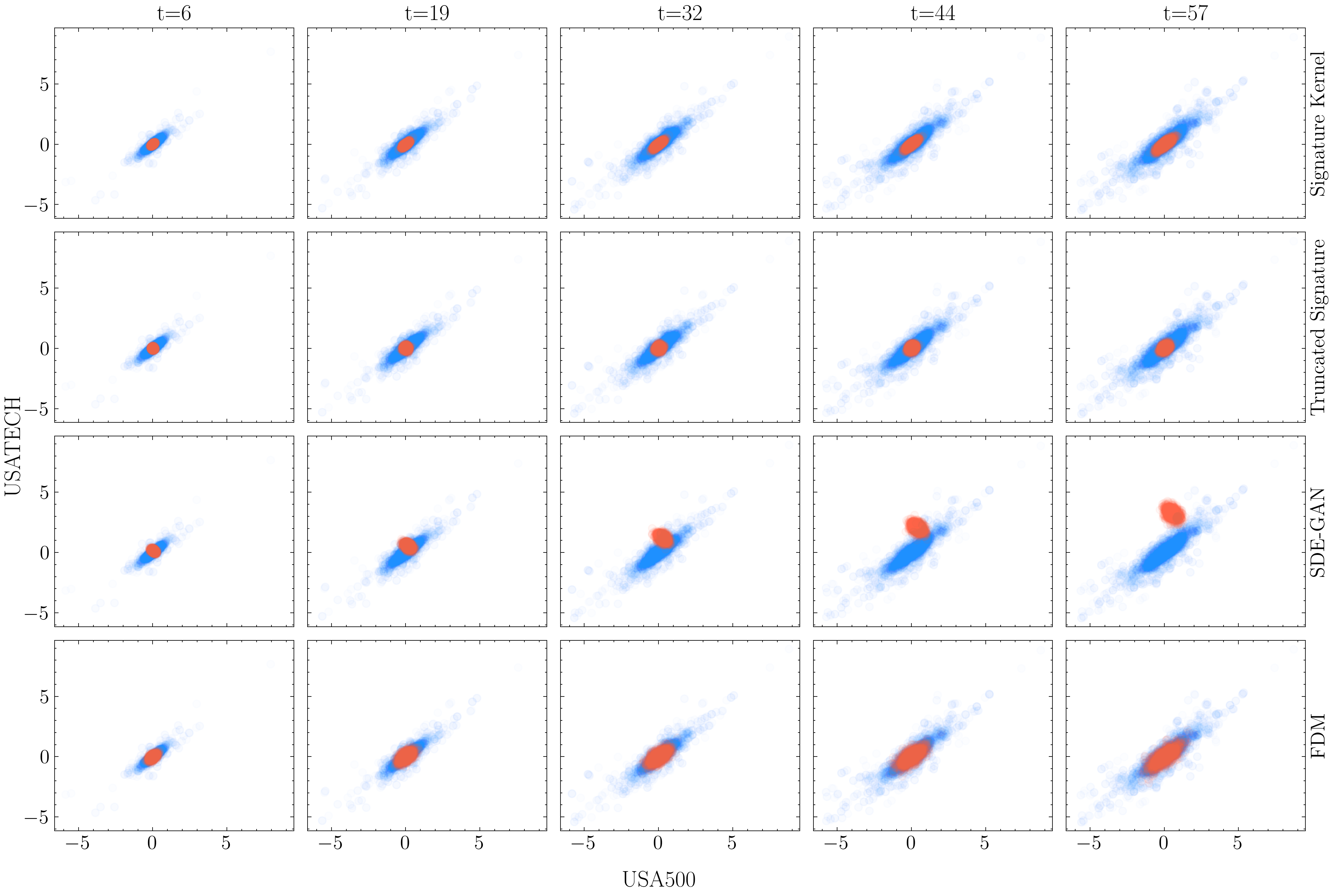} 
\end{center}
\caption{Blue points are real samples and orange points are generated by Neural SDEs. The dynamics of the joint distribution of USA500 and USATECH in the U.S. stock indices data. Each row of plots corresponds to a method and each row corresponds to a timestamp. For each plot, the horizontal axis is USA500 (USA 500 Index) and the vertical axis is USATECH (USA 100 Technical Index).}
\label{fig:indices64_USA500_USATECH}
\end{figure}

\begin{figure}[ht]
\begin{center}
\includegraphics[width=0.9\textwidth]{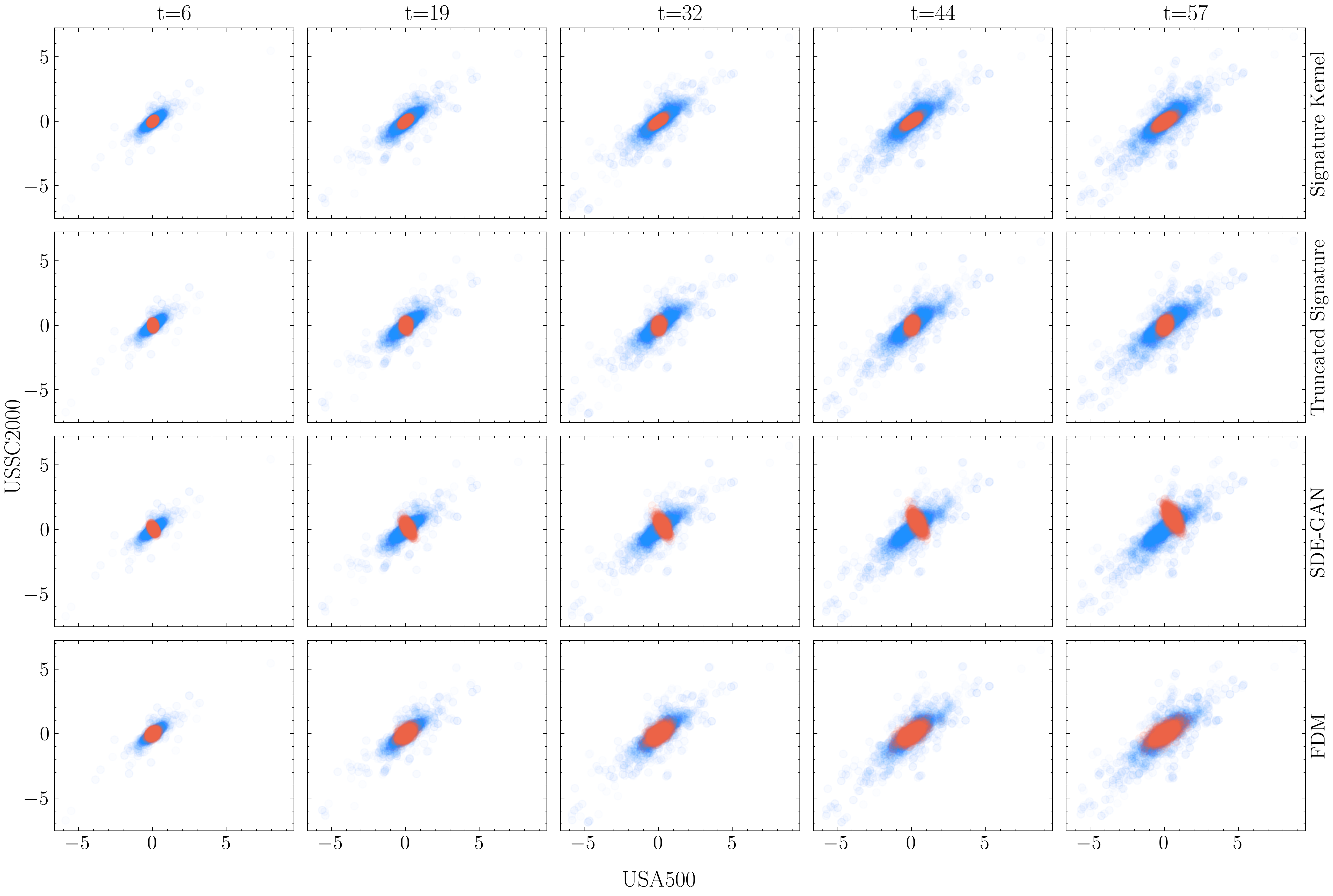} 
\end{center}
\caption{Blue points are real samples and orange points are generated by Neural SDEs. The dynamics of the joint distribution of USA500 and USSC2000 in the U.S. stock indices data. Each row of plots corresponds to a method and each row corresponds to a timestamp. For each plot, the horizontal axis is USA500 (USA 500 Index) and the vertical axis is USSC2000 (US Small Cap 2000).}
\label{fig:indices64_USA500_USSC2000}
\end{figure}

\begin{figure}[ht]
\begin{center}
\includegraphics[width=0.9\textwidth]{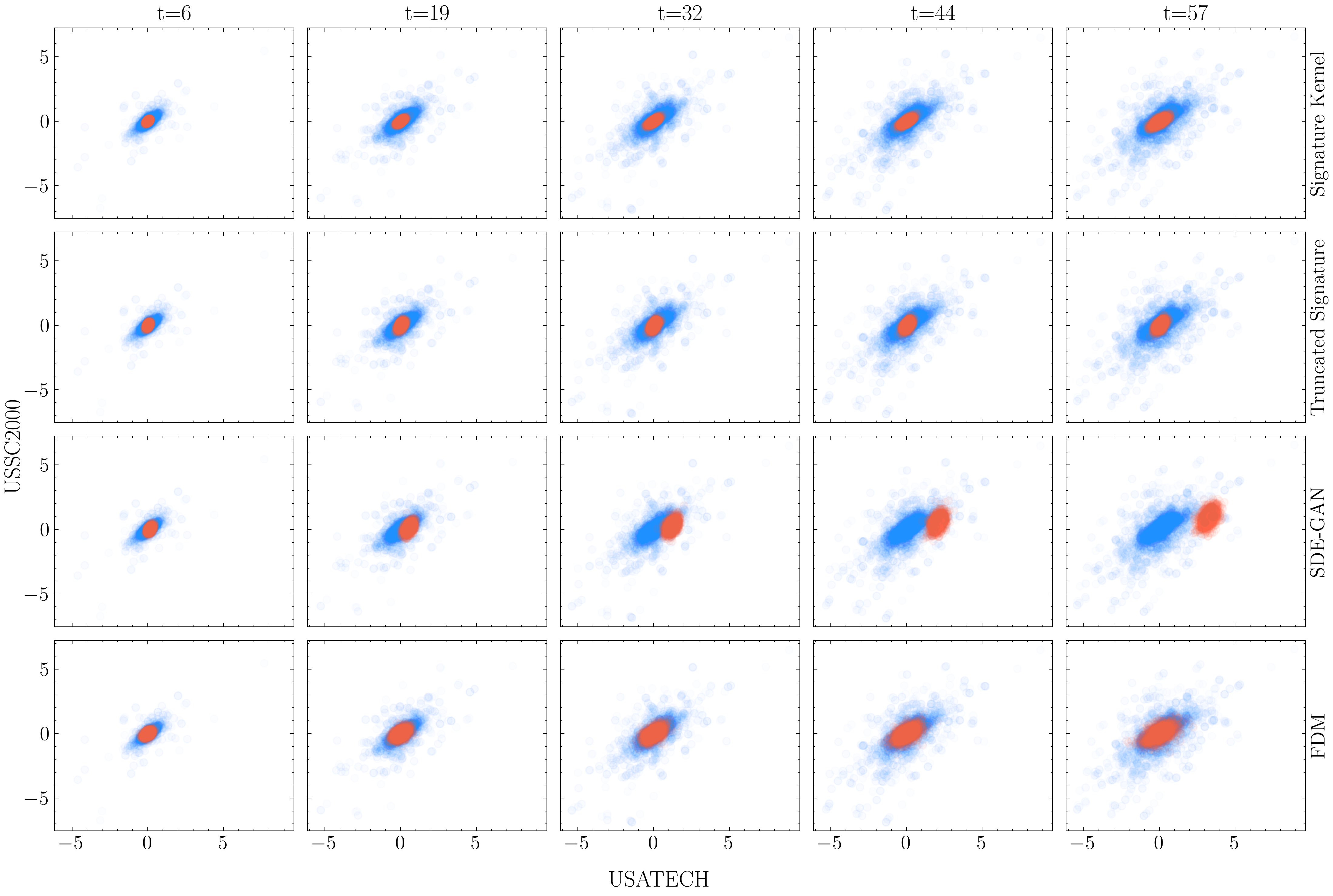} 
\end{center}
\caption{Blue points are real samples and orange points are generated by Neural SDEs. The dynamics of the joint distribution of USATECH and USSC2000 in the U.S. stock indices data. Each row of plots corresponds to a method and each row corresponds to a timestamp. For each plot, the horizontal axis is USATECH (USA 100 Technical Index) and the vertical axis is USSC2000 (US Small Cap 2000).}
\label{fig:indices64_USATECH_USSC2000}
\end{figure}

\begin{figure}[ht]
\begin{center}
\includegraphics[width=0.9\textwidth]{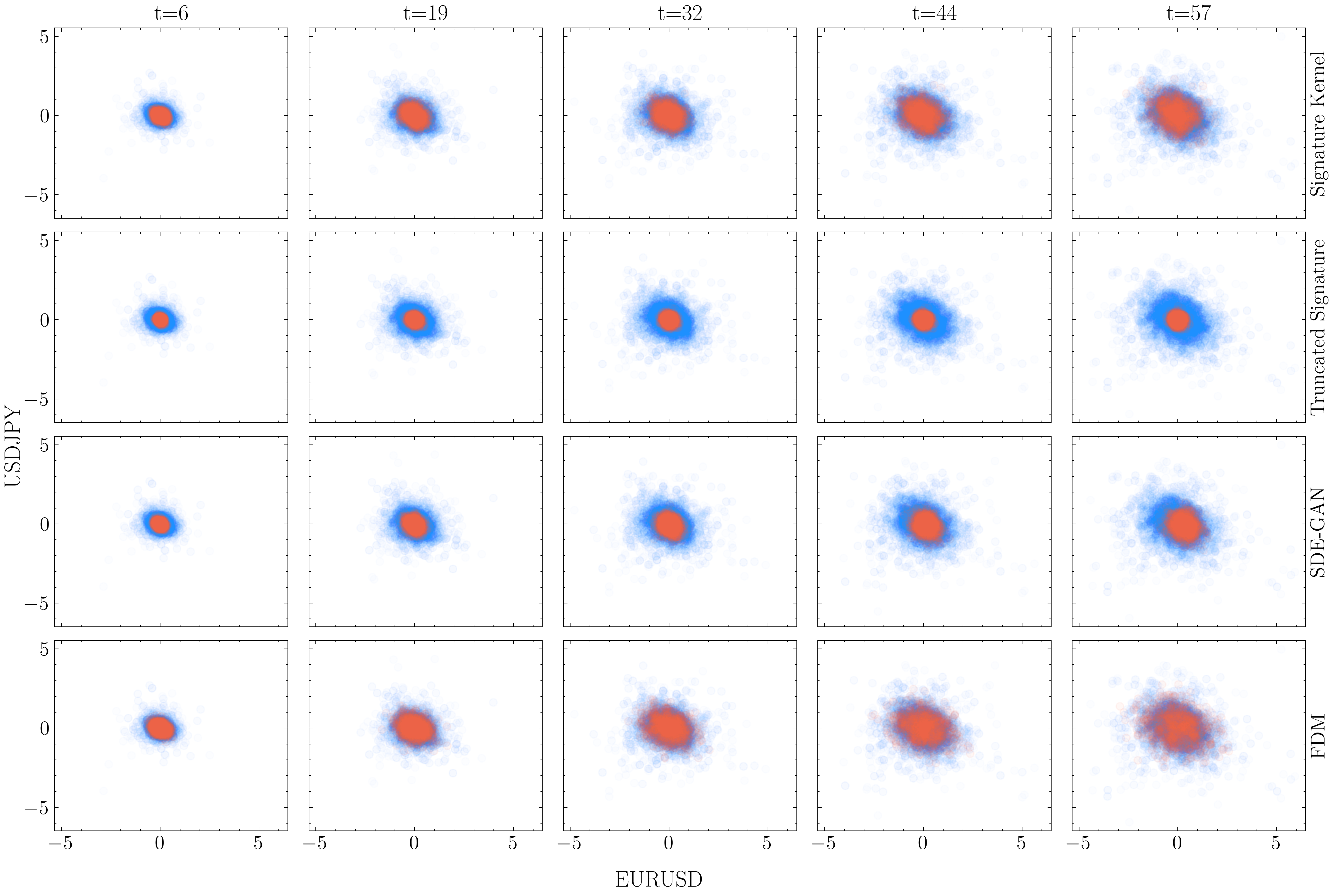} 
\end{center}
\caption{Blue points are real samples and orange points are generated by Neural SDEs. The dynamics of the joint distribution of EUR/USD and USD/JPY in exchange rate data. Each row of plots corresponds to a method and each row corresponds to a timestamp. For each plot, the horizontal axis is EUR/USD and the vertical axis is USD/JPY.}
\label{fig:forex64_EURUSD_USDJPY}
\end{figure}

\begin{figure}[ht]
\begin{center}
\includegraphics[width=0.9\textwidth]{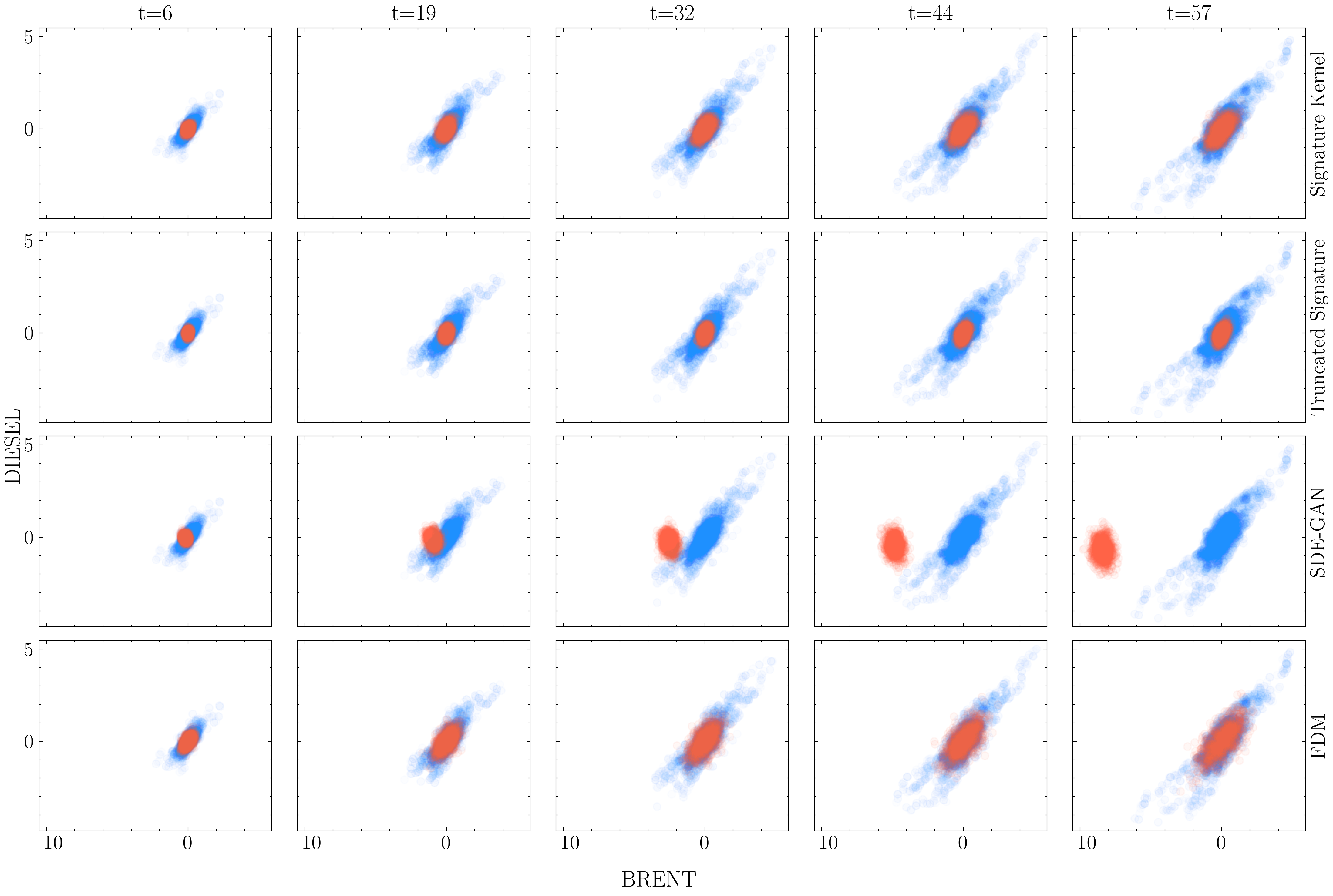} 
\end{center}
\caption{Blue points are real samples and orange points are generated by Neural SDEs. The dynamics of the joint distribution of BRENT and DIESEL in energy data. Each row of plots corresponds to a method and each row corresponds to a timestamp. For each plot, the horizontal axis is BRENT (U.S. Brent Crude Oil) and the vertical axis is DIESEL (Gas Oil).}
\label{fig:energy64_BRENT_DIESEL}
\end{figure}

\begin{figure}[ht]
\begin{center}
\includegraphics[width=0.9\textwidth]{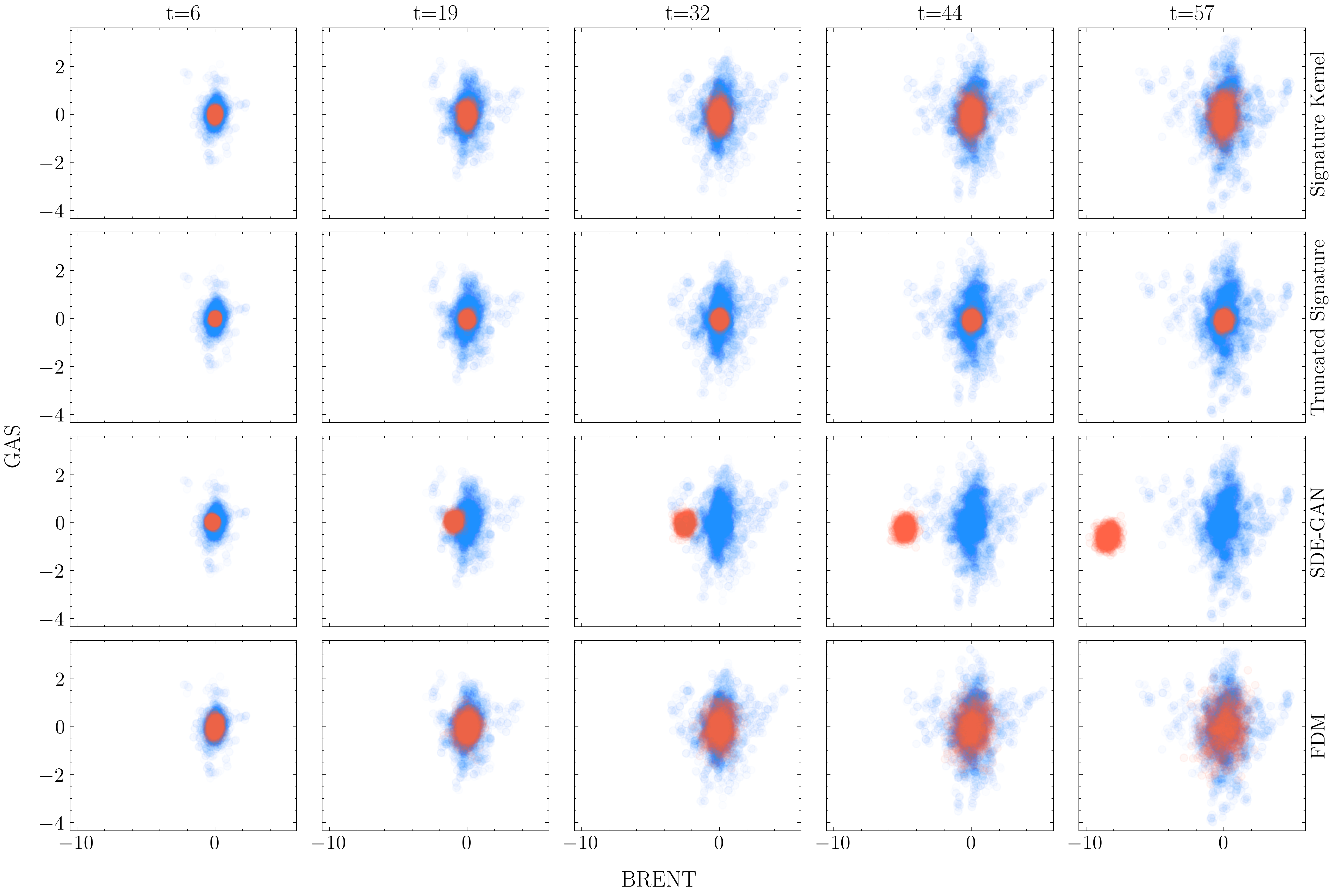} 
\end{center}
\caption{Blue points are real samples and orange points are generated by Neural SDEs. The dynamics of the joint distribution of BRENT and GAS in energy data. Each row of plots corresponds to a method and each row corresponds to a timestamp. For each plot, the horizontal axis is BRENT (U.S. Brent Crude Oil) and the vertical axis is GAS (Natural Gas).}
\label{fig:energy64_BRENT_GAS}
\end{figure}

\begin{figure}[ht]
\begin{center}
\includegraphics[width=0.9\textwidth]{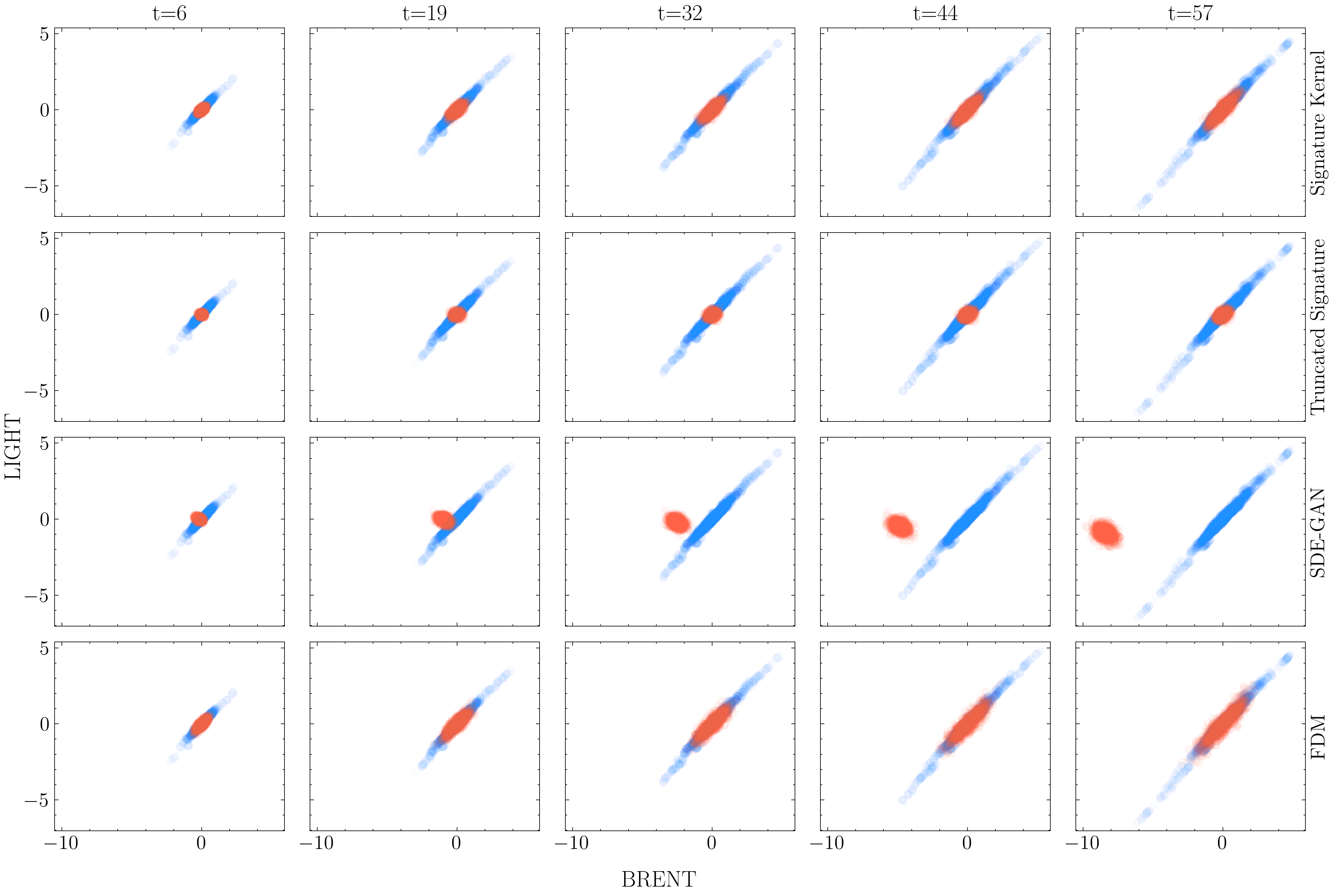} 
\end{center}
\caption{Blue points are real samples and orange points are generated by Neural SDEs. The dynamics of the joint distribution of BRENT and LIGHT in energy data. Each row of plots corresponds to a method and each row corresponds to a timestamp. For each plot, the horizontal axis is BRENT (U.S. Brent Crude Oil) and the vertical axis is LIGHT (U.S. Light Crude Oil).}
\label{fig:energy64_BRENT_LIGHT}
\end{figure}

\begin{figure}[ht]
\begin{center}
\includegraphics[width=0.9\textwidth]{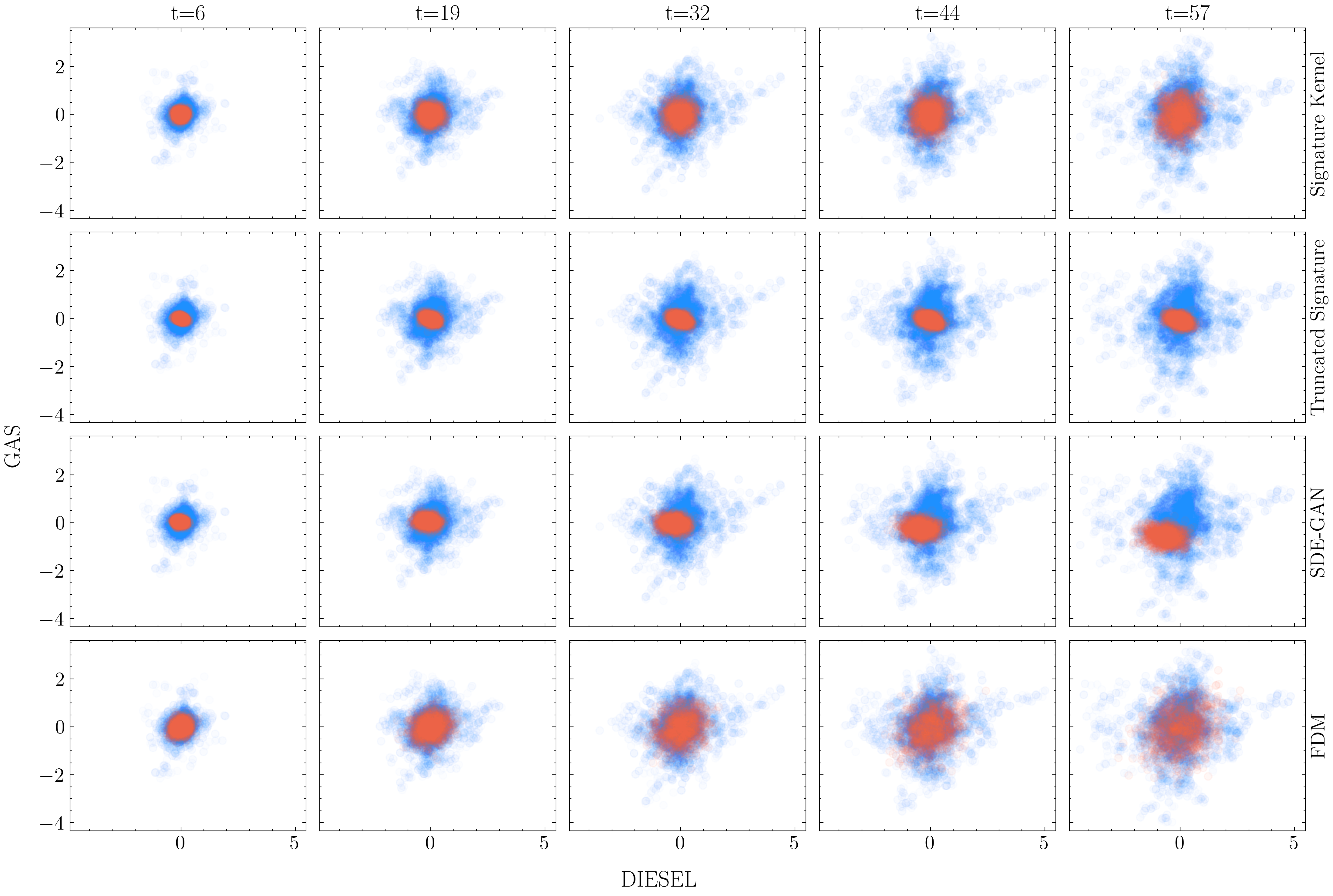} 
\end{center}
\caption{Blue points are real samples and orange points are generated by Neural SDEs. The dynamics of the joint distribution of DIESEL and GAS in energy data. Each row of plots corresponds to a method and each row corresponds to a timestamp. For each plot, the horizontal axis is DIESEL (Gas Oil) and the vertical axis is GAS (Natural Gas).}
\label{fig:energy64_DIESEL_GAS}
\end{figure}

\begin{figure}[ht]
\begin{center}
\includegraphics[width=0.9\textwidth]{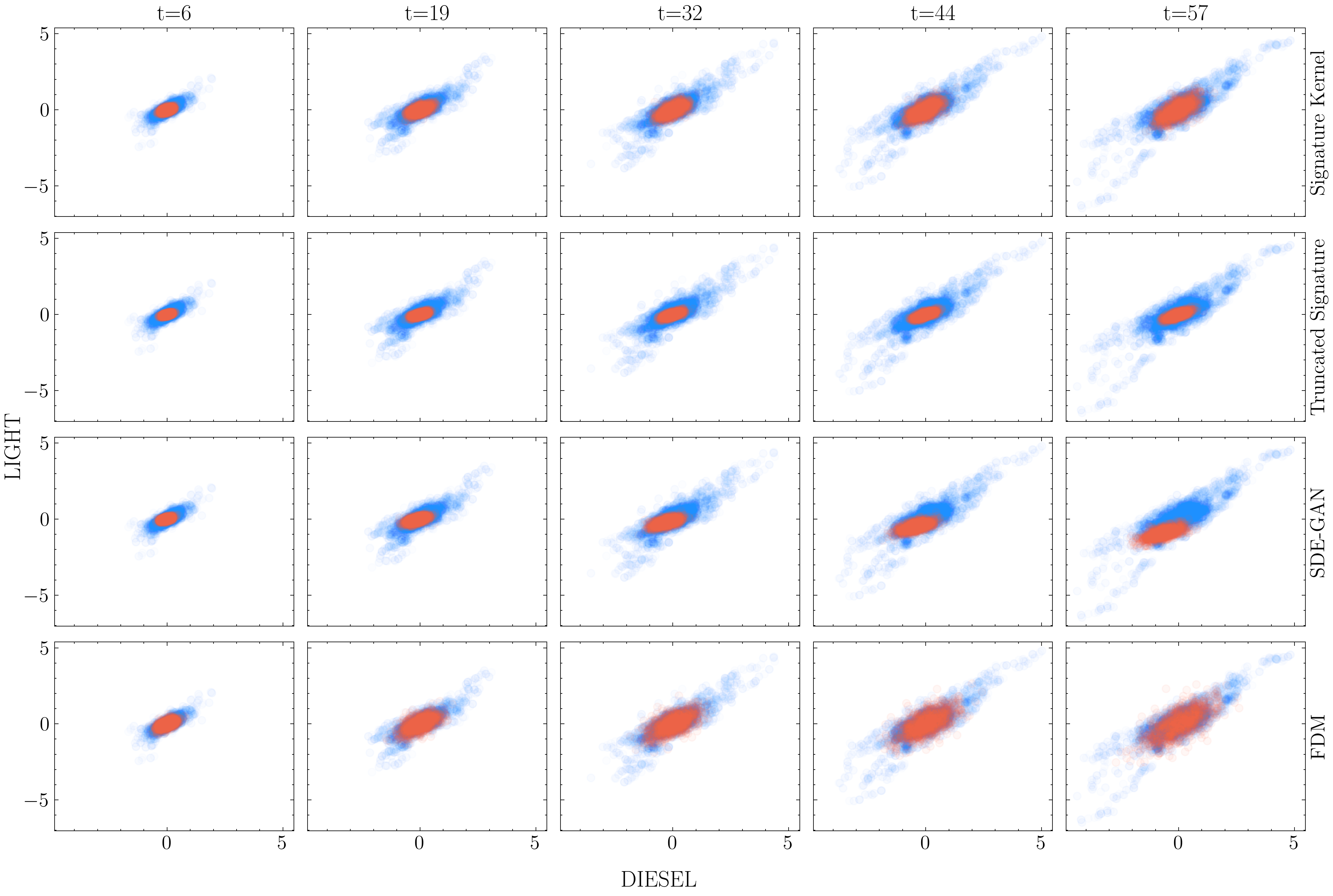} 
\end{center}
\caption{Blue points are real samples and orange points are generated by Neural SDEs. The dynamics of the joint distribution of DIESEL and LIGHT in energy data. Each row of plots corresponds to a method and each row corresponds to a timestamp. For each plot, the horizontal axis is DIESEL (Gas Oil) and the vertical axis is LIGHT (U.S. Light Crude Oil).}
\label{fig:energy64_DIESEL_LIGHT}
\end{figure}

\begin{figure}[ht]
\begin{center}
\includegraphics[width=0.9\textwidth]{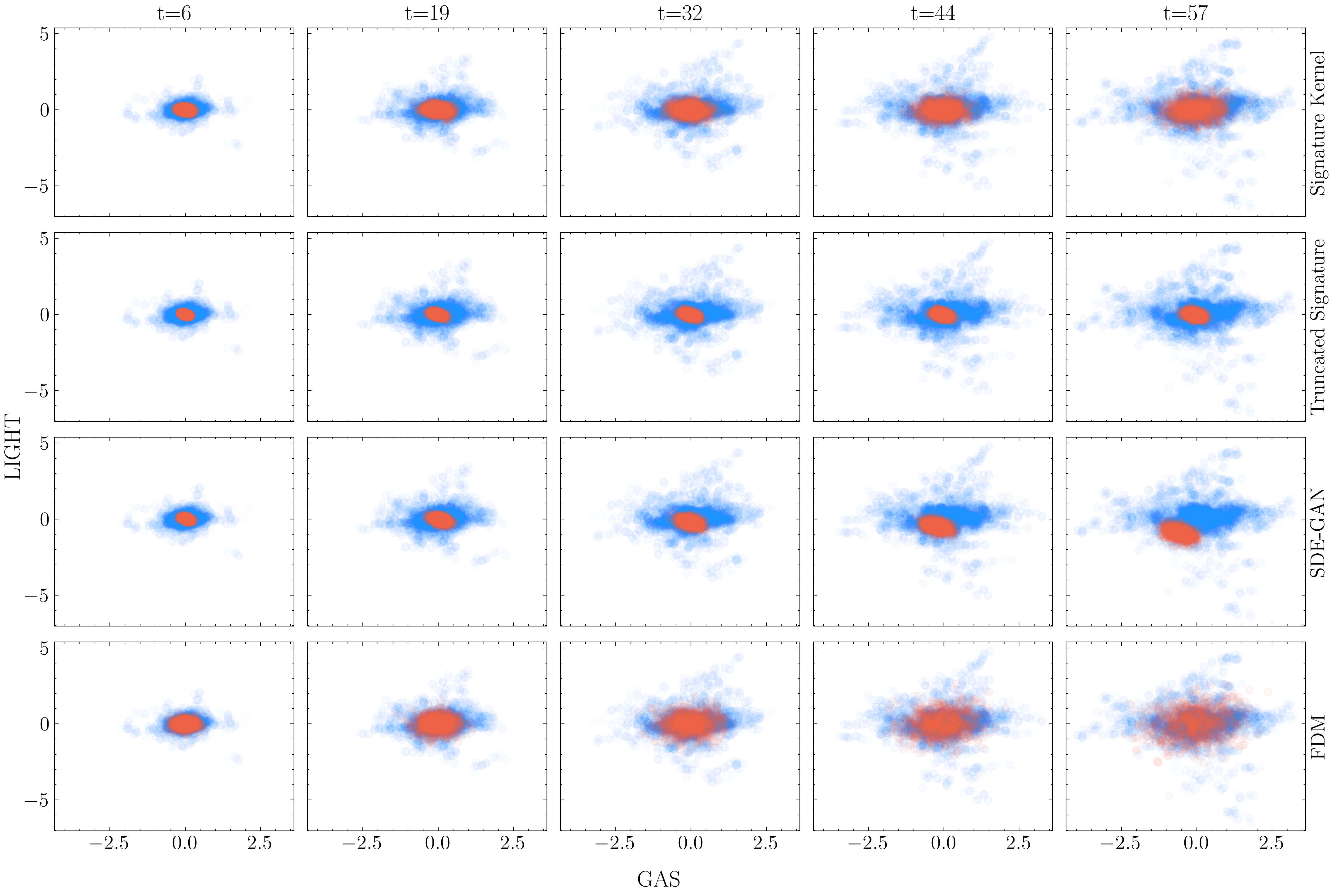} 
\end{center}
\caption{Blue points are real samples and orange points are generated by Neural SDEs. The dynamics of the joint distribution of GAS and LIGHT in energy data. Each row of plots corresponds to a method and each row corresponds to a timestamp. For each plot, the horizontal axis is GAS (Natural Gas) and the vertical axis is LIGHT (U.S. Light Crude Oil).}
\label{fig:energy64_GAS_LIGHT}
\end{figure}

\begin{figure}[ht]
\begin{center}
\includegraphics[width=0.9\textwidth]{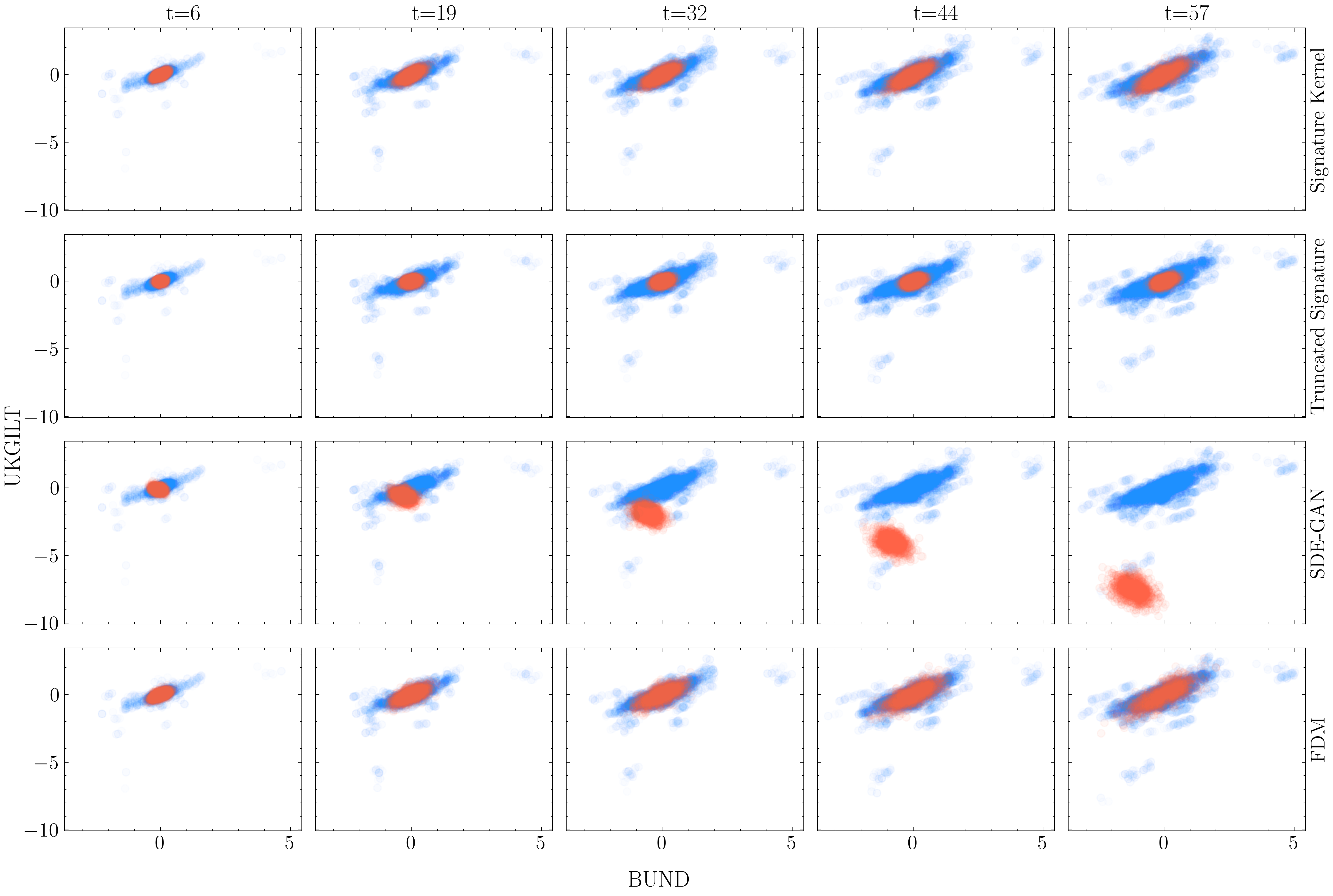} 
\end{center}
\caption{Blue points are real samples and orange points are generated by Neural SDEs. The dynamics of the joint distribution of BUND and UKGILT in bunds data. Each row of plots corresponds to a method and each row corresponds to a timestamp. For each plot, the horizontal axis is BUND (Euro Bund) and the vertical axis is UKGILT (UK Long Gilt).}
\label{fig:bonds64_BUND_UKGILT}
\end{figure}

\begin{figure}[ht]
\begin{center}
\includegraphics[width=0.9\textwidth]{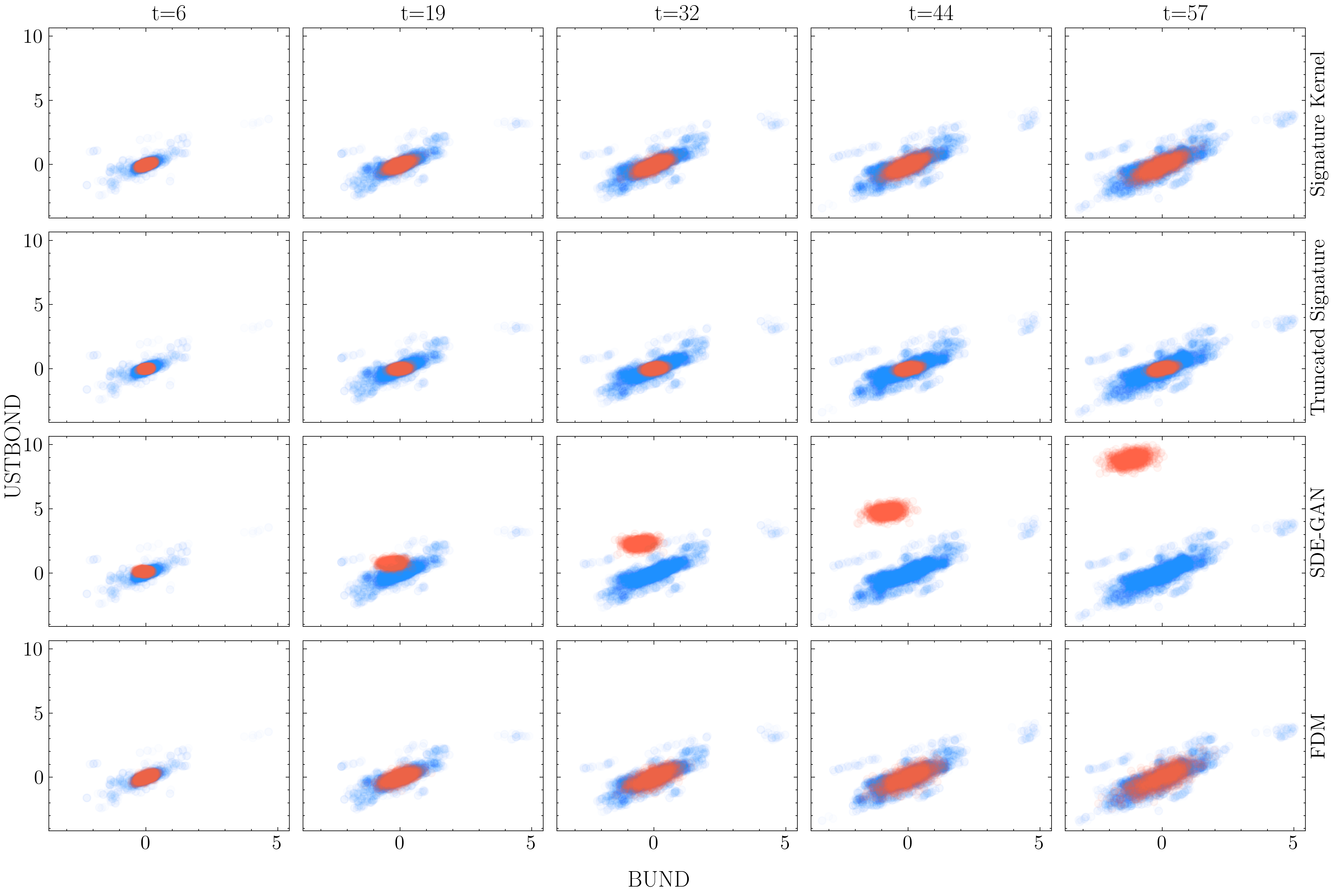} 
\end{center}
\caption{Blue points are real samples and orange points are generated by Neural SDEs. The dynamics of the joint distribution of BUND and USTBOND in bunds data. Each row of plots corresponds to a method and each row corresponds to a timestamp. For each plot, the horizontal axis is BUND (Euro Bund) and the vertical axis is USTBOND (US T-BOND).}
\label{fig:bonds64_BUND_USTBOND}
\end{figure}

\begin{figure}[ht]
\begin{center}
\includegraphics[width=0.9\textwidth]{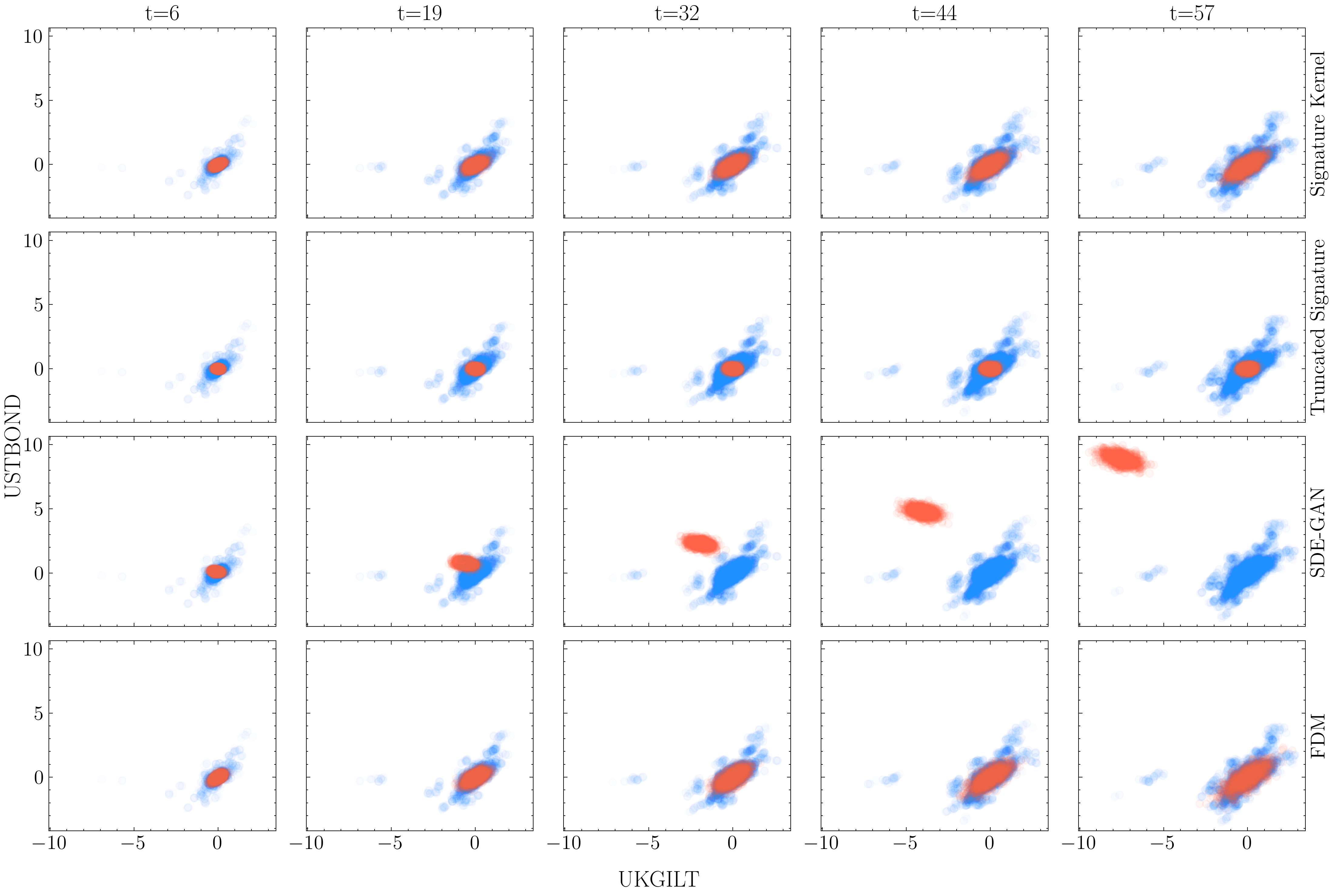} 
\end{center}
\caption{Blue points are real samples and orange points are generated by Neural SDEs. The dynamics of the joint distribution of UKGILT and USTBOND in bunds data. Each row of plots corresponds to a method and each row corresponds to a timestamp. For each plot, the horizontal axis is UKGILT (UK Long Gilt) and the vertical axis is USTBOND (US T-BOND).}
\label{fig:bonds64_UKGILT_USTBOND}
\end{figure}

We present additional qualitative results comparing real and generated sample paths. Results for the exchange rates data are presented in Figures \ref{fig:paths_forex64_dim1} and \ref{fig:paths_forex64_dim2}. Similarly, Figures \ref{fig:paths_indices_dim1} through \ref{fig:paths_indices_dim5} show the sample paths for five features from the stock indices dataset: "DOLLAR," "USA30," "USA500," "USATECH," and "USSC2000". Finally, Figures \ref{fig:paths_metal_dim1} and \ref{fig:paths_metal_dim2} depict the sample paths for silver and gold prices from the metal dataset. These plots demonstrate the ability of Neural SDEs to capture dynamics across diverse datasets. 

\begin{figure}[ht]
\begin{center}
\includegraphics[width=0.9\textwidth]{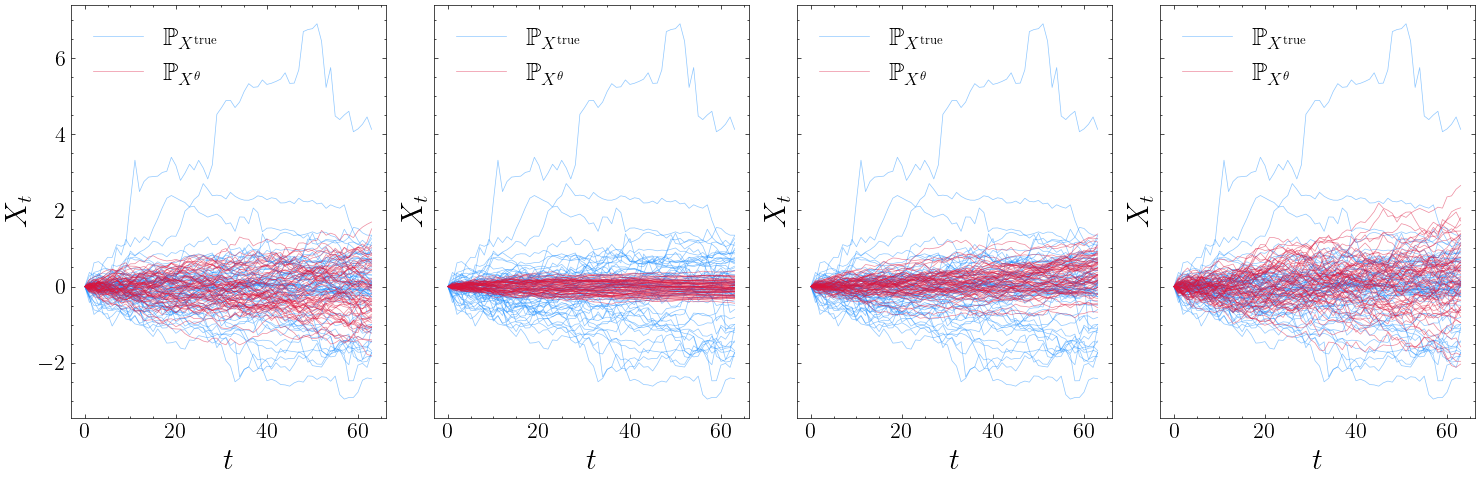} 
\caption{Sample paths for EUR/USD exchange rates from the exchange rate dataset. Blue lines represent real samples, while red lines represent those generated by Neural SDEs. From left to right, the plots correspond to signature kernels, truncated signature, SDE-GAN, and FDM, respectively. The horizontal axis represents time, and the vertical axis represents the EUR/USD exchange rate.}
\label{fig:paths_forex64_dim1}
\end{center}
\end{figure}

\begin{figure}[ht]
\begin{center}
\includegraphics[width=0.9\textwidth]{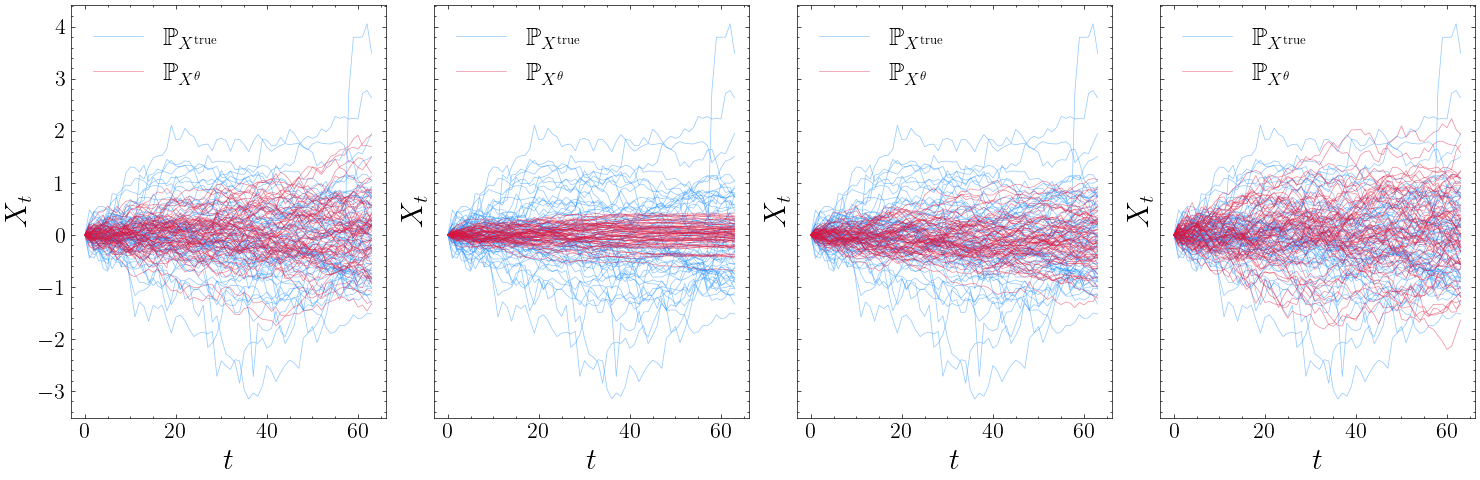}
\caption{Sample paths for USD/JPY exchange rates from the exchange rate dataset. Blue lines represent real samples, while red lines represent those generated by Neural SDEs. From left to right, the plots correspond to signature kernels, truncated signature, SDE-GAN, and FDM, respectively. The horizontal axis represents time, and the vertical axis represents the USD/JPY exchange rate.}
\label{fig:paths_forex64_dim2}
\end{center}
\end{figure}

\begin{figure}[ht]
\begin{center}
\includegraphics[width=0.9\textwidth]{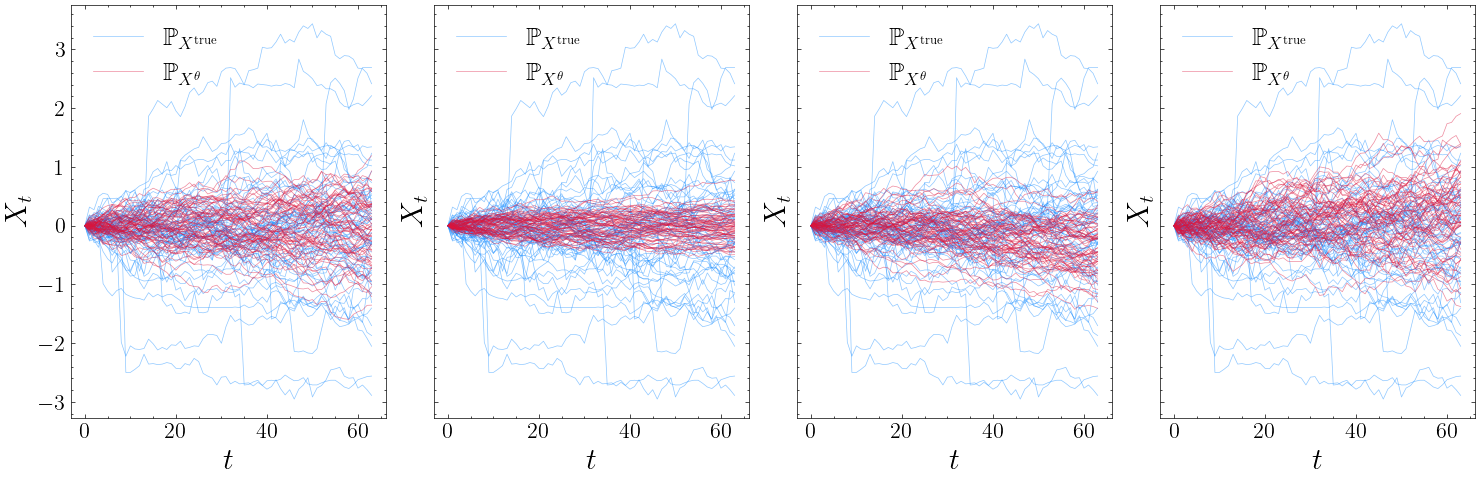}
\caption{Sample paths for "DOLLAR" index from the stock indices dataset. Blue lines represent real samples, while red lines represent those generated by Neural SDEs. From left to right, the plots correspond to signature kernels, truncated signature, SDE-GAN, and FDM, respectively. The horizontal axis represents time, and the vertical axis represents the "DOLLAR" index value.}
\label{fig:paths_indices_dim1}
\end{center}
\end{figure}

\begin{figure}[ht]
\begin{center}
\includegraphics[width=0.9\textwidth]{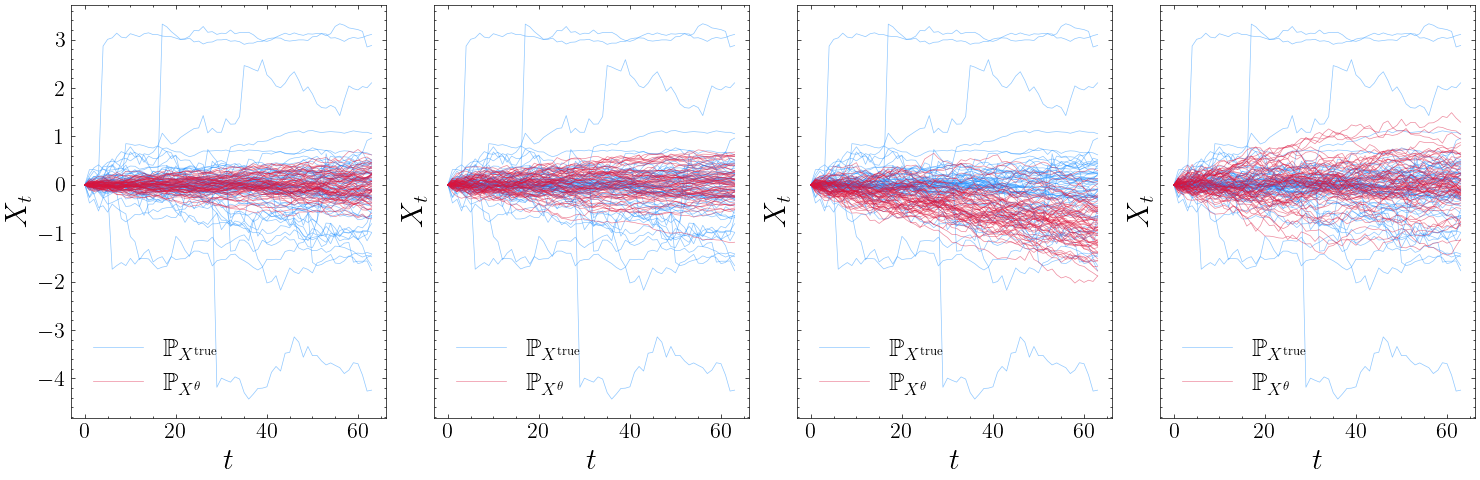}
\caption{Sample paths for "USA30" index from the stock indices dataset. Blue lines represent real samples, while red lines represent those generated by Neural SDEs. From left to right, the plots correspond to signature kernels, truncated signature, SDE-GAN, and FDM, respectively. The horizontal axis represents time, and the vertical axis represents the "USA30" index value.}
\label{fig:paths_indices_dim2}
\end{center}
\end{figure}

\begin{figure}[ht]
\begin{center}
\includegraphics[width=0.9\textwidth]{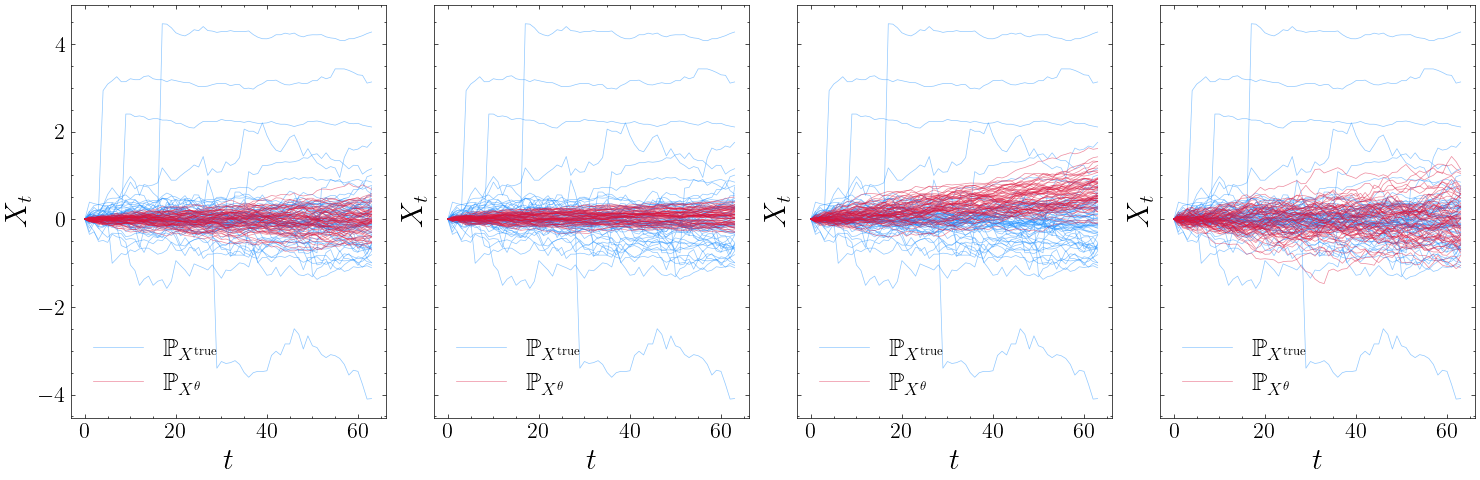}
\caption{Sample paths for "USA500" index from the stock indices dataset. Blue lines represent real samples, while red lines represent those generated by Neural SDEs. From left to right, the plots correspond to signature kernels, truncated signature, SDE-GAN, and FDM, respectively. The horizontal axis represents time, and the vertical axis represents the "USA500" index value.}
\label{fig:paths_indices_dim3}
\end{center}
\end{figure}

\begin{figure}[ht]
\begin{center}
\includegraphics[width=0.9\textwidth]{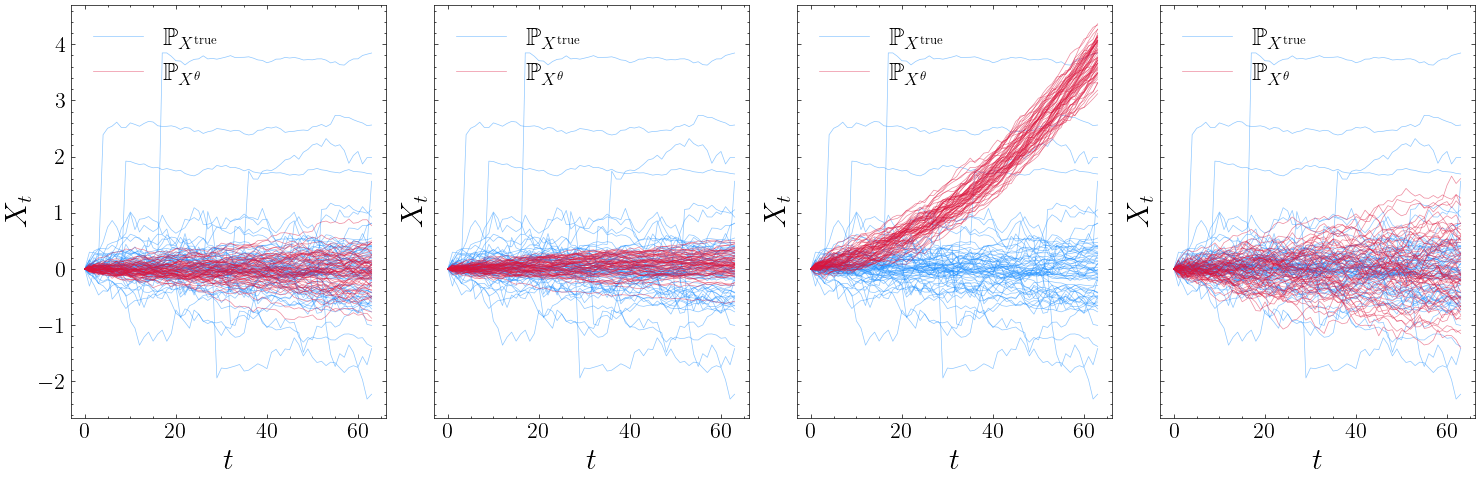}
\caption{Sample paths for "USATECH" index from the stock indices dataset. Blue lines represent real samples, while red lines represent those generated by Neural SDEs. From left to right, the plots correspond to signature kernels, truncated signature, SDE-GAN, and FDM, respectively. The horizontal axis represents time, and the vertical axis represents the "USATECH" index value.}
\label{fig:paths_indices_dim4}
\end{center}
\end{figure}

\begin{figure}[ht]
\begin{center}
\includegraphics[width=0.9\textwidth]{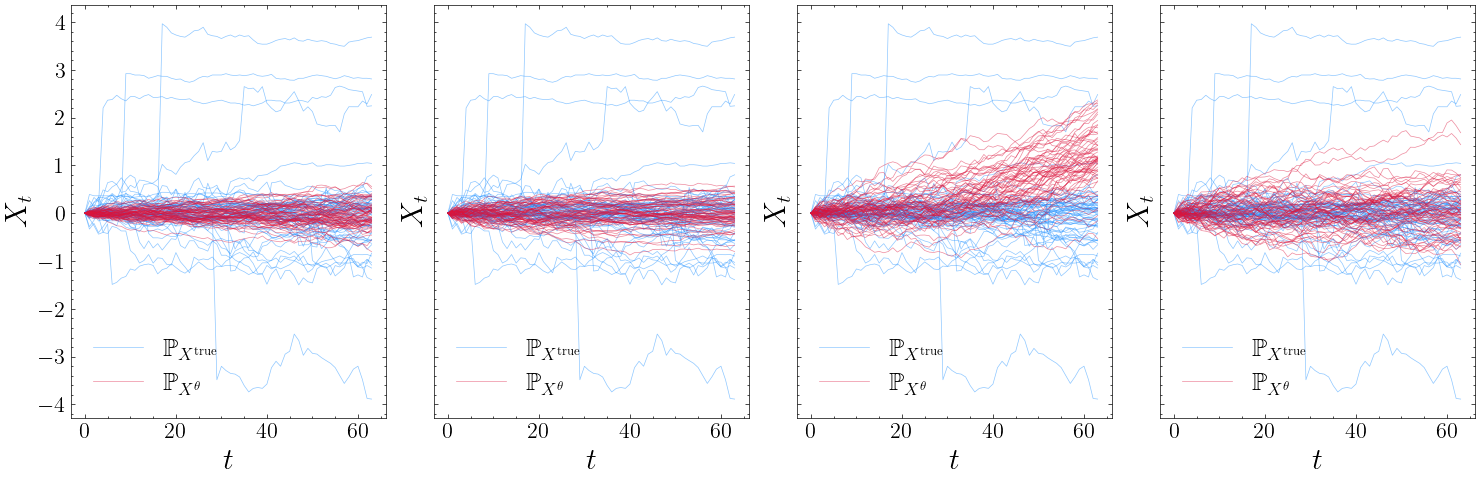}
\caption{Sample paths for "USSC2000" index from the stock indices dataset. Blue lines represent real samples, while red lines represent those generated by Neural SDEs. From left to right, the plots correspond to signature kernels, truncated signature, SDE-GAN, and FDM, respectively. The horizontal axis represents time, and the vertical axis represents the "USSC2000" index value.}
\label{fig:paths_indices_dim5}
\end{center}
\end{figure}

\begin{figure}[ht]
\begin{center}
\includegraphics[width=0.9\textwidth]{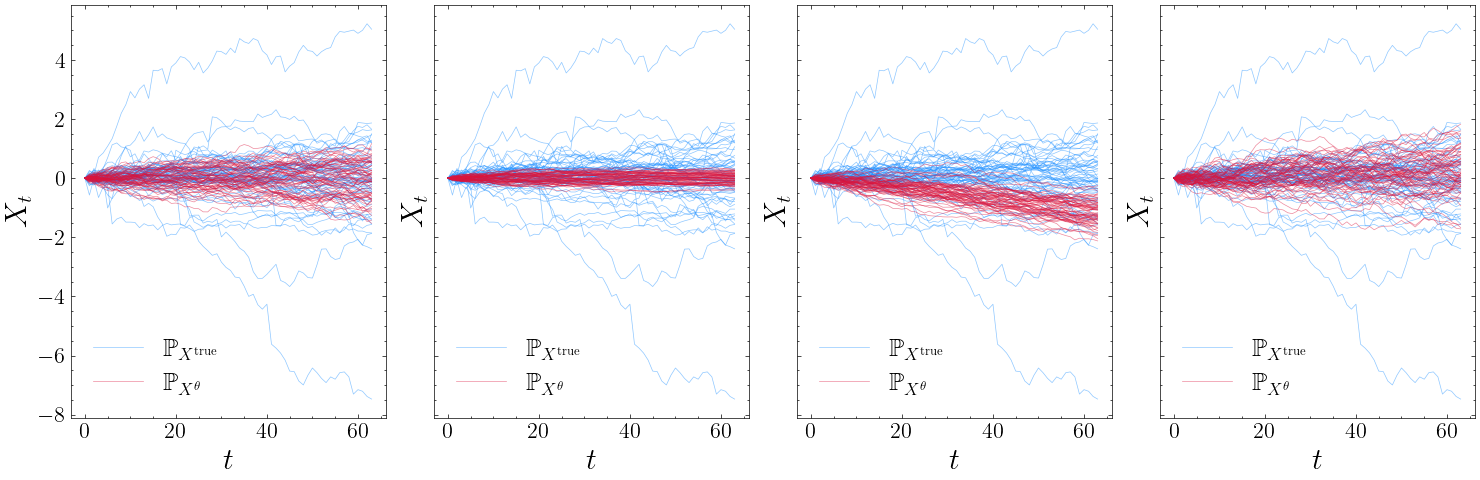}
\caption{Sample paths for silver prices from the metal dataset. Blue lines represent real samples, while red lines represent those generated by Neural SDEs. From left to right, the plots correspond to signature kernels, truncated signature, SDE-GAN, and FDM, respectively. The horizontal axis represents time, and the vertical axis represents silver prices.}
\label{fig:paths_metal_dim1}
\end{center}
\end{figure}

\begin{figure}[ht]
\begin{center}
\includegraphics[width=0.9\textwidth]{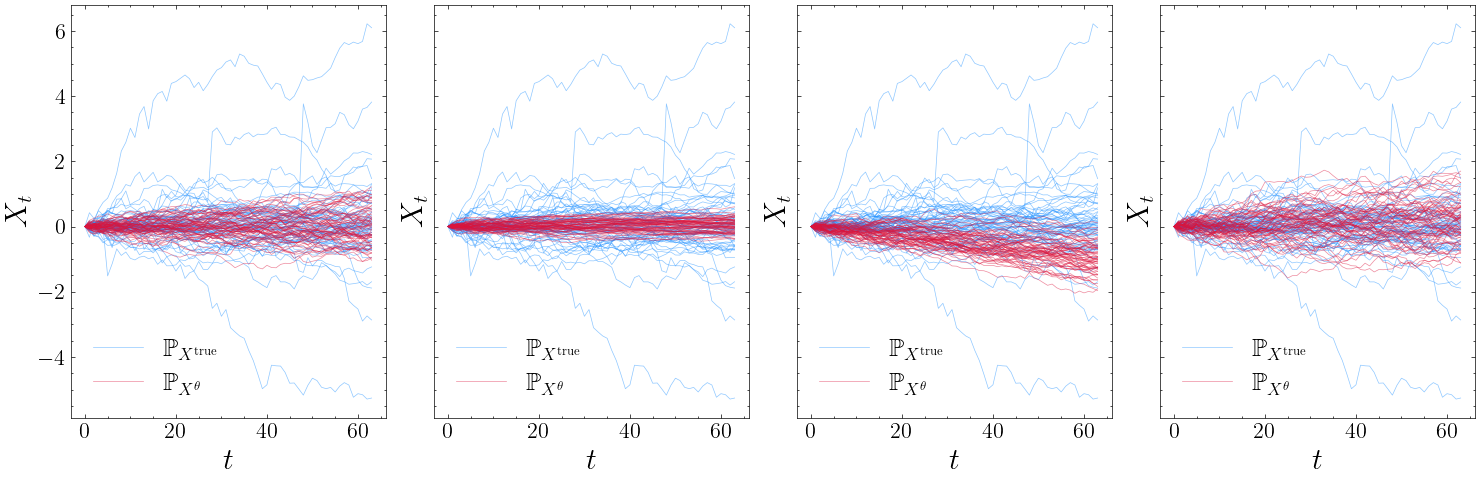}
\caption{Sample paths for gold prices from the metal dataset. Blue lines represent real samples, while red lines represent those generated by Neural SDEs. From left to right, the plots correspond to signature kernels, truncated signature, SDE-GAN, and FDM, respectively. The horizontal axis represents time, and the vertical axis represents gold prices.}
\label{fig:paths_metal_dim2}
\end{center}
\end{figure}

\end{document}

%% file: math_commands.tex
\usepackage{amsmath,amsfonts,bm}

\def\eqref#1{equation~\ref{#1}}

\def\1{\bm{1}}

\DeclareMathAlphabet{\mathsfit}{\encodingdefault}{\sfdefault}{m}{sl}
\SetMathAlphabet{\mathsfit}{bold}{\encodingdefault}{\sfdefault}{bx}{n}



%% file: macro.tex
\usepackage{amsmath,amssymb,amsthm, bbm, mathtools, xcolor}

\renewcommand{\thepage}{\arabic{page}}
\renewcommand{\thesection}{\arabic{section}}

\theoremstyle{plain}
\newtheorem{thm}{Theorem}
\newtheorem{lem}[thm]{Lemma}

\newtheorem{defi}[thm]{Definition}

\usepackage{amsmath}

\newcommand{\EE}{\mathbb{E}}

\newcommand{\NN}{\mathbb{N}}

\newcommand{\PP}{\mathbb{P}}

\newcommand{\RR}{\mathbb{R}}

\newcommand{\cA}{\mathcal{A}}
\newcommand{\cB}{\mathcal{B}}

\newcommand{\cE}{\mathcal{E}}
\newcommand{\cF}{\mathcal{F}}

\newcommand{\cO}{\mathcal{O}}

\newcommand{\cT}{\mathcal{T}}

\DeclarePairedDelimiter\abs{\lvert}{\rvert}%
\DeclarePairedDelimiter\norm{\lVert}{\rVert}%

\DeclarePairedDelimiter\autoparens{(}{)}
\newcommand{\prs}[1]{\autoparens*{#1}}
\DeclarePairedDelimiter\autobrackets{[}{]}
\newcommand{\brs}[1]{\autobrackets*{#1}}
\DeclarePairedDelimiter\autocurlybrackets{\{}{\}}
\newcommand{\cbs}[1]{\autocurlybrackets*{#1}}
\DeclarePairedDelimiter\autoinnerproduct{\langle}{\rangle}
\newcommand{\innerprod}[1]{\autoinnerproduct*{#1}}
\makeatletter
\let\oldabs\abs
\def\abs{\@ifstar{\oldabs}{\oldabs*}}
\let\oldnorm\norm
\def\norm{\@ifstar{\oldnorm}{\oldnorm*}}
\makeatother

%% file: main_paper.bbl
\begin{thebibliography}{31}
\providecommand{\natexlab}[1]{#1}
\providecommand{\url}[1]{\texttt{#1}}
\expandafter\ifx\csname urlstyle\endcsname\relax
  \providecommand{\doi}[1]{doi: #1}\else
  \providecommand{\doi}{doi: \begingroup \urlstyle{rm}\Url}\fi

\bibitem[Arjovsky et~al.(2017)Arjovsky, Chintala, and Bottou]{arjovsky17WGAN}
Martin Arjovsky, Soumith Chintala, and L{\'e}on Bottou.
\newblock {W}asserstein generative adversarial networks.
\newblock In Doina Precup and Yee~Whye Teh (eds.), \emph{Proceedings of the 34th International Conference on Machine Learning}, volume~70 of \emph{Proceedings of Machine Learning Research}, pp.\  214--223. PMLR, 06--11 Aug 2017.
\newblock URL \url{https://proceedings.mlr.press/v70/arjovsky17a.html}.

\bibitem[Arribas et~al.(2021)Arribas, Salvi, and Szpruch]{Arribas2021sigsde}
Imanol~Perez Arribas, Cristopher Salvi, and Lukasz Szpruch.
\newblock Sig-sdes model for quantitative finance.
\newblock In \emph{Proceedings of the First ACM International Conference on AI in Finance}, ICAIF '20, New York, NY, USA, 2021. Association for Computing Machinery.
\newblock ISBN 9781450375849.
\newblock \doi{10.1145/3383455.3422553}.
\newblock URL \url{https://doi.org/10.1145/3383455.3422553}.

\bibitem[Bonnier \& Oberhauser(2024)Bonnier and Oberhauser]{Bonnier24truncated}
Patric Bonnier and Harald Oberhauser.
\newblock {Proper Scoring Rules, Gradients, Divergences, and Entropies for Paths and Time Series}.
\newblock \emph{Bayesian Analysis}, pp.\  1 -- 32, 2024.
\newblock \doi{10.1214/24-BA1435}.
\newblock URL \url{https://doi.org/10.1214/24-BA1435}.

\bibitem[Bouchacourt et~al.(2016)Bouchacourt, Kumar, and Nowozin]{Bouchacourt16DiscoNets}
Diane Bouchacourt, M.~Pawan Kumar, and Sebastian Nowozin.
\newblock Disco nets: dissimilarity coefficient networks.
\newblock In \emph{Proceedings of the 30th International Conference on Neural Information Processing Systems}, NIPS'16, pp.\  352–360, Red Hook, NY, USA, 2016. Curran Associates Inc.
\newblock ISBN 9781510838819.

\bibitem[Choudhary et~al.(2023)Choudhary, Jaimungal, and Bergeron]{choudhary2023funvolmultiassetimpliedvolatility}
Vedant Choudhary, Sebastian Jaimungal, and Maxime Bergeron.
\newblock Funvol: A multi-asset implied volatility market simulator using functional principal components and neural sdes, 2023.
\newblock URL \url{https://arxiv.org/abs/2303.00859}.

\bibitem[Gierjatowicz et~al.(2020)Gierjatowicz, Sabate-Vidales, Šiška, Szpruch, and Žan Žurič]{gierjatowicz2020robust}
Patryk Gierjatowicz, Marc Sabate-Vidales, David Šiška, Lukasz Szpruch, and Žan Žurič.
\newblock Robust pricing and hedging via neural sdes, 2020.

\bibitem[Gneiting \& Raftery(2007)Gneiting and Raftery]{Gneiting07ProperScore}
Tilmann Gneiting and Adrian~E Raftery.
\newblock Strictly proper scoring rules, prediction, and estimation.
\newblock \emph{Journal of the American Statistical Association}, 102\penalty0 (477):\penalty0 359--378, 2007.
\newblock \doi{10.1198/016214506000001437}.
\newblock URL \url{https://doi.org/10.1198/016214506000001437}.

\bibitem[Gretton et~al.(2012)Gretton, Borgwardt, Rasch, Sch{{\"o}}lkopf, and Smola]{gretton12a}
Arthur Gretton, Karsten~M. Borgwardt, Malte~J. Rasch, Bernhard Sch{{\"o}}lkopf, and Alexander Smola.
\newblock A kernel two-sample test.
\newblock \emph{Journal of Machine Learning Research}, 13\penalty0 (25):\penalty0 723--773, 2012.
\newblock URL \url{http://jmlr.org/papers/v13/gretton12a.html}.

\bibitem[Gritsenko et~al.(2020)Gritsenko, Salimans, van~den Berg, Snoek, and Kalchbrenner]{Gritsenko20speech}
Alexey~A. Gritsenko, Tim Salimans, Rianne van~den Berg, Jasper Snoek, and Nal Kalchbrenner.
\newblock A spectral energy distance for parallel speech synthesis.
\newblock In \emph{Proceedings of the 34th International Conference on Neural Information Processing Systems}, NIPS '20, Red Hook, NY, USA, 2020. Curran Associates Inc.
\newblock ISBN 9781713829546.

\bibitem[Hodgkinson et~al.(2021)Hodgkinson, van~der Heide, Roosta, and Mahoney]{Hodgkinson2021normalflows}
Liam Hodgkinson, Chris van~der Heide, Fred Roosta, and Michael~W. Mahoney.
\newblock Stochastic continuous normalizing flows: training {SDEs} as {ODEs}.
\newblock In Cassio de~Campos and Marloes~H. Maathuis (eds.), \emph{Proceedings of the Thirty-Seventh Conference on Uncertainty in Artificial Intelligence}, volume 161 of \emph{Proceedings of Machine Learning Research}, pp.\  1130--1140. PMLR, 27--30 Jul 2021.
\newblock URL \url{https://proceedings.mlr.press/v161/hodgkinson21a.html}.

\bibitem[Hoglund et~al.(2023)Hoglund, Ferrucci, Hernandez, Gonzalez, Salvi, Sanchez-Betancourt, and Zhang]{hoglund2023neuralrdeapproachcontinuoustime}
Melker Hoglund, Emilio Ferrucci, Camilo Hernandez, Aitor~Muguruza Gonzalez, Cristopher Salvi, Leandro Sanchez-Betancourt, and Yufei Zhang.
\newblock A neural rde approach for continuous-time non-markovian stochastic control problems, 2023.
\newblock URL \url{https://arxiv.org/abs/2306.14258}.

\bibitem[Issa et~al.(2023)Issa, Horvath, Lemercier, and Salvi]{issa2023sigker}
Zacharia Issa, Blanka Horvath, Maud Lemercier, and Cristopher Salvi.
\newblock Non-adversarial training of neural sdes with signature kernel scores.
\newblock In A.~Oh, T.~Naumann, A.~Globerson, K.~Saenko, M.~Hardt, and S.~Levine (eds.), \emph{Advances in Neural Information Processing Systems}, volume~36, pp.\  11102--11126. Curran Associates, Inc., 2023.
\newblock URL \url{https://proceedings.neurips.cc/paper_files/paper/2023/file/2460396f2d0d421885997dd1612ac56b-Paper-Conference.pdf}.

\bibitem[Jia \& Benson(2019)Jia and Benson]{jia2019neuraljump}
Junteng Jia and Austin~R Benson.
\newblock Neural jump stochastic differential equations.
\newblock In H.~Wallach, H.~Larochelle, A.~Beygelzimer, F.~d\textquotesingle Alch\'{e}-Buc, E.~Fox, and R.~Garnett (eds.), \emph{Advances in Neural Information Processing Systems 32}, pp.\  9847--9858. Curran Associates, Inc., 2019.
\newblock URL \url{http://papers.nips.cc/paper/9177-neural-jump-stochastic-differential-equations.pdf}.

\bibitem[Kallenberg(2021)]{Kallenberg2021}
Olav Kallenberg.
\newblock \emph{Foundations of Modern Probability}.
\newblock Number~75 in Probability Theory and Stochastic Modelling. Springer Cham, 3 edition, 2021.
\newblock ISBN 978-3-030-61871-1.
\newblock \doi{10.1007/978-3-030-61871-1}.
\newblock URL \url{https://doi.org/10.1007/978-3-030-61871-1}.

\bibitem[Kidger(2022)]{kidger2022neuraldifferentialequations}
Patrick Kidger.
\newblock On neural differential equations, 2022.
\newblock URL \url{https://arxiv.org/abs/2202.02435}.

\bibitem[Kidger et~al.(2020)Kidger, Morrill, Foster, and Lyons]{kidger2020neuralcde}
Patrick Kidger, James Morrill, James Foster, and Terry Lyons.
\newblock {N}eural {C}ontrolled {D}ifferential {E}quations for {I}rregular {T}ime {S}eries.
\newblock \emph{Advances in Neural Information Processing Systems}, 2020.

\bibitem[Kidger et~al.(2021)Kidger, Foster, Li, and Lyons]{kidger21b}
Patrick Kidger, James Foster, Xuechen Li, and Terry~J Lyons.
\newblock Neural sdes as infinite-dimensional gans.
\newblock In Marina Meila and Tong Zhang (eds.), \emph{Proceedings of the 38th International Conference on Machine Learning}, volume 139 of \emph{Proceedings of Machine Learning Research}, pp.\  5453--5463. PMLR, 18--24 Jul 2021.
\newblock URL \url{https://proceedings.mlr.press/v139/kidger21b.html}.

\bibitem[Lee \& Oberhauser(2023)Lee and Oberhauser]{lee2023signaturekernel}
Darrick Lee and Harald Oberhauser.
\newblock The signature kernel, 2023.
\newblock URL \url{https://arxiv.org/abs/2305.04625}.

\bibitem[Li et~al.(2020)Li, Wong, Chen, and Duvenaud]{Li2020latentSDE}
Xuechen Li, Ting-Kam~Leonard Wong, Ricky T.~Q. Chen, and David~K. Duvenaud.
\newblock Scalable gradients and variational inference for stochastic differential equations.
\newblock In Cheng Zhang, Francisco Ruiz, Thang Bui, Adji~Bousso Dieng, and Dawen Liang (eds.), \emph{Proceedings of The 2nd Symposium on Advances in Approximate Bayesian Inference}, volume 118 of \emph{Proceedings of Machine Learning Research}, pp.\  1--28. PMLR, 08 Dec 2020.
\newblock URL \url{https://proceedings.mlr.press/v118/li20a.html}.

\bibitem[Mao(2007)]{mao2007stochastic}
Xuerong Mao.
\newblock \emph{Stochastic Differential Equations and Applications}.
\newblock Horwood Publishing, Chichester, UK, 2nd edition, 2007.
\newblock ISBN 978-1-904275-29-5.

\bibitem[McDiarmid(1989)]{mcdiarmid1989}
Colin McDiarmid.
\newblock On the method of bounded differences.
\newblock In \emph{Surveys in combinatorics, 1989}, volume 141 of \emph{London Mathematical Society Lecture Note Series}, pp.\  148--188. Cambridge University Press, 1989.

\bibitem[Morrill et~al.(2020)Morrill, Salvi, Kidger, Foster, and Lyons]{Morrill2020NeuralRD}
James Morrill, Cristopher Salvi, Patrick Kidger, James Foster, and Terry Lyons.
\newblock Neural rough differential equations for long time series.
\newblock In \emph{International Conference on Machine Learning}, 2020.
\newblock URL \url{https://api.semanticscholar.org/CorpusID:234358733}.

\bibitem[Morrill et~al.(2021)Morrill, Salvi, Kidger, Foster, and Lyons]{morrill2021neuralrough}
James Morrill, Cristopher Salvi, Patrick Kidger, James Foster, and Terry Lyons.
\newblock Neural rough differential equations for long time series.
\newblock \emph{International Conference on Machine Learning}, 2021.

\bibitem[Ni et~al.(2022)Ni, Szpruch, Sabate-Vidales, Xiao, Wiese, and Liao]{Ni2022sigwGAN}
Hao Ni, Lukasz Szpruch, Marc Sabate-Vidales, Baoren Xiao, Magnus Wiese, and Shujian Liao.
\newblock Sig-wasserstein gans for time series generation.
\newblock In \emph{Proceedings of the Second ACM International Conference on AI in Finance}, ICAIF '21, New York, NY, USA, 2022. Association for Computing Machinery.
\newblock ISBN 9781450391481.
\newblock \doi{10.1145/3490354.3494393}.
\newblock URL \url{https://doi.org/10.1145/3490354.3494393}.

\bibitem[Opper(2019)]{opper19variationalSDE}
Manfred Opper.
\newblock Variational inference for stochastic differential equations.
\newblock \emph{Annalen der Physik}, 531\penalty0 (3):\penalty0 1800233, 2019.
\newblock \doi{https://doi.org/10.1002/andp.201800233}.
\newblock URL \url{https://onlinelibrary.wiley.com/doi/abs/10.1002/andp.201800233}.

\bibitem[Pacchiardi \& Dutta(2022)Pacchiardi and Dutta]{pacchiardi2022likelihoodfreeinferencegenerativeneural}
Lorenzo Pacchiardi and Ritabrata Dutta.
\newblock Likelihood-free inference with generative neural networks via scoring rule minimization, 2022.
\newblock URL \url{https://arxiv.org/abs/2205.15784}.

\bibitem[Pacchiardi et~al.(2024)Pacchiardi, Adewoyin, Dueben, and Dutta]{Pacchiardi24forecastingscore}
Lorenzo Pacchiardi, Rilwan~A. Adewoyin, Peter Dueben, and Ritabrata Dutta.
\newblock Probabilistic forecasting with generative networks via scoring rule minimization.
\newblock \emph{Journal of Machine Learning Research}, 25\penalty0 (45):\penalty0 1--64, 2024.
\newblock URL \url{http://jmlr.org/papers/v25/23-0038.html}.

\bibitem[Salvi et~al.(2021)Salvi, Cass, Foster, Lyons, and Yang]{GoursatPDE}
Cristopher Salvi, Thomas Cass, James Foster, Terry Lyons, and Weixin Yang.
\newblock The signature kernel is the solution of a goursat pde.
\newblock \emph{SIAM Journal on Mathematics of Data Science}, 3\penalty0 (3):\penalty0 873--899, 2021.
\newblock \doi{10.1137/20M1366794}.
\newblock URL \url{https://doi.org/10.1137/20M1366794}.

\bibitem[Snow \& Krishnamurthy(2025)Snow and Krishnamurthy]{snow2025efficientneuralsdetraining}
Luke Snow and Vikram Krishnamurthy.
\newblock Efficient neural sde training using wiener-space cubature, 2025.
\newblock URL \url{https://arxiv.org/abs/2502.12395}.

\bibitem[Tzen \& Raginsky(2019)Tzen and Raginsky]{tzen2019neuralstochasticdifferentialequations}
Belinda Tzen and Maxim Raginsky.
\newblock Neural stochastic differential equations: Deep latent gaussian models in the diffusion limit, 2019.
\newblock URL \url{https://arxiv.org/abs/1905.09883}.

\bibitem[Yoon et~al.(2019)Yoon, Jarrett, and van~der Schaar]{Yoon2019timeGAN}
Jinsung Yoon, Daniel Jarrett, and Mihaela van~der Schaar.
\newblock Time-series generative adversarial networks.
\newblock In H.~Wallach, H.~Larochelle, A.~Beygelzimer, F.~d\textquotesingle Alch\'{e}-Buc, E.~Fox, and R.~Garnett (eds.), \emph{Advances in Neural Information Processing Systems}, volume~32. Curran Associates, Inc., 2019.
\newblock URL \url{https://proceedings.neurips.cc/paper_files/paper/2019/file/c9efe5f26cd17ba6216bbe2a7d26d490-Paper.pdf}.

\end{thebibliography}
